\documentclass{article}

\input{style/preambule}

\crefname{lemma}{Lemma}{Lemmas}
\crefname{fact}{Fact}{Facts}
\crefname{theorem}{Theorem}{Theorems}
\crefname{corollary}{Corollary}{Corollaries}
\crefname{claim}{Claim}{Claims}
\crefname{example}{Example}{Examples}
\crefname{problem}{Problem}{Problems}
\crefname{definition}{Definition}{Definitions}
\crefname{assumption}{Assumption}{Assumptions}
\crefname{subsection}{Subsection}{Subsections}
\crefname{section}{Section}{Sections}
\crefname{algorithm}{Algorithm}{Algorithms}
\crefname{algocf}{alg.}{algs.}
\Crefname{algocf}{Algorithm}{Algorithms}
\crefname{proposition}{Proposition}{Propositions}
\crefname{exemple}{Exemple}{Examples}
\crefname{remark}{Remark}{Remarks}

\newtheorem{theorem}{Theorem}

\newtheorem{definition}{Definition}
\newtheorem{lemma}{Lemma}
\newtheorem{assumption}{Assumption}
\newtheorem{proposition}{Proposition}

\newtheorem{remark}{Remark}

\newcommand{\gs}{\vspace{-0.5em}}

\DeclareMathOperator*{\argmin}{arg\,min}

\newcommand{\ffrac}[2]{\ensuremath{\frac{\displaystyle #1}{\displaystyle #2}}}

\newcommand{\SGD}{\texttt{SGD}}
\newcommand{\Artemis}{\texttt{Artemis}}

\newcommand{\Dore}{\texttt{Dore}}
\newcommand{\Diana}{\texttt{Diana}}
\newcommand{\Qsgd}{\texttt{Qsgd}}
\newcommand{\DoubleSqueeze}{\texttt{Double-Squeeze}}
\newcommand{\MCM}{\texttt{MCM}}
\newcommand{\RMCM}{\texttt{Rand-MCM}}
\newcommand{\RMCMG}{\texttt{Rand-MCM-G}}

\newcommand{\TOne}{\alpha\up}
\newcommand{\TTwo}{\alpha\up}

\newcommand{\ghost}{\texttt{Ghost}}

\newcommand{\qqquad}{\qquad\qquad}

\newcommand{\GG}{\mathcal{G}}

\newcommand{\E}{\mathbb{E}}
\newcommand{\N}{\mathbb{N}}
\newcommand{\R}{\mathbb{R}}
\newcommand{\WW}{\R^d}

\newcommand{\lnrm}{\left \|} 
\newcommand{\rnrm}{\right \|} 
\newcommand{\FullExpec}[1]{\E \left[#1\right]} 
\newcommand{\Expec}[2]{\E \left[#1~\middle|~#2\right]} 
\newcommand{\PdtScl}[2]{\left\langle~#1,~#2~\right\rangle}  
\newcommand{\SqrdNrm}[1]{ \lnrm #1\rnrm^2}  
\newcommand{\bigpar}[1]{\left( #1 \right)} 

\newcommand{\omgC}{\omega}

\newcommand{\g}{\textsl{g}} 
\newcommand{\gw}{\g} 
\newcommand{\wkm}{w_{k-1}}
\newcommand{\wkmhat}{\widehat{w}_{k-1}}
\newcommand{\wkmhati}{\widehat{w}_{k-1}^i}

\newcommand{\gwkhat}{\g_{k}(\widehat{w}_{k-1})}
\newcommand{\gwkihat}{\g_{k}^i(\widehat{w}_{k-1})}
\newcommand{\gwkHAT}{\widetilde{\g}_k}
\newcommand{\gwkiHAT}{\widehat{\g}_{k}^i(\wkmhat)}

\newcommand{\sigmstar}{\sigma_{*}}
\newcommand{\gwkihatRdmizd}{\g_{k}^i(\widehat{w}_{k-1}^i)}
\newcommand{\gwkiHATRdmizd}{\widehat{\g}_{k}^i(\widehat{w}_{k-1}^i)}
\newcommand{\iN}{{i=1}^N}
\newcommand{\gwktilde}{\widetilde{g}_k}

\newcommand{\RandOrNot}{\mathbf{\textcolor{green!50!black}{C}}}

\newcommand{\wkmHSqrd}{\ffrac{1}{N} \sum_\iN \SqrdNrm{\wkm - H_{k-2}^i}}

\newcommand{\C}{\mathcal{C}}

\newcommand{\up}{_\mathrm{up}}
\newcommand{\dwn}{_\mathrm{dwn}}

\newcommand{\dwnMemTerm}{\Upsilon}
\newcommand{\upMemTerm}{\Xi}
\newcommand{\heterog}{\mathrm{Heterog}}

\newcommand{\cst}{C}

\newcommand{\sizefig}{0.42}

\definecolor{brickred}{rgb}{0.8, 0.25, 0.33}

\definecolor{darkbrown}{rgb}{0.4, 0.26, 0.13}
\usepackage{booktabs}

\begin{document}
\maketitle
\addtocontents{toc}{\protect\setcounter{tocdepth}{0}}
\begin{abstract}
    We develop a new approach to tackle communication constraints in a distributed learning problem with a central server. 
    We propose and analyze a new algorithm that performs bidirectional compression and achieves the same convergence rate as algorithms using only uplink (from the local workers to the central server) compression. 
    To obtain this improvement, we design \MCM, an algorithm such that the downlink compression \emph{only impacts local models}, while the global model is preserved. 
    As a result, and contrary to previous works, the gradients on local servers are computed on \emph{perturbed models}. Consequently, convergence proofs are more challenging and require a precise control of this perturbation. To ensure it, \MCM~additionally combines model compression with a memory mechanism.
    This analysis opens new doors, e.g. incorporating  worker dependent randomized-models and partial participation.
\end{abstract}
\gs\gs
\section{Introduction}\gs\gs
\label{sec:intro}
Large scale distributed machine learning is widely used in many modern applications \cite{abadi_tensorflow_2016,caldas_leaf_2019,seide_cntk_2016}. The training is distributed over a potentially large number $N$ of workers that communicate either with a central server \citep[see][on federated learning]{konecny_federated_2016,mcmahan_communication-efficient_2017}, or using peer-to-peer communication \cite{colin_gossip_2016,vanhaesebrouck_decentralized_2017,tang_d2_2018}.

In this work, we consider a setting using a central server that aggregates updates from remote nodes.  Formally, we have a number of features $d\in \mathbb{N^*}$, and a convex cost  function $F: \mathbb{R}^d \rightarrow \mathbb{R}$. We want to solve the following distributed convex optimization problem using stochastic gradient algorithms \cite{robbins_stochastic_1951,bottou_large-scale_2010}: $\min_{w \in \mathbb{R}^d} F(w) \text{ with } F(w) = \frac{1}{N} \sum_{i=1}^N  F_i(w)$,
where $(F_i)_{i=1}^N$ is a \emph{local} risk function (empirical risk  or expected risk in a streaming framework).  This applies to both instances of \textit{distributed} and \textit{federated} learning.

An important issue of those frameworks is the high communication cost between the workers and the central server \citep[Sec.~3.5]{kairouz_advances_2019}. This cost is a concern from several points of view. First, exchanging information can be the bottleneck in terms of speed. Second, the data consumption and the bandwidth usage of training large distributed models can be problematic; and furthermore, the energetic and environmental impact of those exchanges is a growing concern.
Over the last few years, new algorithms were introduced, compressing  messages in the  \textit{upload communications} (i.e., from remote devices to the central server) in order to reduce the size of those exchanges \cite{seide_1-bit_2014,alistarh_qsgd_2017,wu_error_2018,agarwal_cpsgd_2018,wangni_gradient_2018,stich_sparsified_2018,stich_error-feedback_2020,mishchenko_distributed_2019,li_acceleration_2020}. More recently, a new trend has emerged to also compress the \textit{downlink communication}: this is \textit{bidirectional compression}. 

The necessity for bidirectional compression can depend on the situation. For example, a single uplink compression could be sufficient in \textit{asymmetric} regimes in which broadcasting a message to $N$ workers (``one to $N$'') is faster than aggregating the information coming from each node (``$N$ to one''). However, in other regimes, e.g. with few machines, where the bottleneck is the transfer time of a heavy model (up to several GB in modern Deep Learning architectures)  the downlink communication cannot be disregarded, as the upload and download speed are of the same order \cite{philippenko_artemis_2020}. Furthermore, in a situation in which participants have to systematically download an update (e.g., on their smartphones) to participate in the training, participants would prefer to receive a small size update (compressed) rather than a heavier one.  To encompass all situations, we consider algorithms for which the information exchanged is compressed in both directions.

To perform downlink communication, existing bidirectional algorithms  \cite{tang_doublesqueeze_2019,zheng_communication-efficient_2019,sattler_robust_2019,liu_double_2020,philippenko_artemis_2020,horvath_better_2020,xu_training_2020,gorbunov_linearly_2020} first aggregate all the information they have received, compress them and then carry out the broadcast. Both the main ``global'' model and the ``local'' ones perform the \textit{same} update with this compressed information. Consequently, the model hold on the central server and the one used on the local workers (to query the gradient oracle) are identical.
However, this means that the model on the central server has been artificially \textit{degraded}: instead of using all the information it has received, it is updated with the compressed information. 

Here, we focus on \textit{preserving} (instead of \textit{degrading}) the central model: the update made on its side does not depend on the downlink compression. This implies that the local models are \textit{different} from the central model. 
The local gradients are thus measured on a \textit{``perturbed model''} (or \textit{``perturbed iterate''}): such an approach requires a more involved analysis and the algorithm must be carefully designed to control the deviation between the local and global models~\citep{mania_perturbed_2016}.
For example, algorithms directly compressing the model or the update would simply not converge.

We propose \MCM~- \textit{Model Compression with Memory} - a~new algorithm that 1) preserves the central model, and 2) uses a memory scheme to reduce the variance of the local model. We prove that the convergence  of this method is similar to the one of algorithms using only unidirectional compression.

\vspace{-0.9em}
\paragraph{Potential Impact.} Proposing an analysis that handles perturbed iterates is the key to unlock three major challenges of distributed learning run with bidirectionally compressed gradients. 
First, we show that it is possible to improve the convergence rate by sending \textit{different \textbf{randomized} models} to the different workers, this is \RMCM. 
Secondly, this analysis also paves the way to deal with partially participating machines: the adaptation of \RMCM~to this framework is straightforward; while adapting existing algorithms~\cite{sattler_robust_2019} to partial participation is not practical. 
Thirdly, this framework is also promising in terms of business applications, e.g., in the situation of learning with privacy guarantees and \textit{with a trusted central server}. We  detail those three possible extensions in  \Cref{subsec:communication_tradeoffs}.

\gs
\paragraph{Broader impact.} 
This work is aligned with a global effort to make the usage of large scale Federated Learning sustainable by minimizing its environmental impact.
Though the impact of such algorithms is expected to be positive, at least on environmental concerns, cautiousness is still required, as a rebound effect may be observed \cite{grubb_communication_1990}: having energetically cheaper and faster algorithms may result in an increase of such applications, annihilating the gain made by algorithmic progress.

\gs
\paragraph{Contributions.} We make the following contributions:
\gs
\begin{enumerate}[topsep=0pt,itemsep=1pt,leftmargin=*,noitemsep]
    \item We propose a new algorithm \MCM, combining a memory process to the ``preserved'' update. To convey the key steps of the proof, we also introduce an auxiliary hypothetical algorithm, \ghost.
    \item For those algorithms, we carefully control the variance of the local models w.r.t.~the global one.  We provide a \textit{contraction equation} involving the control on the local model's variance and show that \MCM~achieves the same rate of convergence as single compression in strongly-convex, convex and non-convex regimes. We give a comparisons of \MCM's rates with existing algorithms in \Cref{tab:summary_rate}. 
    \item We propose a variant, \RMCM~incorporating diversity into  models shared with the local workers and show that it improves convergence for quadratic functions. 
\end{enumerate}

This is the first algorithm for double compression to focus on a \textbf{preserved central model}. We underline, both theoretically and in practice, that we get the same asymptotic convergence rate for simple and double compression - which is a major improvement. Our approach is one of the first to allow for worker dependent model, and to naturally adapt to worker dependent compression levels.

The rest of the paper is organized as follows: in \Cref{sec:pb_statment} we present the problem statement and introduce  \MCM~and \RMCM. Theoretical results on these algorithms are  successively presented in \Cref{sec:theory,,sec:theory_randomization}. Finally, we present experiments supporting the theory in \Cref{sec:experiments}.

\begin{table*}[!htp]
    \centering
         \caption{Features of the main existing algorithms performing compression. $e_k^i$ (resp. $E_k$) denotes the use of error-feedback at uplink (resp. downlink). $h_k^i$ (resp. $H_k$) denotes the use of a memory at uplink (resp. downlink).  Note that \texttt{Dist-EF-SGD} is identical to \DoubleSqueeze~but has been developed simultaneously and independently.\gs}
    \resizebox{\linewidth}{!}{\begin{tabular}{lcccccccc}
    \toprule
         & Compr. & $e_k^i$ & $h_k^i$ & $E_k$ & $H_k$ & Rand. & {update point} \\
          \midrule
         \Qsgd~\cite{alistarh_qsgd_2017} & \textcolor{red}{one-way} &  & &  &  & &  \\
         \texttt{ECQ-sgd}~\cite{wu_error_2018} & \textcolor{red}{one-way} & \cmark & &  &  &  & \\
         \Diana~\cite{mishchenko_distributed_2019} & \textcolor{red}{one-way} &  & \cmark  &  &  &  & \\
         \Dore~\cite{liu_double_2020} & two-way & & \cmark & \cmark &  &  & degraded \\
         \DoubleSqueeze~\cite{tang_doublesqueeze_2019},  \texttt{Dist-EF-SGD}~\cite{zheng_communication-efficient_2019} & two-way &  \cmark & & \cmark &  &  & degraded \\
         \Artemis~\cite{philippenko_artemis_2020} & two-way &  & \cmark &  &  &  & degraded \\
         \midrule
         \MCM~& two-way &  & \cmark &  & \cmark &  & \textcolor{green!50!black}{non-degraded} \\
         \RMCM & two-way &  & \cmark &  & \cmark & \cmark & \textcolor{green!50!black}{non-degraded} \\
          \bottomrule
    \end{tabular}}
    \label{tab:algo_summary}
    \vspace{-0.40cm}
\end{table*}

\gs\gs
\section{Problem statement}\gs
\label{sec:pb_statment}
We consider the minimization problem described in \cref{sec:intro}. In the convex case, we assume there exists an optimal parameter  $w_*$, and denote $F_*=F(w_*)$. We use $\lnrm \cdot \rnrm$ to denote the  Euclidean norm. 
To solve this problem, we rely on a stochastic gradient descent (SGD) algorithm. A stochastic gradient $\g_{k+1}^i$  is provided at iteration $k$ in $\N$ to the  device $i$ in $\llbracket 1, N \rrbracket$. This gradient oracle can be computed on a mini-batch of size $b$. This function is then evaluated at point $w_k$. In the classical centralized framework (without compression), for a  learning rate $\gamma$,  SGD corresponds to:
\gs
\begin{align}
   \label{eq:sgd_statement}
    w_{k+1} = w_k - \gamma \frac{1}{N}\sum_{i=1}^N \g_{k+1}^i(w_k) \,. 
\end{align}\gs
We now describe the framework used for compression.

\gs
\subsection{Bidirectional compression framework}\gs
\label{sec:bidirectional_framework}

Bidirectional compression consists in compressing communications in both directions between the central server and remote devices. 
We use two  different compression operators, respectively $\C\up$ and $\C\dwn$ to compress the message in each direction. 
 Roughly speaking, the update in \cref{eq:sgd_statement} becomes:
\begin{align*}
    w_{k+1} = w_k - \gamma \C\dwn\bigg( \frac{1}{N}\sum_{i=1}^N \C\up(\g_{k+1}^i(w_k)) \bigg)\,.
\end{align*}
However, this approach has a major drawback. The central server receives and aggregates information $ \frac{1}{N}\sum_{i=1}^N \C\up(\g_{k+1}^i(w_k))$. But in order to be able to broadcast it back, it compresses it, \textit{before} applying the update. We refer to this strategy as the ``degraded update'' approach. Its major advantage is simplicity, and it was used in all previous papers performing double compression.
Yet, it appears to be  a waste of valuable information. In this paper, we update the global model $w_{k+1}$ independently of the downlink compression:
\begin{align}
   \label{eq:model_compression_eq_statement}
 \left\{ \begin{array}{l}
       w_{k+1} = w_k - \gamma \frac{1}{N}\sum_{i=1}^N \C\up \left( \g_{k+1}^i( \hat w_k) \right) \,.  \\
       \hat w_{k+1}  =  C\dwn (w_{k+1})
  \end{array}  \right.
\end{align}
However, bluntly compressing $w_{k+1}$ in \cref{eq:model_compression_eq_statement} hinders convergence, thus the second part of the update needs to be refined by adding a memory mechanism. \textbf{We now describe both communication stages of the real \MCM, which is entirely defined by the following uplink and downlink equations.} 

\gs\gs
\begin{align}
\begin{array}{ll}
\textbf{\text{Downlink}}     &  \textbf{\text{Uplink}}    \\
    \begin{aligned}\label{eq:downlink}
\left\{
    \begin{array}{ll}
        \Omega_{k+1} = w_{k+1} - H_{k} \,, \\
    \widehat{w}_{k+1} = H_{k} + \C\dwn(\Omega_{k+1}) \\
    H_{k+1} = H_k + \alpha\dwn \C\dwn({\Omega}_{k+1}) .\\
    \end{array}
\right. 
\end{aligned}  & \begin{aligned}
\left\{
       \begin{array}{l}
    \forall i \in \llbracket1, N \rrbracket, \Delta_{k}^i = \g_{k+1}^i(\widehat{w}_{k}) - h_k^i  \\
    w_{k+1} = w_k - \ffrac{\gamma}{N} \sum_\iN \C\up(\Delta_{k}^i) + h_k^i\\
      h_{k+1}^{i} = h_k^i + \alpha\up \C\up({\Delta}_k^i).
    \end{array}
\right. 
\end{aligned}
\end{array}
\end{align}

\textbf{Downlink Communication.} We introduce a \textit{downlink memory term} $(H_{k})_{k}$, which is available on both workers and central server. The difference $\Omega_{k+1}$ between the model and this memory is compressed and exchanged, then the local model is reconstructed from this information. The memory is then updated as defined on left part of \cref{eq:downlink}, with a learning rate $\alpha\dwn$.

Introducing this memory mechanism is crucial to control the variance of the local model $\widehat{w}_{k+1}$. To the best of our knowledge \MCM~is the first algorithm that uses such a memory mechanism for downlink compression. This mechanism  was introduced by \citet{mishchenko_distributed_2019} for the uplink compression but with the other purpose of mitigating the impact of heterogeneity, while we use it here to avoid divergence of the local model’s variance.

\textbf{Uplink Communication.} The motivation to introduce an uplink memory term $h_k^i$ for each device $i \in \llbracket1, N \rrbracket$ is different, and better understood. Indeed, for the uplink direction, this mechanism is only necessary (and then crucial) to handle heterogeneous workers \citep[i.e., with different data distributions, see e.g.][]{philippenko_artemis_2020}.
Here, the difference $\Delta_k^i$ between the stochastic gradient $\g_{k+1}^i$ at the local model $\widehat{w}_{k}$ (as defined in \cref{eq:downlink}) and the memory term is compressed and exchanged. The memory is then updated as defined on right part of \cref{eq:downlink} with a rate $\alpha\dwn$.

\begin{remark}[Rate $\alpha\dwn$] It is necessary to use $\alpha\dwn<1$. Otherwise, the compression noise tends to propagate and is amplified, because of the multiplicative nature of the compression. In \Cref{fig:mcm_various_option} we compare \MCM, with 3~other strategies: compressing only the update, compressing $w_k-\wkmhat$, (i.e., $\alpha\dwn=1$), and compressing the model (i.e., $H_k=0$), showing that only \MCM~converges.
\end{remark}

\begin{remark}[Memory vs Error Feedback]\label{rem:EF} Error feedback is another technique, introduced by~\citet{seide_1-bit_2014}. 
In the context of double compression, it has been shown to improve convergence for a restrictive class of \emph{contracting} compression operators (which are generally biased) by \citet{zheng_communication-efficient_2019,tang_doublesqueeze_2019}. 
However, we note several differences to our approach. (1) For unbiased operators - as considered in \Dore, it did not lead to any theoretical improvement \citep[Remark 2 in Sec. 4.1.,][]{liu_double_2020}. (2) Moreover, only a fraction (namely $(1+\omgC\dwn)^{-1}$) of the ``error'' $w_{k+1} - \hat{w}_{k+1}$ can be preserved in the EF term (see line 18 in algo 1 in Liu et al.). It is thus impossible to recover the  central preserved model as a function of the degraded model and the EF term. (3) \cite{zheng_communication-efficient_2019} consider a biased operator 
and the same compression level for uplink and downlink compression. They also rely on stronger assumptions on the gradient (uniformly bounded) 
and only tackle the homogeneous case.
\end{remark}  

In \Cref{tab:algo_summary} we summarize the main algorithms for compression in distributed training.
As downlink communication can be more efficient than uplink, we consider distinct operators $\C\dwn$, $\C\up$ and allow the corresponding compressions levels to be distinct: those quantities are defined in \Cref{asu:expec_quantization_operator}.
\begin{assumption}
\label{asu:expec_quantization_operator}
There exists constants $\omgC\up\,,\omgC\dwn \in \R^*_+$, such that the compression operators $\C{\up}$ and $\C{\dwn}$ satisfy the two following properties for all $w$ in $\R^d$: $\E [\C_{\mathrm{up}\slash\mathrm{dwn}}(w)] = w$, and $\E [ \| \C_{\mathrm{up}\slash\mathrm{dwn}}(w) - w\|^2] \leq \omgC_{\mathrm{up}\slash\mathrm{dwn}} \|w\|^2$.
The higher is $\omgC$, the more aggressive the compression is.\gs
\end{assumption}
We only consider unbiased operators, that encompass sparsification, quantization and sketching. References and a discussion on those operators, and possible extensions of our results to biased operators are provided in \Cref{app:subsec:compression_oper}.

\begin{remark}[Related work on Perturbed iterate analysis]
The theory of perturbed iterate analysis was introduced by \citet{mania_perturbed_2016} to deal with asynchronous SGD. More recently, it was used by \citet{stich_error-feedback_2020,gorbunov_linearly_2020} to analyze the convergence of algorithms with uplink compressions, error feedback and asynchrony. Using gradients at randomly perturbed points can also be seen as a form of randomized smoothing~\citep{scaman_optimal_2018}, a point we discuss in \Cref{app:subsec:smoothing}.
\end{remark}

\gs
\subsection{The randomization mechanism, \RMCM}\gs
\label{sec:randomization_def}

In this subsection, we describe the key feature introduced in \RMCM: \textit{randomization}. It consists in performing an independent compression for each device instead of performing a single one for all of them. As a consequence, each worker holds a different model centered around the global one. This introduces some supplementary randomness that stabilizes the algorithm. Formally, we will consider $N$ mutually independent compression operators $\C_{\mathrm{dwn},i}$ instead of a single one $\C\dwn$, and the central server will send to the device $i$ at iteration $k+1$ the compression of the difference between its model and the local memory on worker~$i$:  $\C_{\mathrm{dwn},i} (w_{k+1} - H_k^i)$. The tradeoffs associated with this modification are discussed in \Cref{sec:theory_randomization}.

The pseudocode of \RMCM~is given in \Cref{algo} in \Cref{app:sec:complementary}. It incorporates all  components described above: 1) the bidirectional compression, 2) the model update using the non-degraded point, 3) the two memories, 4) the up and down  compression operators, 5) the randomization mechanism.

\gs\gs
\section{Assumptions and Theoretical analysis}\gs
\label{sec:theory}

We  make standard assumptions on $F:  \mathbb{R}^d \rightarrow \mathbb{R}$. We first  assume that the loss function $F$ is smooth.
\begin{assumption}[Smoothness]
\label{asu:smooth}
$F$ is twice continuously differentiable, and  is $L$-smooth, that is for all vectors $w_1, w_2$ in $\WW$: $\|\nabla F(w_1) - \nabla F(w_2) \| \leq L \| w_1 - w_2 \|$.

\gs
\end{assumption} 
Results  in \Cref{sec:theory} are provided in a convex, strongly-convex and non-convex setting. 
\begin{assumption}[Strong convexity]
\label{asu:cvx_or_strongcvx}
$F$ is $\mu$-strongly convex (or convex if $\mu=0$), that is for all vectors $w_1, w_2$ in $\WW$:
$F(w_2) \geq F(w_1) + (w_2 -w_1)^T \nabla F(w_1) + \frac{\mu}{2} \| w_2 - w_1 \|^2_2\,.$
\gs
\end{assumption}

Next, we present the assumption on the stochastic gradients.

\begin{assumption}[Noise over stochastic gradients computation]
\label{asu:noise_sto_grad}
The noise over stochastic gradients for a mini-batch of size $b$, is uniformly bounded: there exists a constant $\sigma \in \mathbb{R}_+$, such that for all $k$ in $\N$, for all $i$ in $\llbracket 1, N \rrbracket\,$ and for all $w$ in $\R^d$ we have: $E[\|\g_k^i(w) - \nabla F(w)\|^2] \leq \sigma^2/b$.
\gs
\end{assumption}

We here provide guarantees of convergence for \MCM. 
\MCM~incorporates an uplink memory term, designed to handle heterogeneous workers. To highlight our main contributions, that concerns the downlink compression, we  present the results in the homogeneous setting, that is with $F_i=F_j$ and $\alpha\up=0$. Similar results (almost identical, up to constant numerical factors) in to the heterogeneous setting are described in \Cref{app:sec:adaptation_to_heterogeneous_case}.
Experiments are also performed on heterogeneous workers. We provide here  convergence results in the strongly-convex, then convex case. 

\gs
\textbf{Notations and settings.} For $k$ in $\N$, we denote $\dwnMemTerm_k =  \SqrdNrm{w_k - H_{k-1}}$, and define $V_k = \E[ \SqrdNrm{w_k - w_*}] + 32\gamma L \omgC\dwn^2 \E [\Upsilon_k]$, which serves as Lyapunov function. $V_k$ is composed of two terms: the first one controls the quadratic distance to the optimal model, and the second controls the variance of the local models $\hat w_k$. For both theorems, we choose $\alpha\dwn = (8 \omgC\dwn)^{-1}$.
We denote $\Phi(\gamma) := (1 + \omgC\up) \bigpar{1 + 64 \gamma L \omgC\dwn^2}$.

\textbf{Limit learning rate:} There exists a maximal learning rate to ensure convergence. More specifically, we define $\gamma_{\max}:=\min(\gamma_{\max}^\mathrm{up}, \gamma_{\max}^\mathrm{dwn}, \gamma_{\max}^\Upsilon)$, where 
$\gamma_{\max}^\mathrm{up}:= (2 L \bigpar{ 1 + {\omgC\up/N}})^{-1}$ corresponds to the classical constraint on the learning rate in the unidirectional regime \cite[see][]{mishchenko_distributed_2019,philippenko_artemis_2020},  $\gamma_{\max}^\mathrm{dwn}:= (8 L \omgC\dwn)^{-1}$ 
is a similar constraint coming from the downlink compression, and $\gamma_{\max}^\Upsilon:= \big(8 \sqrt 2 L\omgC\dwn\sqrt{ 8{\omgC\dwn} + {\omgC\up/N}} \big)^{-1}$ is a combined constraint that arises when controlling the variance term $\Upsilon$.\footnote{The dependency in $\omgC^{3/2}$ is similar to the one obtained by \citet{horvath_stochastic_2019} in unidirectional compression in the non-convex case (Theorem 4).} 
Overall, this constraints are weaker than in the ``degraded'' framework \cite{liu_double_2020,philippenko_artemis_2020}, in which $\gamma_{\max}^{\text{Dore}} \le  \big(8 L( 1 + {\omgC\dwn})(1 + {\omgC\up/N})\big)^{-1}$.  
Especially, in the regime in which $\omgC_{\mathrm{up}, \mathrm{dwn}}\to \infty$ and  $\omgC\dwn \simeq \omgC\up \simeq :\omega$, the maximal learning rate for \MCM~is $(L\omgC^{3/2})^{-1}$, while it is $(L\omgC^{2})^{-1}$ in \cite{liu_double_2020,philippenko_artemis_2020}. Our $\gamma_{\max}$ is thus larger by a factor $\sqrt{\omgC}$, see \Cref{tab:summary_rate}. We define $\widetilde{L}$ such that $\gamma_{\max} = (2 \widetilde{L})^{-1}$.

\begin{theorem}[Convergence of \MCM~in the homogeneous and strongly-convex case]
\label{thm:cvgce_mcm_strongly_convex}
Under \Cref{asu:cvx_or_strongcvx,asu:expec_quantization_operator,asu:smooth,asu:noise_sto_grad} with $\mu>0$, for $k$ in $\N$,
for any sequence $(\gamma_k)_{k\geq 0}\le \gamma_{\max}$ we have:
\begin{align}\label{eq:Lyapunov-str-convex}
    V_k &\leq (1 - \gamma_{k} \mu) V_{k-1}  - \gamma_{k} \FullExpec{F(\wkmhat) - F(w_*)} +  \ffrac{\gamma_{k}^2 \sigma^2 \Phi(\gamma_k)}{Nb}\,,
\end{align}
Consequently, (1) if $\sigma^2 =0 $ (noiseless case), for $\gamma_k\equiv \gamma_{\max}$ we recover a linear convergence rate: $\E [\SqrdNrm{{w}_k - w_*}]\le (1-\gamma_{\max} \mu)^k V_0$; \ \  (2) if $\sigma^2 >0 $, taking for all $K$ in $\N$, $\gamma_K =2/(\mu (K+1) + \widetilde{L})$, for the weighted Polyak-Ruppert average $\bar{w}_K = \sum_{k=1}^K \lambda_k \wkm / \sum_{k=1}^K \lambda_k$, with $\lambda_k : =(\gamma_{k-1})^{-1}$,
\begin{align}
\label{eq:strongly_convex_polyak_ruppert}
    \hspace{-0.1cm}\FullExpec{F(\bar{w}_K) - F(w_*)} \leq \ffrac{ \mu + 2 \widetilde L}{4\mu K^2} \SqrdNrm{w_0 - w_*} + \ffrac{4\sigma^2(1 + \omgC\up)}{\mu K Nb} \bigpar{1 + \ffrac{64 L \omgC\dwn^2}{ \mu K}\ln(\mu K + \widetilde{L}) }.
\end{align}
\end{theorem}

\gs
\textbf{Limit Variance (\Cref{eq:Lyapunov-str-convex}).}    
For a constant $\gamma$, the variance term (i.e., term proportional to $\sigma^2$) in \Cref{eq:Lyapunov-str-convex} is upper bounded by $ \frac{ \gamma^2 \sigma^2}{Nb} (1+ \omgC\up) (1 + 64 \textcolor{Green}{\gamma L} \omgC\dwn^2)$. The impact of the downlink compression is attenuated by a factor $\gamma$. As $\gamma$ decreases,  this makes the limit variance similar to the one of \Diana, i.e., without downlink compression~\citep[Eq. 16 in Th. 2]{mishchenko_distributed_2019} and much lower than the variance for previous algorithms using double compression for which the variance scales quadratically with the compression constants as $\gamma^2 \sigma^2 (1+\omega\up)(1+\omega\dwn) /N$: (1) for \Dore, see Corollary~1 in \citet{liu_double_2020} (who indicate $(1-\rho)^{-1}\geq (1+\omgC\up/N) (1+\omgC\dwn)$), (2) for \Artemis~see Table 2 and Th. 3 point 2 in \cite{philippenko_artemis_2020}, (3) for \cite{gorbunov_linearly_2020}, see Theorem I.1. (with $\gamma D_1'\varpropto \gamma^2 \sigma^2 (1+\omega\up)(1+\omega\dwn) /N$).

Bound \ref{eq:strongly_convex_polyak_ruppert} has a quadratic dependence on $\omgC\dwn$, but the corresponding term is divided by an extra factor $K$, the number of iterations. For example in experiments, for \textit{w8a} using quantization with $s=2^0$, we have $\omgC\dwn \simeq 17$, and after only $50$ epoch with a batch size $b=12$, we have $K \simeq 2500$. Hence, the term $\omgC^2 / K$ is vanishing through iterations and we asymptotically recover a rate of convergence equivalent to algorithms using unidirectional compression.

\textbf{Convergence and complexity:} With a decaying sequence of steps, we obtain a convergence rate scaling as $O(K^{-1})$ in \Cref{eq:strongly_convex_polyak_ruppert}, without dependency on the $\omgC\dwn$ in the dominating term, which only appears in faster decaying terms scaling as $K^{-2}$. The iteration complexity (i.e., number of iterations to achieve $\epsilon$ expected error) is thus at first order $O_{\epsilon\to 0}( \frac{\sigma^2(1 + \omgC\up)}{\mu \epsilon N b})$. Again, this matches the complexity of Diana \cite[see Theorem 1 and Corollary 1]{horvath_stochastic_2019} and is smaller by a factor $1+\omgC\dwn$ than the one of \Artemis, \Dore, \texttt{DIANAsr-DQ} (see Corollary I.1. in \cite{gorbunov_linearly_2020}). Next, we  give a convergence result in the convex case. 
\begin{theorem}[Convergence of \MCM, convex case]
\label{thm:cvgce_mcm_convex}
Under \Cref{asu:cvx_or_strongcvx,asu:expec_quantization_operator,asu:smooth,asu:noise_sto_grad} with $\mu=0$. For all $k>0$, for any $\gamma \le \gamma_{\max}$,   we have, for $\bar{w}_k = \frac{1}{k}\sum_{i=0}^{k-1} w_i$,
\begin{align}\label{eq:Lyapunov-convex}
   \hspace{-1em}\gamma \FullExpec{F(w_{k-1}) - F(w_*)} &\leq V_{k-1} - V_k   + \ffrac{\gamma^2\sigma^2 \Phi(\gamma)}{Nb} \Longrightarrow    \E[    F(\bar{w}_k) - F_*] \leq \frac{V_{0}}{\gamma k}+ \ffrac{\gamma\sigma^2 \Phi(\gamma)}{Nb}\,.
\end{align}
Consequently, for $K$ in $\N$ large enough, a step-size $\gamma= \sqrt{\frac{ \SqrdNrm{w_{0} - w_*}Nb  }{(1 + \omgC\up)  \sigma^2 K }}$, we have:
\begin{align} \label{eq:convergence-MCM}
\E[    F(\bar{w}_K) - F_*] \leq 2 \sqrt{\frac{ \SqrdNrm{w_{0} - w_*} (1 + \omgC\up) \sigma^2 }{Nb K }}  + O(K^{-1}).
\end{align}
Moreover if $\sigma^2 =0 $ (noiseless case), we recover a faster convergence: $\E[    F(\bar{w}_K) - F_*] =O(K^{-1})$. 
\end{theorem}

\gs
\textbf{Limit Variance (Eq. \eqref{eq:Lyapunov-convex}).} The variance term  is identical to the strongly-convex case.

\textbf{Convergence and complexity (\Cref{eq:convergence-MCM}).} The downlink compression constant only appears in the second-order term, scaling as $1/K$. In other words, the convergence rate is equivalent to the convergence rate of \Diana, in the non-strongly-convex. As $K$ increases, this complexity scales as  $\frac{(1+\omgC\up)}{n\epsilon^2}$ independently of the downlink compression. Again, for previous algorithms with double compression the complexity is at least $O\left(\frac{(1+\omgC\up)(1+\omgC\dwn)}{n\epsilon^2}\right)$ (see Corollary I.2 in \cite{gorbunov_linearly_2020}).

\textbf{Control of the variance of the local model.}

\begin{wrapfigure}[11]{R}{0.48\textwidth}
\flushright
    \gs\gs\gs\gs\gs
    \begin{center}
    \includegraphics[width=0.23\textwidth]{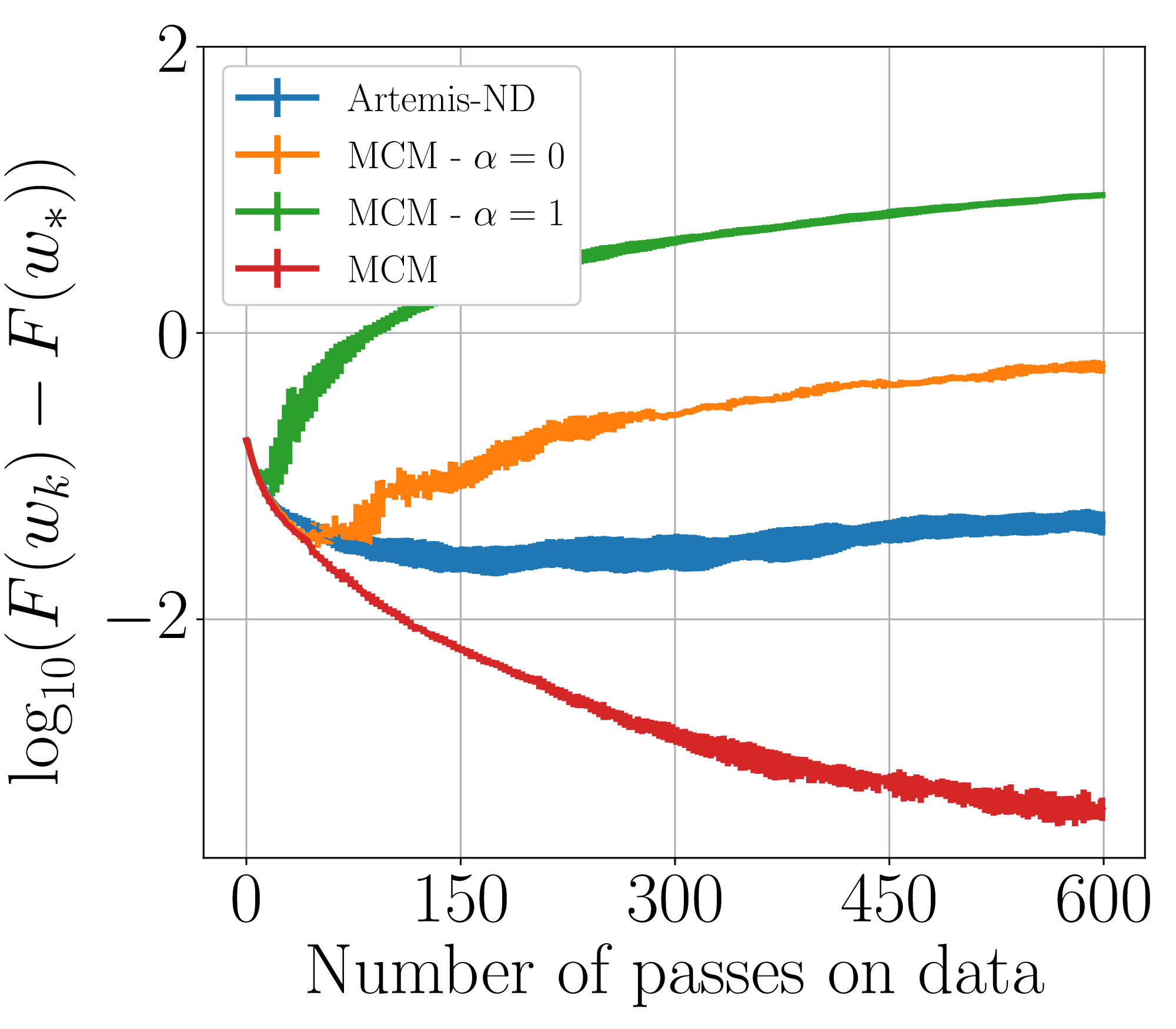}
    \includegraphics[width=0.23\textwidth]{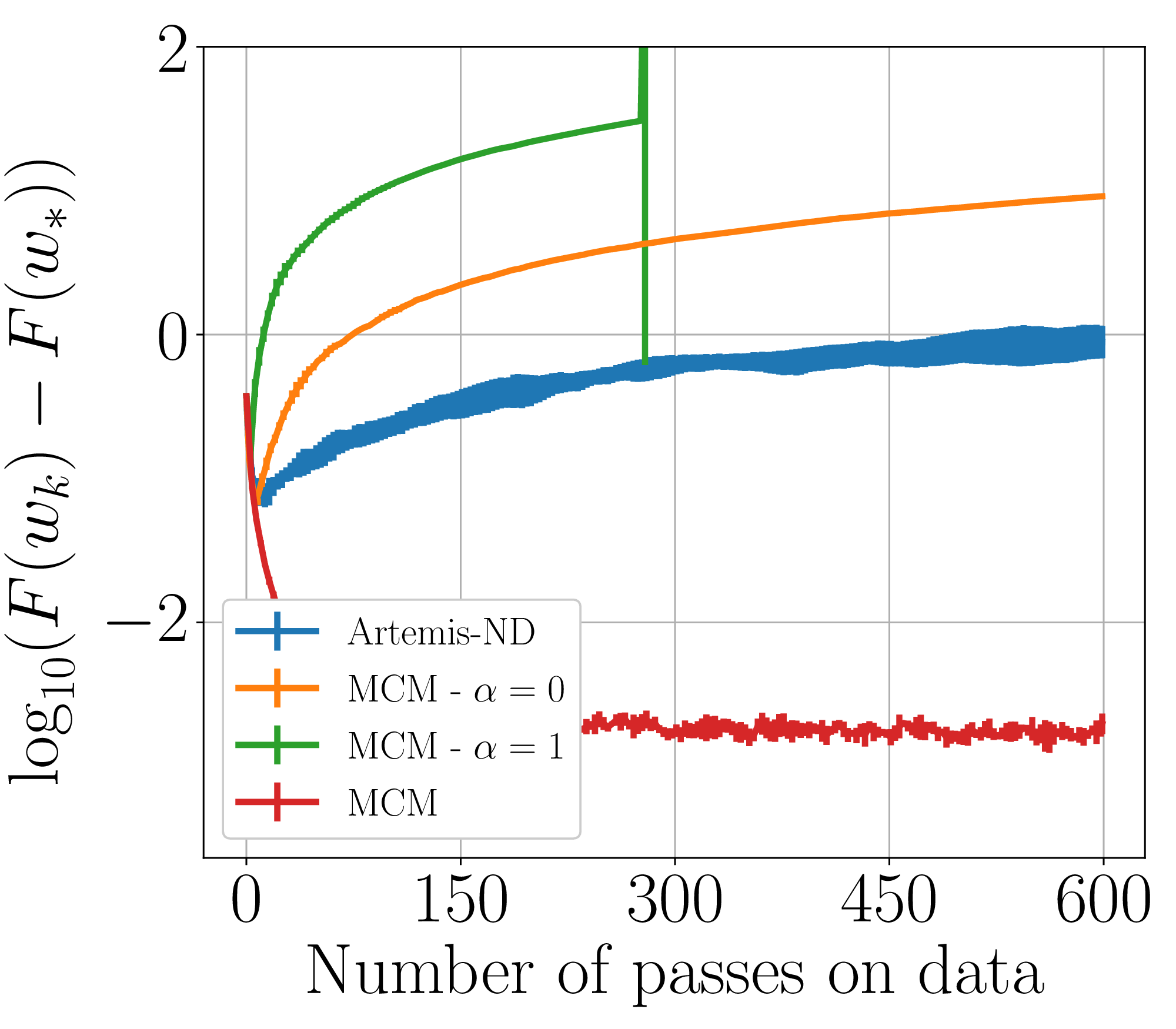}
    \end{center}
    \caption{Comparing \MCM~on two datasets with three other algorithms using a non-degraded update, $\gamma = 1/L$. \texttt{Artemis-ND} stands for \Artemis~with a non-degraded update.\vspace{-0.9em}}
    \label{fig:mcm_various_option}
\end{wrapfigure}

We here  present the backbone Lemma of \MCM's proof. It allows to control the variance of the local model $\E[\SqrdNrm{\hat w_k - w_k}|w_k]$ (which is upper-bounded by $ \omgC\dwn \E[\SqrdNrm{\dwnMemTerm_k}|w_k] $) and to build the Lyapunov function defined in \Cref{thm:cvgce_mcm_strongly_convex,thm:cvgce_mcm_convex}. 

This result  highlights the impact of the downlink  memory term. Without memory, i.e., with $\alpha\dwn =0$, the variance of the local model $\|\hat w_k- w_k\|^2 $ increases with the number of iterations. On the other hand, if $\alpha\dwn$ is too large (close to 1), this variance diverges.  This behavior is illustrated on two real datasets on \Cref{fig:mcm_various_option}. This phenomenon is similar to the divergence observed in frameworks involving error feedback, when the compression operator is not contractive. 

\begin{theorem}
\label{thm:contraction_mcm}
Consider the \MCM~update as in \cref{eq:model_compression_eq_statement}. Under \Cref{asu:expec_quantization_operator,asu:smooth,asu:noise_sto_grad} with $\mu=0$,  if $\gamma \leq ({8\omgC\dwn L})^{-1}$ and $\alpha \leq (4 \omgC\dwn){-1}$, then for all $k$ in $\N$:
\begin{align*}
    \FullExpec{\dwnMemTerm_{k}} &\leq \bigpar{1 - \ffrac{\alpha\dwn}{2}} \FullExpec{\dwnMemTerm_{k-1}} + 2 \gamma^2\bigpar{\frac{1}{\alpha\dwn} + \ffrac{\omgC\up}{N}} \FullExpec{\SqrdNrm{\nabla F(\wkmhat)}} + \ffrac{2\gamma^2 \sigma^2(1+\omgC\up)}{Nb} \,.
\end{align*}
\end{theorem}
This bound provides a recursive control on $\Upsilon_k$. Beyond the $(1-\alpha\dwn)$ contraction, the bound comprises the squared-norm of the gradient at the previous perturbed iterate, and a noise term.

\textbf{Summary of rates.} In \Cref{tab:summary_rate}, we summarize the rates and complexities, and maximal learning rate for \Diana, \Artemis, \Dore~and \MCM. For simplicity, we ignore absolute constants, and provide asymptotic values for large $\omgC\up$, $\omgC\dwn$, and complexities for $\epsilon\to 0$.
\begin{table*}[!htp]
    \centering
    \caption{Summary of rates on the initial condition, limit variance, asympt. complexities and $\gamma_{\max}$. \gs}
    \label{tab:summary_rate}
\resizebox{\linewidth}{!}{\begin{tabular}{lp{3.3cm}lll}
\toprule
Problem &  &Diana &Artemis, Dore & MCM, Rand-MCM   \\
\midrule
& $L \gamma_{\max} \varpropto$   &$ 1/(1+\omgC\up)$& $ 1/(1+\omgC\up)(1+\omgC\dwn)$ & $ 1/ (1+\omgC\dwn) \sqrt{1+\omgC\up} \wedge 1/ (1+\omgC\up)$ \\
  & Lim. var.  $\varpropto\gamma^2\sigma^2 / n \times $  &$(1+\omgC\up)$&$ (1+\omgC\up)(1+\omgC\dwn)$ &$ (1+\omgC\up)(1+\textcolor{green!50!black}{\gamma L}\omgC\dwn^2)$ \\
\midrule
Str.-convex & Rate on init. cond. (SC)&$(1-\gamma \mu)^k$&$(1-\gamma \mu)^k$  &  $(1-\gamma \mu)^k$ \\
& Complexity & $(1 + \omgC\up) / \mu \epsilon N $ & $(1+ \omgC\dwn)(1 + \omgC\up) / \mu \epsilon N $ & $(1 + \omgC\up) / \mu \epsilon N $\\
\midrule
Convex  & Complexity & $(\omgC\up + 1) / \epsilon^2$ & $(1 + \omgC\up ) (1+ \omgC\dwn )  / \epsilon^2$ & $(\omgC\up + 1) / \epsilon^2$ \\
\bottomrule
\end{tabular}}\gs
\end{table*}

\textbf{Proof in the heterogeneous case.}
To extend \Cref{thm:cvgce_mcm_convex,thm:cvgce_mcm_strongly_convex,thm:contraction_mcm} in the heterogeneous setting for a convex objective (\Cref{app:sec:adaptation_to_heterogeneous_case}), we assume  that there exists a constant $B$ in $\mathbb{R_+}$, s.t.: 
$\frac{1}{N} \sum_{i=0}^N \| \nabla F_i(w_*)\|^2 = B^2\,.$
We further define $\upMemTerm_k = \frac{1}{N^2} \sum_\iN \SqrdNrm{h_{k}^i - \nabla F_i(w_*)}$,  where for all $i$ in $\llbracket 1, N \rrbracket$.
This term is recursively controled \cite{mishchenko_distributed_2019,philippenko_artemis_2020} and combined  into the Lyapunov function.  

\textbf{Proofs.} To convey the best understanding of the theorems and the spirit of the proof, we introduce a \ghost~algorithm (impossible to implement) in \Cref{app:sec:ghost_def}. A sketch of the proof describes the main steps in the case of \ghost, those steps are similar for \MCM. Fundamentally, our proof relies on a tight analysis, related to perturbed iterate analysis~\cite{mania_perturbed_2016}.  Proofs of \Cref{thm:cvgce_mcm_strongly_convex,thm:cvgce_mcm_convex,thm:contraction_mcm} are given in \Cref{app:sec:proofs_mcm}.
Th.~\ref{app:thm:mcm_non_convex} in \Cref{app:subsec:nonconvexMCM} ensures convergence for a non-convex $F$. Note that the proof for non-convex follows a different approach than the one in \Cref{thm:cvgce_mcm_strongly_convex,thm:cvgce_mcm_convex}.

As mentioned in the introduction, our analysis of perturbed iterate in the context of double compression opens new directions: in particular, it opens the door to handling a different model for each worker. In the next section, we detail those possibilities, and provide theoretical guarantees for \RMCM, the variant of \MCM~in which instead of sending the same model to all workers, the compression noises are mutually \textit{independent}.

\begin{remark}[Communication budget]
How to split a given communication budget between uplink and downlink to optimize the convergence is an open question which is intrinsically related to the situation. Indeed it depends on many factors like the selected operators of compression, the upload/downlink speed or the number of participating workers at each iteration. 
However, our approach provides some insights on this question. Because \textit{asymptotically} the impact of double compression is marginal, for a fixed budget, \Cref{thm:cvgce_mcm_convex} suggests to strongly compress on the downlink direction (which leads to a large $\omega_{dwn}$), but to perform a weaker compression in the uplink direction.
\end{remark}

\gs\gs
\section{Extension to \RMCM}\gs
\label{sec:theory_randomization}

\subsection{Communication and convergence trade-offs}\gs
\label{sec:communication_trade_offs}

\label{subsec:communication_tradeoffs}
In \RMCM, we leverage the fact that the compressions used for each worker need not to be identical. On the contrary, it is possible to consider \textit{independent} compressions. By doing so, we reduce the impact of the downlink compression.

The relevance of such a modification depends on the framework: while the convergence rate will be improved, the computational time can be slightly increased. Indeed, $N$ compressions need to be computed instead of one: however, this computational time is typically not a bottleneck w.r.t. the communication time. 
A more important aspect is the communication cost. While the size of each message will remain identical, a different message needs to be sent to each worker. That is, we go from a ``one to $N$'' configuration to $N$ ``one to one'' communications. While this is a drawback, it is not an issue  when the bandwidth/transfer time are the bottlenecks, as \RMCM~will result in a better convergence with almost no cost. Furthermore, we argue that handling worker dependent models is essential for several major applications. \RMCM~can directly be adapted to those frameworks.

\textbf{1. Worker dependent compression.} A first simple situation is the case in which workers are allowed to choose the size (or equivalently the compression level) of their updates.

\textbf{2. Partial participation (PP).}
Similarly, having $N$ different messages to send to each worker may be unavoidable in the case of \textit{partial participation} of the workers.
This is a key feature in Federated Learning frameworks \cite{mcmahan_communication-efficient_2017}. 
In the classical distributed framework (without downlink constraints) it is easy to deal  with it, as each available worker just queries the global model to compute its gradient on it~\citep[see for example][]{horvath_better_2020}.  
On the other hand, for bidirectional compression, to ensure that all the local models match the central model, the adaptation to partial participation relies on  a \textit{synchronization step}. During this step, each  worker that has not participated in the last $S$ steps receives the last $S$ corresponding messages as long as it costs less to send this sequence than a full uncompressed model. This is described in the description of the adaptation to partial participation in \citep{philippenko_artemis_2020}, in the remark preceding Eq. (20) in \citep{sattler_robust_2019} and by \citet[][v2 on arxiv for the distributed case]{tang_doublesqueeze_2019}, who use a buffer.
On the contrary, \RMCM~naturally handles a different model, memory and update per worker. The adaptation to partial participation is thus straightforward. Though theoretical results are out of the scope of this paper, we provide experiments on PP in \Cref{app:subsec:expe_partial_part,fig:one_mem}.

One drawback is the necessity to store the $N$ memories $(H_k^i)_{i\in [N]}$ instead of one, which results in an additional memory cost. To circumvent this issue we propose two independent solutions. 1) Keep and use a single memory $\bar H_k = N^{-1}\sum_{{i=1}}^NH_k^i$ (as suggested in \citep{philippenko_artemis_2020}).  It is then necessary to periodically reset the local memories $H_k^i$ on all workers to the averaged value $\bar H_k$ (rarely enough not to impact the communication budget). This  is illustrated in \cref{fig:one_mem}. 2) Use \RMCM~with an arbitrary number of groups $G \ll N$ of workers. In each group $\mathcal G_g$, $g\in [G]$, all workers share the same  memory $(H_k^g)$ and receive the same update $\C_{\mathrm{dwn},g} (w_{k+1} - H_k^g)$. We call this algorithm \RMCMG.

\begin{remark}[Protecting the global model from honest-but-curious clients]
\label{remrk:privacy}
Another business advantage of \MCM~and \RMCM~is 
that providing \textit{degraded} models to the participants can be used to guarantee privacy, or to ensure the workers participate in good faith, and not only to obtain the model. 
This issue of detecting ill-intentioned clients (free-riders) that want to obtain the model without actually contributing  has been studied by \citet{fraboni_free-rider_2021}.
\end{remark}

\gs\gs
\subsection{Theoretical results}\gs
In this Section, we provide two main theoretical results for \RMCM. First \Cref{thm:cvgce_rmcm}  ensures that the theoretical guarantees are at least as good for \RMCM~as for \MCM. Then, in \Cref{thm:rmcm_quad_guarantee}, we provide convergence result for both \MCM~and \RMCM~in the case of quadratic functions.

\begin{theorem}
\label{thm:cvgce_rmcm}
\Cref{thm:cvgce_mcm_strongly_convex,thm:cvgce_mcm_convex,thm:contraction_mcm} are valid for \RMCM~and \RMCMG.
\end{theorem} \gs
The improvement in \RMCM~comes from the fact that we are ultimately averaging the gradients at several random points, reducing the variance coming from this aspect. The goal is obviously to reduce the impact of $\omgC\dwn$. Keeping in mind  that the dominating term in the rate is independent of $\omgC\dwn$, \textit{we can thus only expect to reduce the second-order term}. Next, the uplink compression noise increases with the variance of the randomized model, which will not be directly reduced by \RMCM. As a consequence, we only expect the improvement to be visible in the part of the second-order term that does not depend on $\omgC\up$ (that is, the effect would be the most significant if $\omgC\up$ is small or 0).

This intuition is corroborated by the following result, in which we show that the convergence is improved when adding the randomization process for a quadratic function. Extending the proof beyond quadratic functions is possible, though it requires an assumption on third or higher order derivatives of $F$ (e.g., using self-concordance \citep{bach_self-concordant_2010}) to  control of $\Expec{||\nabla F(\wkmhat) - \E[\nabla F(\wkmhat)]||^2}{\wkm}$. 

\begin{theorem}[Convergence in the quadratic case]\label{thm:rmcm_quad_guarantee}
Under \Cref{asu:cvx_or_strongcvx,asu:expec_quantization_operator,asu:smooth,asu:noise_sto_grad} with $\mu=0$, if the function is quadratic, after running $K>0$ iterations, for any $\gamma \le \gamma_{\max}$, and we have
\begin{align*}
   \E[    F(\bar{w}_K) - F_*] \leq \frac{V_{0}}{\gamma K}+ \frac{\gamma\sigma^2 \Phi^{\mathrm{Rd}}(\gamma)}{Nb} \,,
\end{align*}
with $\Phi^{\mathrm{Rd}}(\gamma)=(1 + \omgC\up) \bigpar{1 + \frac{4 \gamma^2 L^2 \omgC\dwn}{K} (\frac{1}{\RandOrNot} + \frac{\omgC\up }{N})}$ and $ \RandOrNot = N$ for \RMCM,  $ \RandOrNot = G$  \RMCMG, and $ \RandOrNot = 1$  for \MCM. 
\end{theorem}

This result is derived in \Cref{app:sec:proofs_quad}. We can make the following comments:
(1) The convergence rate for quadratic functions is slightly better than for smooth functions. More specifically, the right hand term in $\Phi$ is multiplied by an additional  $ \gamma \left( \frac{1}{\RandOrNot} + \frac{\omgC\up}{N}\right)$ (w.r.t.~\Cref{thm:cvgce_mcm_convex}), which is decaying at the same rate as $\gamma$. 
  Besides, the proof for \RMCM~is substantially modified, as  $\E[\nabla F(\wkmhat)] $ is an unbiased estimator of  $\nabla F (\wkm)$.
   (2) Moreover, the randomization in \RMCM~(resp. \RMCMG) further reduces by a factor $N$ (resp. $G$)  this  term. Depending on the relative sizes of $\omgC\up$ and $N$, this can lead to a significant improvement up to a factor of $N$. In practice the impact of \RMCM~is noticeable, as illustrated in the following experiments. 

\gs
\section{Experiments}\gs
\label{sec:experiments}

In this section, we  illustrate the validity of the theoretical results given in the previous section on both synthetic and real datasets,  on (1) least-squares linear regression (LSR), (2) logistic regression (LR), and (3) non-convex deep learning. We compare \MCM~with classical algorithms used in distributed settings: \Diana, \Artemis, \Dore~and of course the simplest setting - \SGD, which is the baseline.

In these experiments, we provide results on the log of the excess loss $F(w_k) - F_*$, averaged on $5$ runs (resp. $2$) in convex settings (resp. deep learning),  with errors bars displayed on each figure (but not in the ``zoom square''), corresponding to the standard deviation of $\log_{10}(F(w_k)-F_*)$. 
On \Cref{fig:real_dataset}, the X-axis is respectively the number of iterations and the number of bits exchanged. 

Each experiment has been run with $N=20$ workers using stochastic scalar quantization~\citep{alistarh_qsgd_2017}, w.r.t. $2$-norm. To maximize compression, we always quantize on a single level ($s=2^0$), unless for PP ($s=2^1$) and neural network (the value of $s$ depends on the dataset). 

We used $9$ different datasets.
\begin{itemize}[topsep=0pt,itemsep=1pt,leftmargin=*,noitemsep]
    \item One toy dataset devoted to linear regression in an homogeneous~setting. This toy dataset allows to illustrate \MCM~properties in a simple~framework, and in particular to ilustrate that when $\sigma^2=0$, we recover a linear convergence\footnote{Even stronger, we show in experiments that we recover a linear rate if we have $\sigmstar = 0$ (the noise over stochastic gradient computation at the optimum point $w_\star$).}, see \Cref{fig:LSR_full}. 
    \item Five datasets commonly used in convex optimization (a9a, quantum, phishing, superconduct and w8a); see \Cref{app:tab:settings_convex} for more details. Experiments were conducted with heterogeneous workers obtained by clustering (using  \textit{TSNE} \citep{maaten_visualizing_2008}) the input points.
    \item Four dataset in a non-convex settings (CIFAR10, Fashion-MNIST, FE-MNIST, MNIST); see \Cref{app:tab:settings_nonconvex} for more details.
\end{itemize}

All experiments are performed without any tuning of the algorithms, (e.g., with the same learning rate for all algorithms and without reducing it after a certain number of epochs). Indeed, our goal is to show that our method achieves a performance close to the unidirectional-compression framework (\Diana), while performing an important downlink compression. More details about experiments can be found in \Cref{app:sec:experiments}.

On \Cref{fig:real_dataset}, we display the excess loss for quantum  and a9a w.r.t. the number of iteration and number of communicated bits. The plots of phising, superconduct and w8a are not provided but can be found on our \href{https://github.com/philipco/mcm-bidirectional-compression/notebook}{github repository}. We only report their excess loss after $450$ iterations in \Cref{tab:exp_convex}.

\gs\gs
\begin{table}[!htp]
\caption{\MCM - convex experiments, $b$ is the batch size}.
\label{tab:exp_convex}
\centering
\begin{tabular}{lrrrrr}
Excess loss after $450$ epochs & \SGD & \Diana & \MCM & \Dore & Ref\\
\hline 
a9a ($b=50$) & $-3.5$ & $-2.7$ & $-2.7$ & $-1.8$ & \cite{chang_libsvm_2011}\\
quantum ($b=400$) & $-3.4$ & $-3.2$ & $-3.2$ & $-2.6$ & \cite{caruana_kdd-cup_2004}\\
phishing ($b=50$) & $-3.7$ & $-3.5$ & $-3.4$ & $-2.7$ & \cite{chang_libsvm_2011}\\
superconduct ($b=50$) & $-1.6$ & $-1.6$ & $-1.55$ & $-1.45$ & \cite{hamidieh_data-driven_2018}\\
w8a ($b=12$) & $-3.5$ & $-3.0$ & $-2.5$ & $-1.75$ & \cite{chang_libsvm_2011}\\
\hline 
Compression & no & uni-dir & bi-dir & bi-dir & \\
\hline
\end{tabular}
\end{table}

\begin{figure}
    \centering
    \begin{subfigure}{0.24\linewidth}
        \centering
        \includegraphics[width=1\textwidth]{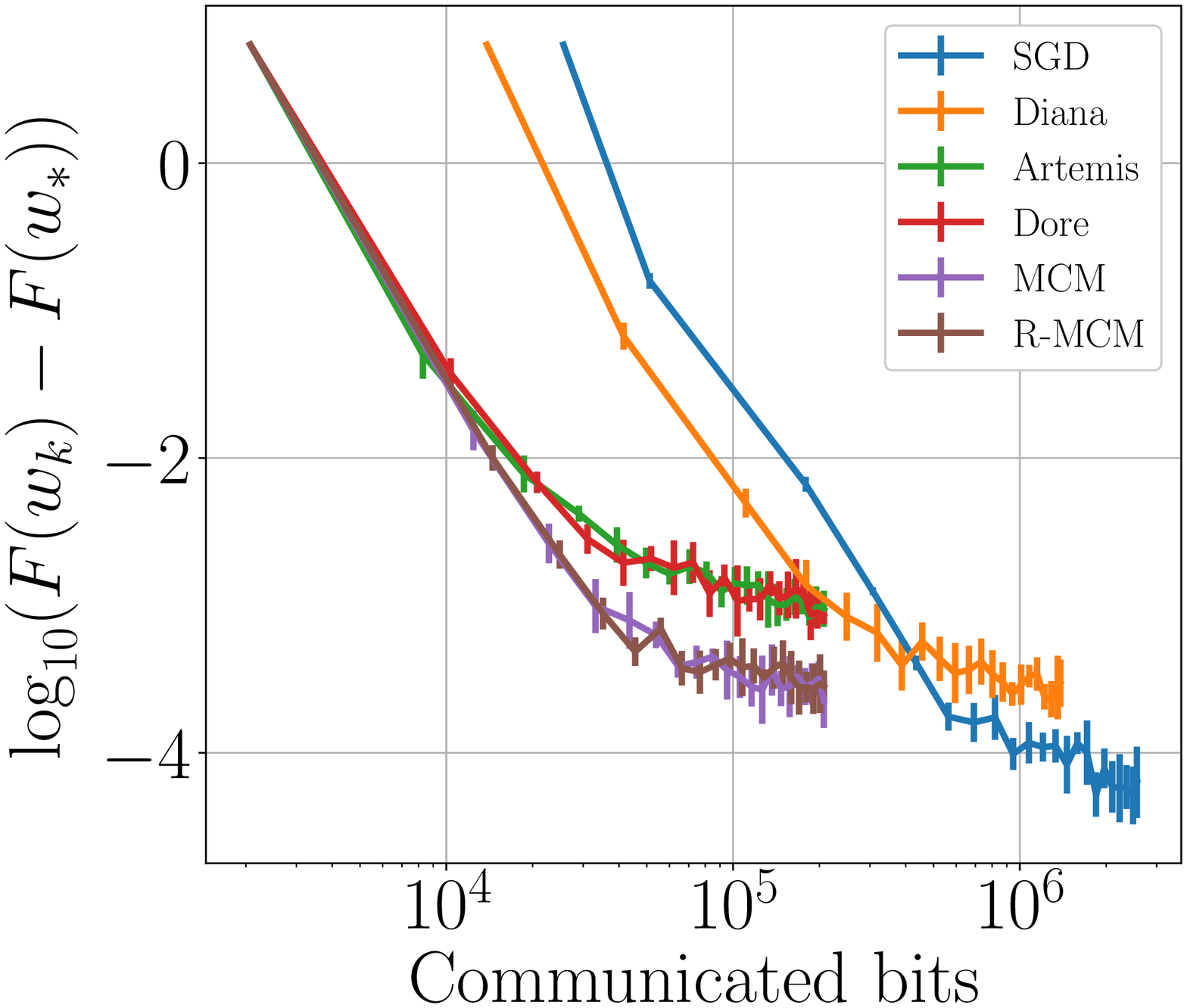}
        \caption{$\sigma^2\neq0$, $\gamma\!=(L\sqrt{k})^{-1}$\vspace{-0.5em}}
        \label{fig:LSR_sto}
    \end{subfigure}
    \begin{subfigure}{0.24\linewidth}
        \centering
        \includegraphics[width=1\textwidth]{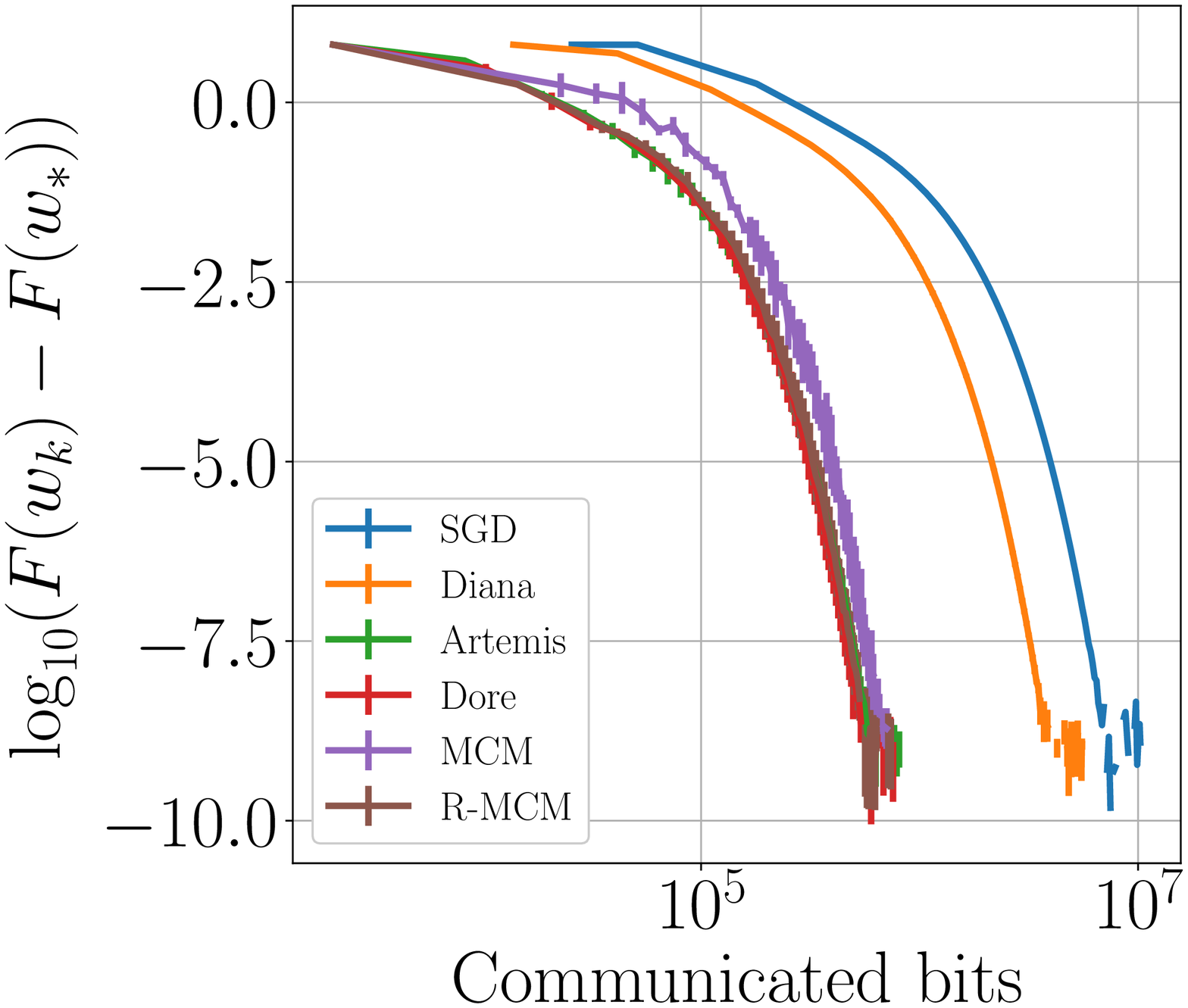}
        \caption{$\sigma^2\!=\!0$, $\gamma\!=\!L^{-1}$\vspace{-0.5em}}
        \label{fig:LSR_full}
    \end{subfigure}
    \begin{subfigure}{0.24\linewidth}
        \centering
        \includegraphics[width=1\textwidth]{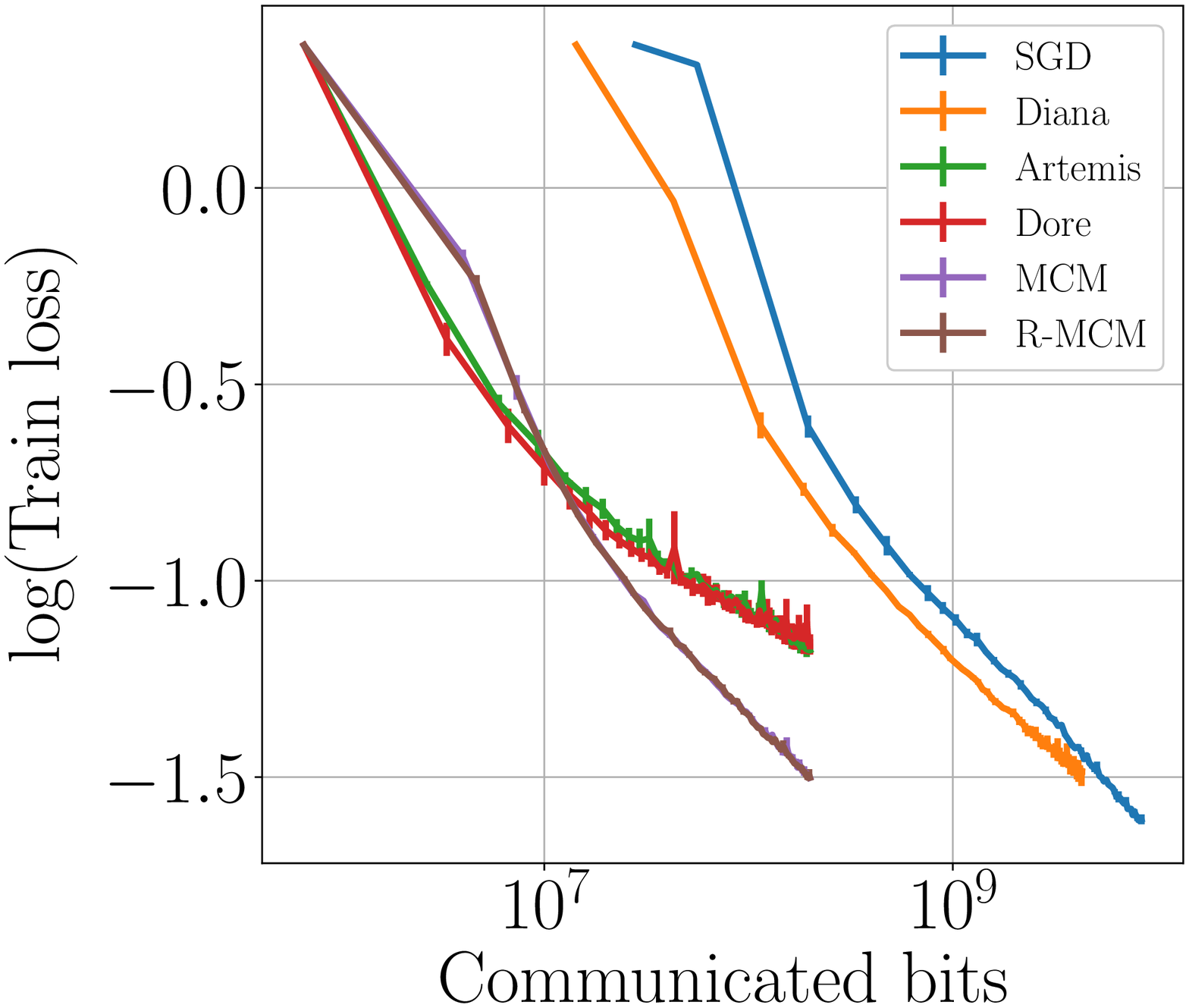}
        \caption{MNIST with a CNN\vspace{-0.5em}}
        \label{fig:MNIST}
    \end{subfigure}
    \begin{subfigure}{0.24\linewidth}
        \centering
        \includegraphics[width=1\textwidth]{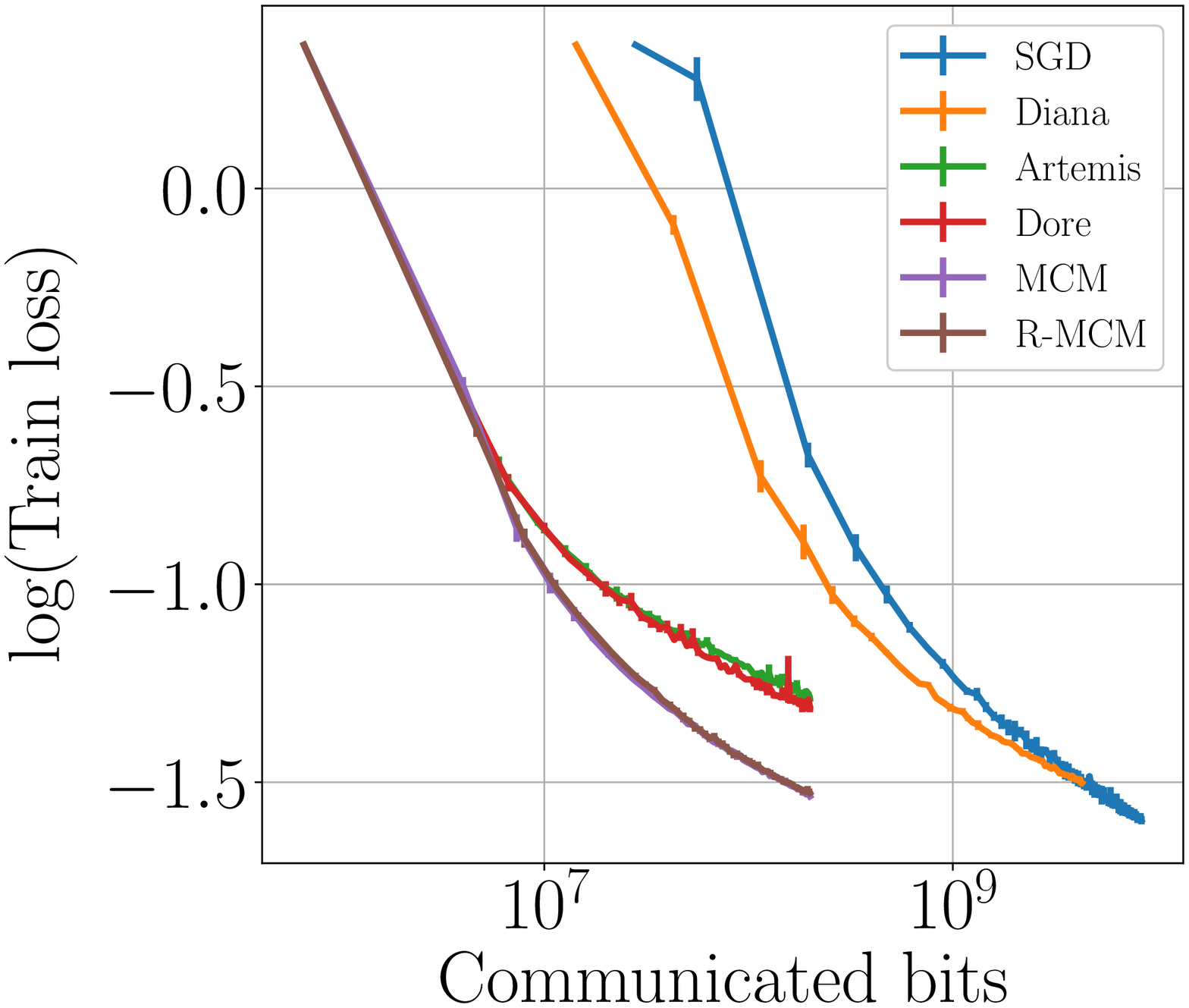}
        \caption{FE-MNIST with a CNN\vspace{-0.5em}}
        \label{fig:FEMNIST}
    \end{subfigure}
    \caption{Convergence on neural networks.}
    \label{fig:neural_net}
\end{figure}

\begin{figure}
    \centering
    \begin{subfigure}{0.23\linewidth}
        \centering
        \includegraphics[width=1\textwidth]{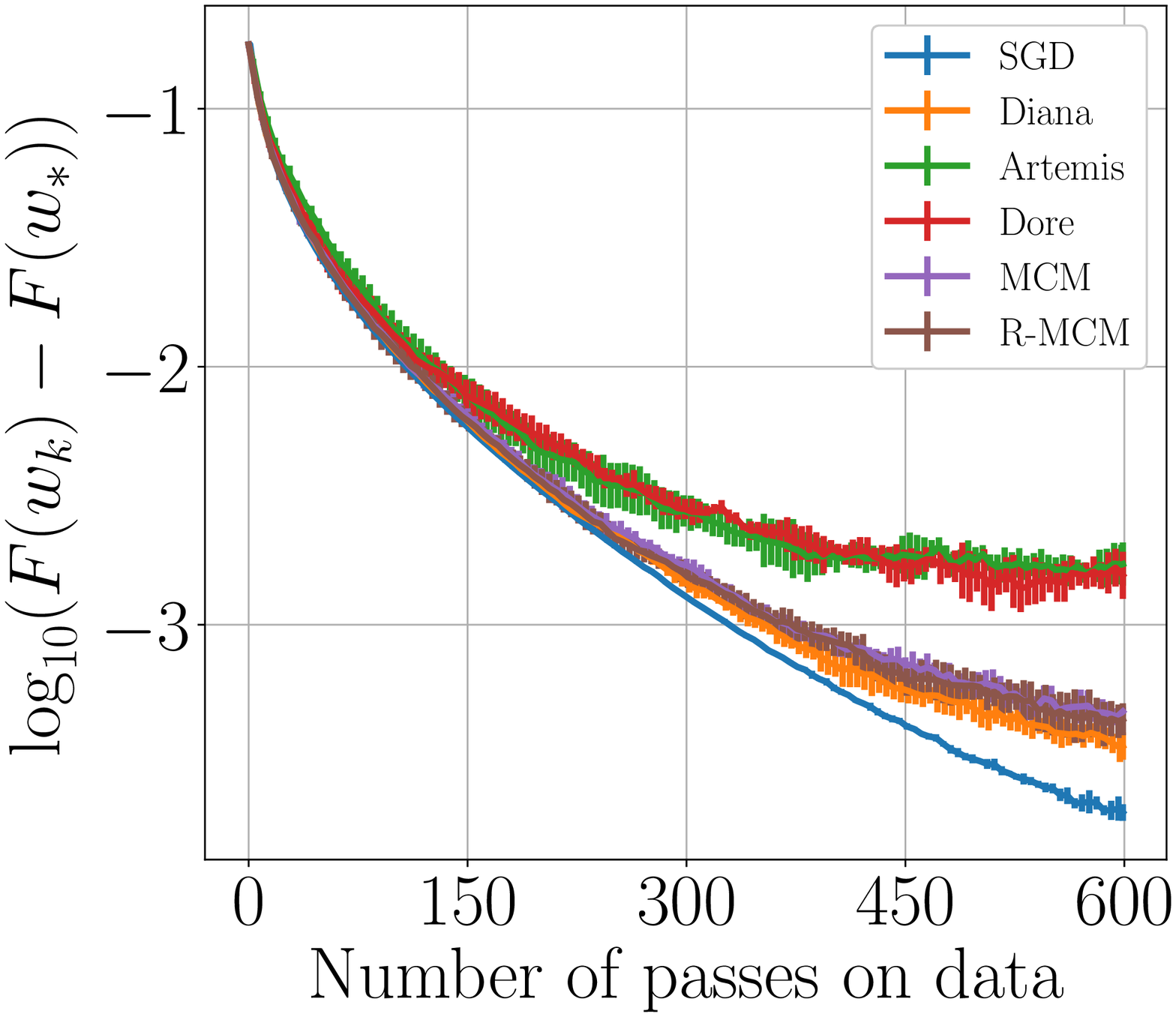}
        \caption{Quantum in \#iter.\vspace{-0.5em}}
        \label{fig:quantum_it}
    \end{subfigure}
    \begin{subfigure}{0.23\linewidth}
        \centering
        \includegraphics[trim=0cm 0 0 0,clip,width=1\textwidth]{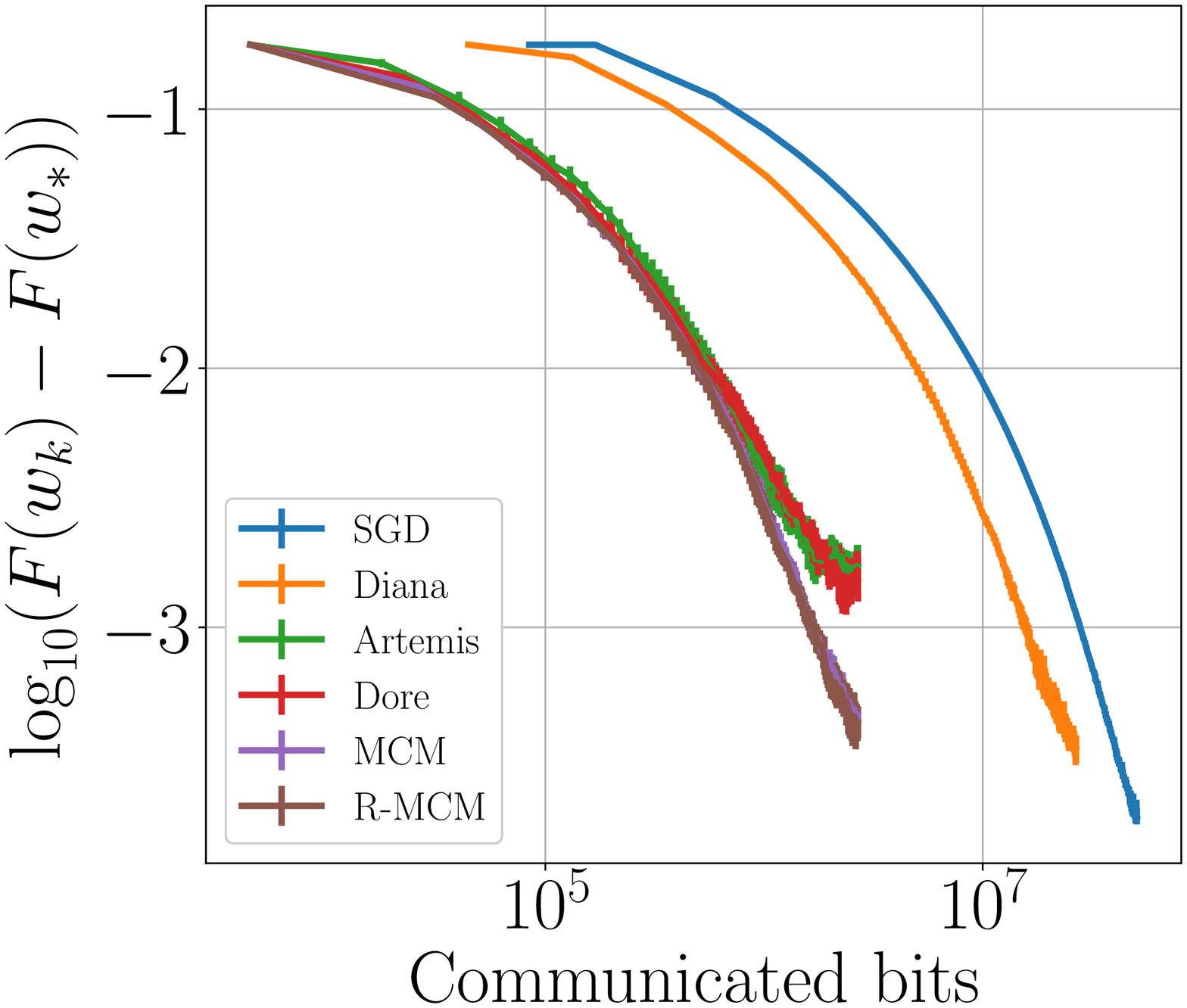} 
        \caption{Quantum in \#bits\vspace{-0.5em}}
        \label{fig:quantum_bits}
    \end{subfigure}
    \begin{subfigure}{0.23\linewidth}
        \centering
        \includegraphics[trim=0cm 0 0 0,clip,width=1\textwidth]{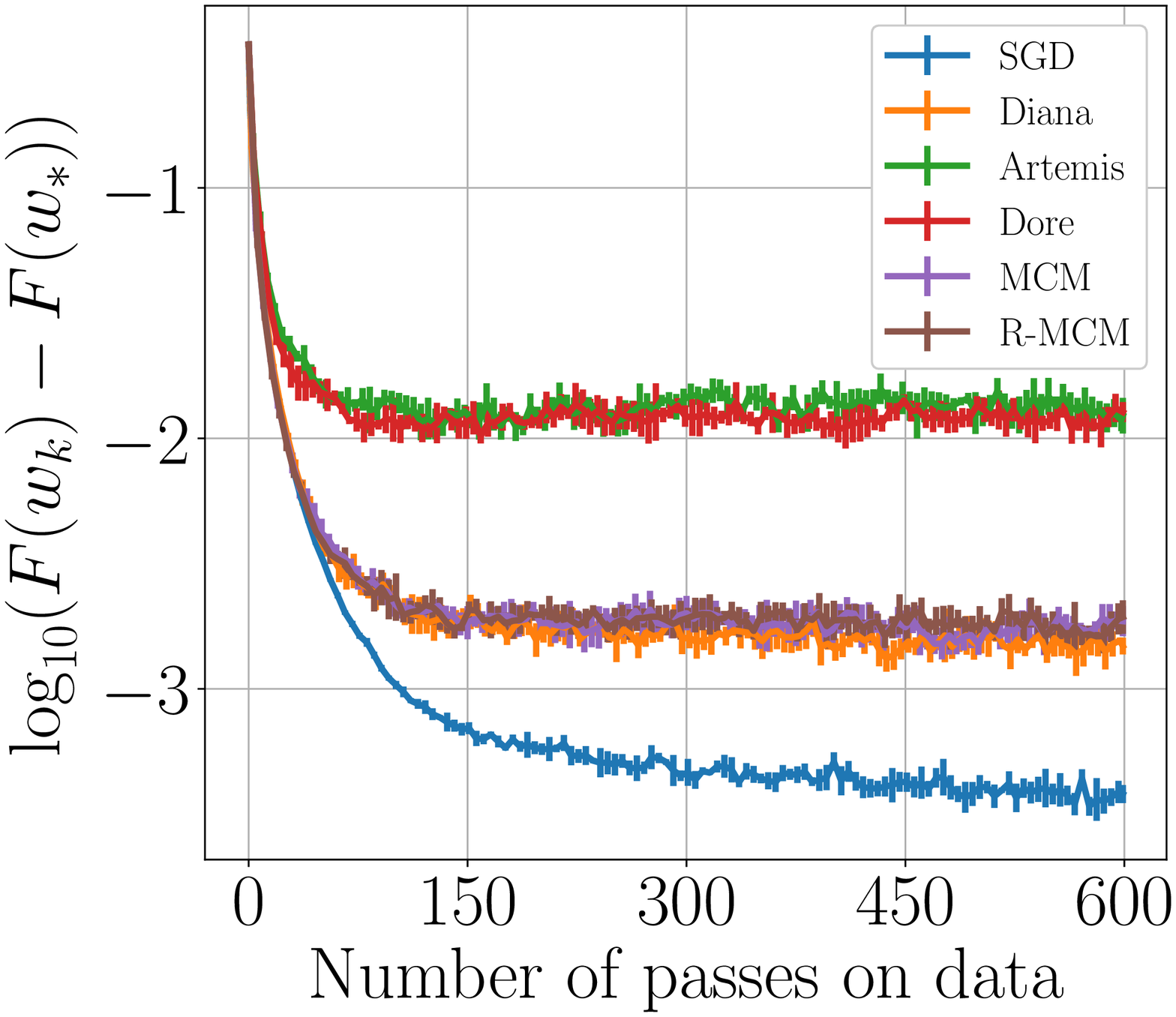}
        \caption{A9A in \#iter.\vspace{-0.5em}}
        \label{fig:a9a_it}
    \end{subfigure}
    \begin{subfigure}{0.23\linewidth}
        \centering
        \includegraphics[trim=0cm 0 0 0,clip,width=1\textwidth]{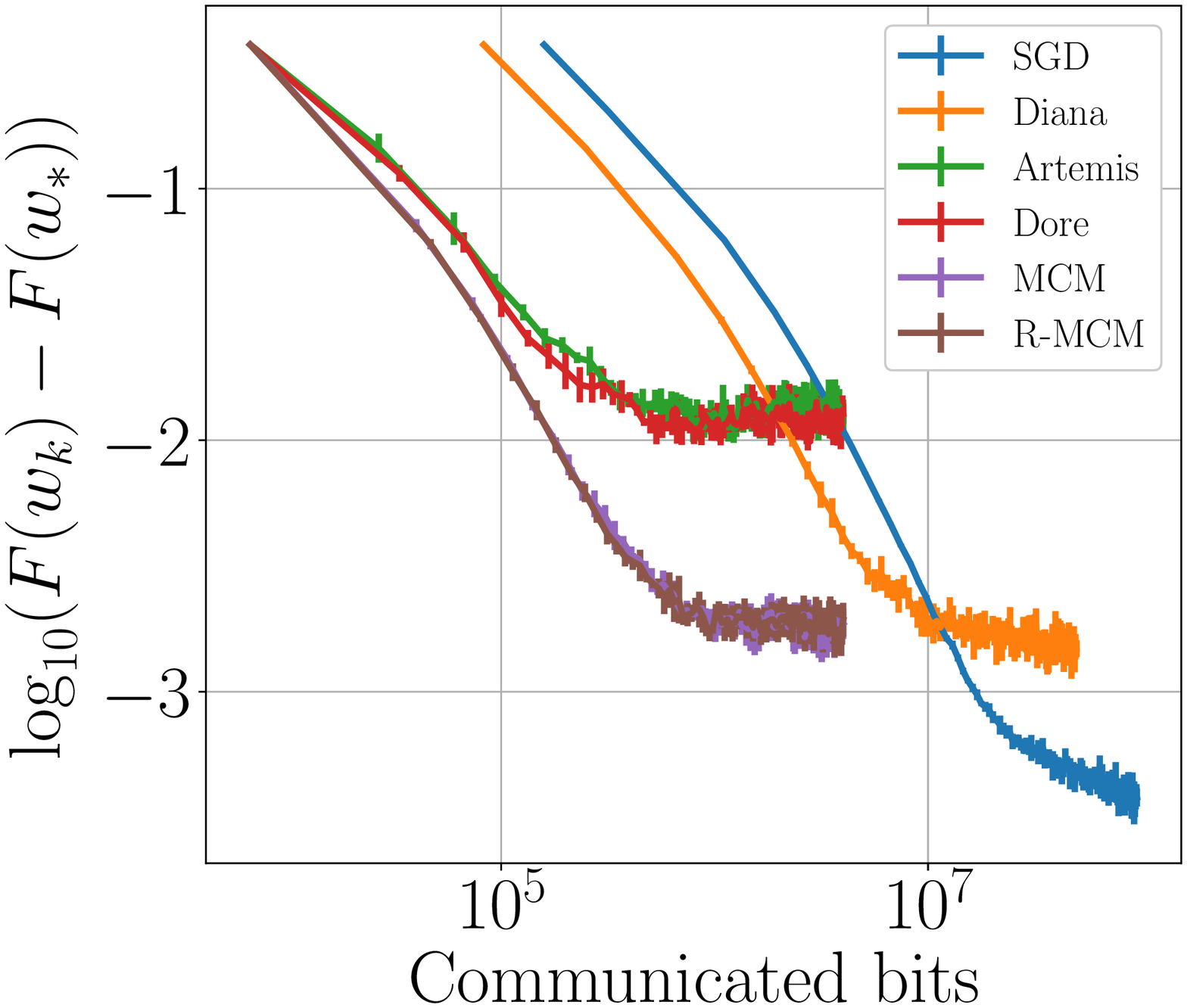}
        \caption{A9A in \#bits\vspace{-0.5em}}
        \label{fig:a9a_bits}
    \end{subfigure}
    \caption{Experiments on real dataset with $\gamma = 1/L$, quantization with $s=1$, LSR (a,b), LR (c,d).\vspace{-1.5em}}
    \label{fig:real_dataset}
\end{figure}

\begin{wrapfigure}[12]{R}{0.25\textwidth}
\flushright
    \gs\gs\gs
    \begin{center}
    \includegraphics[trim=2.2cm 0 0 0, clip,width=0.24\textwidth]{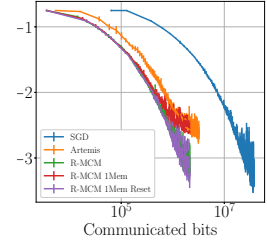}
    \end{center}
    \caption{\RMCM~(PP) on quantum with a \textit{single memory} ($s=2$).\vspace{-1.5em}}
    \label{fig:one_mem}
\end{wrapfigure}

\textbf{Saturation level.} All experiments are performed with a \textit{constant learning rate} $\gamma$ to observe the bias (initial reduction) and the variance (saturation level) independently. Stochastic gradient descent results in a   fast convergence during the first iterations, and then reaches a saturation at a given level proportional to $\sigma^2$. \Cref{thm:cvgce_mcm_convex} states that the variance of \MCM~is proportional to $\omgC\up$, this is experimentally observed on \Cref{tab:exp_nonconvex,tab:exp_convex,fig:neural_net,fig:real_dataset}: \MCM~meets \Diana~while~\Artemis~and \Dore~saturate at a higher level (scaling as $\omgC\up \times \omgC\dwn$). These trade-offs  are preserved with optimized learning rates. 

\textbf{Linear convergence when $\sigma^2 = 0$.} The six algorithms  present a linear convergence when $\sigma^2 = 0$. This is illustrated by \Cref{fig:LSR_full}: we ran experiments with a full gradient descent. Note that in these settings \MCM~has a slightly worse performance than other methods; however, this slow-down is compensated by \RMCM.

\textbf{Impact of randomization.} The impact of randomization is noticeable on \Cref{app:fig:superconduct_it,fig:LSR_full}. Randomization helps to stabilise convergence of it reduces the variance of the runs and when $\sigma^2=0$, it performs identically to \SGD. \Cref{fig:one_mem} illustrates the impact of using a single memory, instead of $N$, to alleviate the memory cost in the  PP setting  (\Cref{sec:communication_trade_offs}),  with or without periodic reset. Without reset, performance are slightly degraded, but with it, we recover previous results.

\textbf{Deep learning.} \Cref{tab:exp_nonconvex,fig:FEMNIST,fig:MNIST} illustrate experiments with neural networks, details on dataset settings and networks architecture are given in \Cref{app:subsec:deep_learning}. Again, \MCM~meets \Diana~rates as stated by \Cref{app:thm:mcm_non_convex} (theorem in the non-convex case). 

\gs\gs
\begin{table}[!htp]
\centering
\caption{Accuracy and train loss in non-convex experiments, detailed settings can be found in \Cref{app:tab:settings_nonconvex}.}
\label{tab:exp_nonconvex}
\begin{tabular}{p{0.16\textwidth}p{0.1\textwidth}cccc}
\toprule
 & Algorithm & MNIST & Fashion MNIST & FE-MNIST & CIFAR-10 \\
\midrule
Accuracy after &  \SGD: & $99.0\%$ & $92.4\%$ & $99.0\%$ & $69.1\%$   \\
$300$ epochs & \cellcolor{green!25}\Diana: & \cellcolor{green!25}$98.9\%$ & \cellcolor{green!25}$92.4\%$ & \cellcolor{green!25}$98.9\%$ & \cellcolor{green!25}$64.0\%$  \\
& \cellcolor{green!25}\MCM: & \cellcolor{green!25}$98.8\%$& \cellcolor{green!25}$90.6\%$ & \cellcolor{green!25}$98.9\%$ & \cellcolor{green!25}$63.5\%$     \\
& \Artemis: & $97.9\%$ & $86.7\%$ & $98.3\%$ & $54.8\%$ \\
& \Dore: & $97.9\%$ & $87.9\% $ & $98.5\%$ & $56.3\% $  \\
Train loss after & \SGD: & $0.025$ & $0.093$ & $0.026$ & $0.909$ \\
 $300$ epochs & \cellcolor{green!25} \Diana: & \cellcolor{green!25}$0.034$ & \cellcolor{green!25}$0.141$ & \cellcolor{green!25}$0.031$ & \cellcolor{green!25}$1.047$ \\
&  \cellcolor{green!25} \MCM: & \cellcolor{green!25}$0.033$ & \cellcolor{green!25}$0.209$ & \cellcolor{green!25}$0.030$ & \cellcolor{green!25}$1.096$ \\
&  \Artemis: & $0.075$  & $0.332$ &  $0.052$ & $1.342$ \\
& \Dore: & $0.072$ & $0.300$ & $0.048$ & 1.292 \\
\bottomrule
\end{tabular}

\gs
\end{table}

Overall, these experiments show the benefits of \MCM~and \RMCM, that reach the saturation level of \Diana~while exchanging at 10x to 100x fewer bits. More experiments with  partial participation for \RMCM~are given in  \Cref{app:subsec:expe_partial_part}. All the code is provided on our \href{https://github.com/philipco/mcm-bidirectional-compression/}{github repository}.
\gs
\gs

\section{Conclusion}\gs
\label{sec:conclusion}

In this work, we propose a new algorithm to perform bidirectional compression while achieving the convergence rate of  algorithms using compression in a single direction.
One of the main application of this framework is Federated Learning. 
With \MCM~we stress the importance of not degrading the global model. In addition, we add the concept of randomization which allows to reduce the variance associated with the downlink compression. 
The analysis of \MCM~is challenging as the algorithm involves perturbed iterates.  Proposing such an analysis is the key to unlocking numerous challenges in distributed learning, e.g., proposing practical algorithms for partial participation, incorporating privacy-preserving schemes \textit{after} the global update is performed, dealing with local steps, etc. This approach could also be pivotal in non-smooth frameworks, as it can be considered as a weak form of randomized smoothing.

\section*{Acknowledgments}
We would like to thank Richard Vidal, Laeticia Kameni from Accenture Labs (Sophia Antipolis, France) and Eric Moulines from École Polytechnique for insightful discussions. This research was supported by the \emph{SCAI: Statistics and Computation for AI} ANR Chair of research and teaching in artificial intelligence, by \textit{Hi!Paris}, and by \emph{Accenture Labs} (Sophia Antipolis, France).

\bibliography{main}

\newpage	
	
\newpage	
\onecolumn
\appendix

\begin{center}
	{\Large{\bf Supplementary material}}
	\end{center}
	
	\setcounter{equation}{0}
	\setcounter{figure}{0}
	\setcounter{table}{0}
	\renewcommand{\theequation}{S\arabic{equation}}
	\renewcommand{\thefigure}{S\arabic{figure}}
	\renewcommand{\thetheorem}{S\arabic{theorem}}
	\renewcommand{\thelemma}{S\arabic{lemma}}
	\renewcommand{\theproposition}{S\arabic{proposition}}
	\renewcommand{\thecorollary}{S\arabic{proposition}}
	\renewcommand{\thetable}{S\arabic{table}}
	
	In this appendix, we provide additional details about our work. First, in \Cref{app:sec:complementary} we give complementary references on operators of compression and on perturbed iterate analysis. We also give the pseudo-code of \RMCM. Secondly, in \Cref{app:sec:experiments} we enlarge figures provided in \Cref{sec:experiments} and complete them with experiments on partial participation and with a comparison between \MCM~and other algorithms using non-degraded updates. The next sections are all devoted to theoretical results.
	In \Cref{app:sec:technical_results} we detail some technical results required to demonstrate \Cref{thm:cvgce_mcm_strongly_convex,thm:cvgce_mcm_convex,thm:cvgce_rmcm,thm:rmcm_quad_guarantee,thm:contraction_mcm}, in \Cref{app:sec:proof_for_ghost} we highlight the key stages of the demonstration in the easier case of \ghost, in \Cref{app:sec:proofs_mcm} we completely prove the given guarantees of convergence in three regimes: convex, strongly-convex and non-convex. In \Cref{app:sec:proofs_quad} we show the benefit of \RMCM~compared to \MCM~in the context of quadratic functions. In \Cref{app:sec:adaptation_to_heterogeneous_case} we adapt the proof to the heterogeneous scenario. And finally, in \Cref{app:sec:checklist} we answer to the Neurips checklist. 
	
	\hypersetup{linkcolor = black}
	\setlength\cftparskip{2pt}
	\setlength\cftbeforesecskip{2pt}
	\setlength\cftaftertoctitleskip{3pt}
	\addtocontents{toc}{\protect\setcounter{tocdepth}{2}}
	\setcounter{tocdepth}{1}
	\tableofcontents
	\hypersetup{linkcolor=blue}
	
	\addtocontents{toc}{
    \protect\thispagestyle{empty}} 
    \thispagestyle{empty}

\section{Complementary discussions and references}
\label{app:sec:complementary}

We give the pseudo-code of \RMCM~in \Cref{algo}. It summarizes the algorithm's description given in \Cref{sec:intro}.

\begin{wrapfigure}[29]{R}{0.58\textwidth}
\flushright
\vspace{-1.6em}
\begin{minipage}{1\linewidth}
\begin{algorithm}[H]
\caption{Pseudocode of \RMCM}
\label{algo}
\begin{algorithmic}
  \STATE {\bfseries Input:} Mini-batch size $b$, learning rates $\alpha\up, \alpha\dwn, \gamma > 0$, initial model  $w_0 \in \WW$ (on all devices),  operators $\C_{\up}$ and $\C_{\dwn}$, $S = \llbracket 1, N \rrbracket$ the set of devices.
  \STATE {\bfseries Init.:} Memories: $\forall i \in S$, $h_0^i = \g^i_{1}(w_{0})$ and $H_{-1}^i = w_0$
  \STATE {\bfseries Output:} Model  $w_K$
  \FOR{$k = 1, 2, \dots, K $}
    \FOR{each device $i=1, 2, 3, \dots, N$}
		\STATE Receive $\widehat{\Omega}_{k-1}^i$, and set: $w_{k-1}^i = \widehat{\Omega}_{k-1}^i + H_{k-2}^i$
		\STATE Compute $\g^i_{k}(w_{k-1}^i)$ (with mini-batch)
		\STATE Update down memory: $H_{k-1}^i = H_{k-2}^i + \alpha\dwn \widehat{\Omega}_{k-1}^i$
		\STATE Up compr.: $\widehat{\Delta}_{k-1}^i = \C_{\up}(\g^i_k(w_{k-1}^i) - h_{k-1}^i)$ 
		\STATE \textbf{Update uplink memory}: $h_{k}^i = h_{k-1}^i + \alpha\up \widehat{\Delta}_{k-1}^i$
		\STATE Send $\widehat{\Delta}_{k-1}^i$ to central server
	\ENDFOR
	\STATE Receive $(\widehat{\Delta}_{k-1}^i)_\iN$ from all remote servers
	\STATE Compute $\widehat{\gw_k} = \ffrac{1}{N} \sum_\iN \widehat{\Delta}_{k-1}^i + h_{k-1}^i$ \;
	\STATE Update up memory: $\forall i \in S, h_{k}^i = h_{k-1}^i + \alpha\up \widehat{\Delta}_{k-1}^i$ \;
	\STATE Non-degraded update: $w_{k} = w_{k-1} - \gamma \widehat{\gw_k}$ \;
    \STATE Down compr.: $\forall i \in S,\, \widehat{\Omega}_{k}^i = \C_{\dwn, i}(w_{k} - H_{k-1}^i)$
    \STATE \textbf{Update downlink memory:} $H_{k}^i = H_{k-1}^i + \alpha\dwn \widehat{\Omega}_k^i$
    \STATE Send $(\widehat{\Omega}_{k}^i)_\iN$ to all remote servers
    \ENDFOR
\end{algorithmic}
\end{algorithm}
\end{minipage}
\end{wrapfigure}

\subsection{Compression Operators}
\label{app:subsec:compression_oper}

In this section, we give additional details on compression operators (see \Cref{asu:expec_quantization_operator}).

Operators of compression can be biased or unbiased and they may have drastically different impacts on convergence. 
For instance, if the operator is not contracting, algorithms with error-feedback may diverge. \citet{horvath_better_2020} propose a method to unbiase a biased operator and a general study of biased operator has been carried out by \citet{beznosikov_biased_2020}. 
But in this work, as stated by \Cref{asu:expec_quantization_operator}, we consider only unbiased operators: for instance \texttt{s-quantization}.

The choice of the operator of compression is crucial when compressing data. 
Operators of compression may be classified into three mains categories: 1) sparsification \cite{stich_sparsified_2018,hu_sparsified_2020,khirirat_communication_2020,alistarh_convergence_2018,khirirat_communication_2020,mayekar_ratq_2020} 2) quantization \cite{seide_1-bit_2014,zhou_dorefa-net_2018,alistarh_qsgd_2017,horvath_stochastic_2019,wen_terngrad_2017} and 3) sketching \cite{ivkin_communication-efficient_2019,li_privacy_2019}.

\paragraph{Possible Extensions}
Our analysis could be extended to biased uplink operators, following similar lines of proof as \cite{beznosikov_biased_2020}. 

The extension for the downlink operator seems more difficult as our analysis relies on numerous occurrences on the fact that the expectation of $\wkmhat$ knowing $\wkm$ is $\wkm$.

\subsection{Relation to Randomized Smoothing}
\label{app:subsec:smoothing}

Our approach can also be related to randomized smoothing. Formally, $\nabla F (\wkmhat)$ can be considered as an unbiased gradient of the smoothed function $F_\rho$ at point $\wkm$, with $F_\rho : w \mapsto \E[F(w + \wkmhat - \wkm )]$. Then $\E\PdtScl{\nabla F (\wkmhat)}{\wkm - w_*} = \E \PdtScl{\nabla F_\rho (\wkm)}{\wkm - w_*}$. One key aspect is that the condition number $\mu_\rho/L_\rho$ of $F_\rho$ is always larger (better) than the one for $F$.
However, the minimum of $F_\rho$  is different and moving, thus the proof techniques from Randomized smoothing are not adapted to a varying noise which distribution is unknown. Providing a theoretical result that quantifies the smoothing impact of \MCM~is an interesting open direction.

Randomized smoothing has been applied to non-smooth problems by \citet{duchi_randomized_2012}. The aim is to transform a non-smooth function into a smooth function, before computing the gradient. This is achieved by adding a Gaussian noise to the point where the gradient is computed. This mechanism has been applied by \citet{scaman_optimal_2018} to convex problems. 
We consider in this work a randomized version of compression: at iteration $k$ in $\N$ each worker $i$ in  $\llbracket 1, N \rrbracket$ receives a noisy estimate $\widehat{w}_k^i$ of the global model $w_k$ kept on central server. Thus, we compute the local gradient at a perturbed point $w_k + \delta_k^i$. Unlike the randomization process as defined by \citet{duchi_randomized_2012}, the noise here is not chosen to improve the function's regularity but results from the compression.

\section{Experiments}
\label{app:sec:experiments}
In this section we provide additional details about our experiments. We first give the settings of our experiments in \Cref{app:tab:settings_nonconvex,app:tab:settings_convex}. Next, we describe the numerical results obtained on our $9$ datasets. Thirdly, we add some explanation concerning the wall clock time. Finally, we provide an estimation of the carbon footprint required by this paper.

We use the same operator of compression for uplink and downlink, thus we consider that $\omgC\up=\omgC\dwn$. In addition, we choose $\alpha\up = \alpha\dwn = \ffrac{1}{2(1+\omgC_{\mathrm{up}/\mathrm{dwn}})}$. 

\textbf{Convex settings} are given in \Cref{app:tab:settings_convex}. We obtain non-i.i.d.~data distributions by computing a TSNE representation \citep[defined in][]{maaten_visualizing_2008} followed by a clustering. Experiments have been performed with $600$ epochs. Apart from the case of partial participation, we use quantization \citep[defined in][]{alistarh_qsgd_2017} with $s = 2^0$.

\begin{table}[!htp]
\caption{Settings of experiment in the convex mode.}
\label{app:tab:settings_convex}
\centering
\begin{tabular}{lccccc}
Settings & a9a & quantum & phishing & superconduct & w8a \\
\hline 
references & \cite{chang_libsvm_2011} & \cite{caruana_kdd-cup_2004} & \cite{chang_libsvm_2011} & \cite{hamidieh_data-driven_2018} &\cite{chang_libsvm_2011} \\
model & LR & LR & LSR & LR & LR \\
dimension $d$ & $124$ & $66$ & $69$ & $82$ & $301$ \\
training dataset size & $ 32,561$ & $50,000$ & $11,055$ & $21,200$ & $49,749$ \\
batch size $b$& $50$ & $400$ & $50$ & $50$ & $12$ \\
compression rate $s$ & \multicolumn{5}{c}{$2^0$ (\textit{i.e.} two levels)}\\
norm quantization  &\multicolumn{5}{c}{$\| \cdot \|_{2}$}  \\
momentum $m$ & \multicolumn{5}{c}{no momentum}\\
step size $\gamma$ & \multicolumn{5}{c}{$1/L$} \\
\bottomrule
\end{tabular}
\end{table}

\textbf{Deep-learning settings} are provided in \Cref{app:tab:settings_nonconvex}. All experiments have been performed with $300$ epochs

\begin{table}[!htp]
\caption{Settings of experiments in the non-convex mode.}
\label{app:tab:settings_nonconvex}
\centering
\begin{tabular}{lcccc}
Settings & MNIST & Fashion-MNIST & FE-MNIST & CIFAR10 \\
\hline 
references & \cite{lecun_gradient-based_1998} & \cite{xiao_fashion-mnist_2017}& \cite{caldas_leaf_2019} & \cite{krizhevsky_learning_2009}\\
model & CNN  & Fashion CNN & CNN & LeNet \\
trainable parameters $d$ & $20\times 10^3$ & $400\times 10^3$ & $20\times 10^3$ & $62\times 10^3$ \\
training dataset size & $60,000$ & $60,000$ & $805,263$ & $60,000$ \\
compression rate $s$ & $2^2$ & $2^2$ & $2^2$ & $2^4$\\
momentum $m$ & $0$ & $0$ & $0$ & $0.9$ \\
norm quantization  &\multicolumn{4}{c}{$\| \cdot \|_{2}$}  \\
batch size $b$& \multicolumn{4}{c}{$128$} \\
step size $\gamma$ & \multicolumn{4}{c}{$0.1$} \\ 
loss & \multicolumn{4}{c}{Cross Entropy} \\
\bottomrule
\end{tabular}
\end{table}

\subsection{Convex settings}
\label{app:subsec:exp_convex}

In this section, we provide the plot of excess loss for the toy dataset, for quantum and for a9a datasets. For results on superconduct, phishing and w8a, see our \href{https://github.com/philipco/mcm-bidirectional-compression}{github repository}. For these last three datasets, we give only the excess loss w.r.t. number of iteration in the basic settings of full participation on \Cref{app:fig:phishing_superconduct_w8a}. We detail experiments in the PP settings in \Cref{app:subsec:expe_partial_part}. At the left side (resp. right side) we display the result w.r.t. the number of iterations (resp. number of communicated bits).

We provide results on the log of the excess loss $F(w_k) - F_*$,  with error bars displayed on each figure, corresponding to the standard deviation of $\log_{10}(F(w_k)-F_*)$. \Cref{app:fig:LSR_full_bits,app:fig:LSR_sto_bits,app:fig:a9a,app:fig:quantum} correspond to \Cref{fig:LSR_sto,fig:LSR_full,fig:real_dataset} given in \Cref{sec:experiments}. Additionally, we provide results for the synthetic dataset (\Cref{fig:LSR_sto,fig:LSR_full}) w.r.t~to the number of iterations in \Cref{app:fig:LSR_sto} (stochastic gradient) and \Cref{app:fig:LSR_full} (full batch gradient). As predicted by \Cref{thm:cvgce_mcm_convex}, when $\sigma=0$, we observe a linear convergence.

\begin{figure}
    \centering
    \begin{subfigure}{\sizefig\textwidth}
        \centering
        \includegraphics[width=1\textwidth]{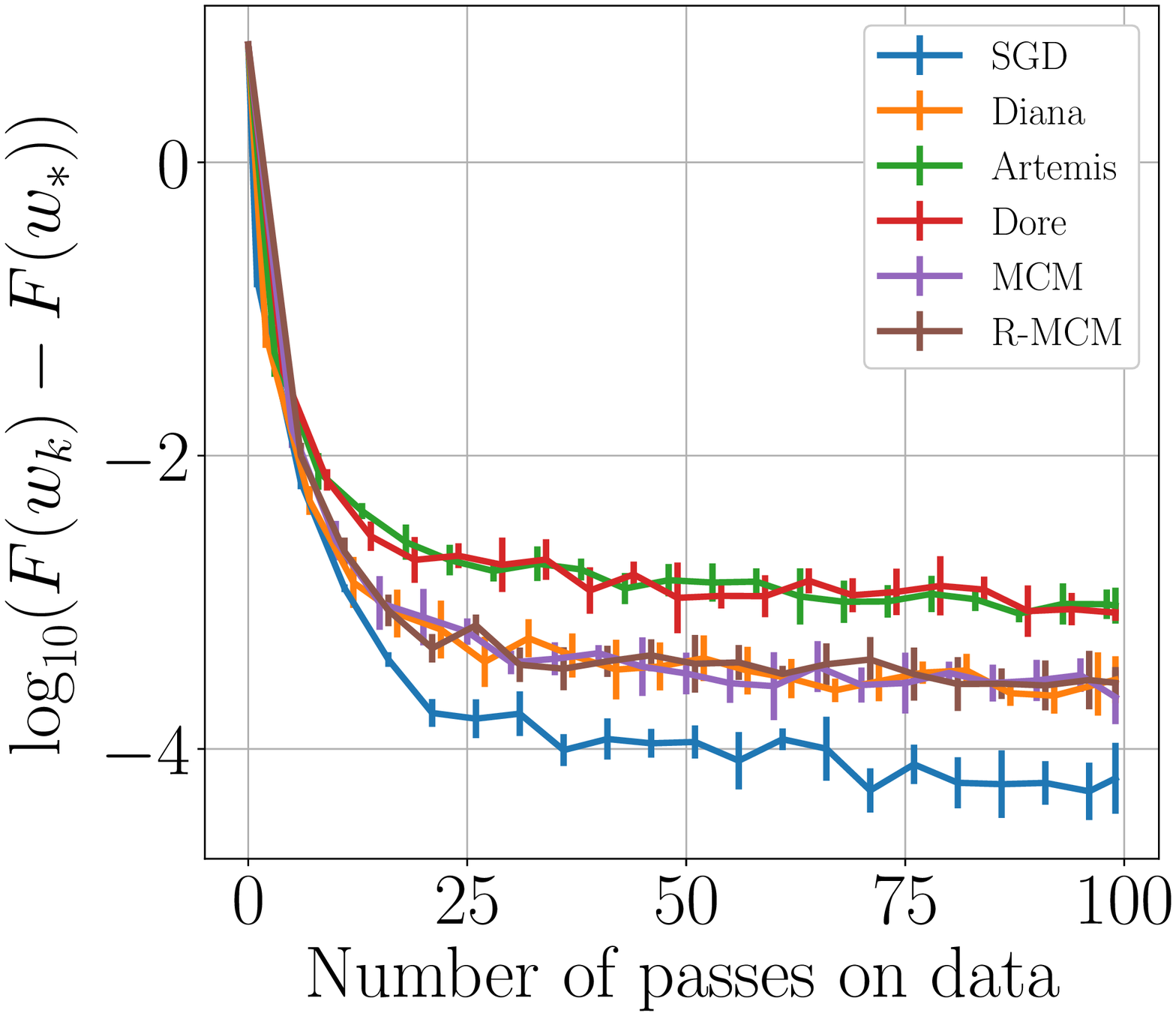}
        \caption{X axis in \# iterations.\vspace{-0.5em}}
        \label{app:fig:LSR_sto_iter}
    \end{subfigure}
    \begin{subfigure}{\sizefig\textwidth}
        \centering
        \includegraphics[width=1\textwidth]{pictures/exp/synth_noised/mcm-vs-existing/bits-noavg-sto-b1.eps}
        \caption{X axis in \# bits.\vspace{-0.5em}}
        \label{app:fig:LSR_sto_bits}
    \end{subfigure}
    \caption{Least-square regression, toy dataset:  $\gamma = (L\sqrt{k})^{-1}$, $\sigma \neq 0$. \vspace{-0.9em}}
    \label{app:fig:LSR_sto}
\end{figure}

\begin{figure}
    \centering
    \begin{subfigure}{\sizefig\textwidth}
        \centering
        \includegraphics[width=1\textwidth]{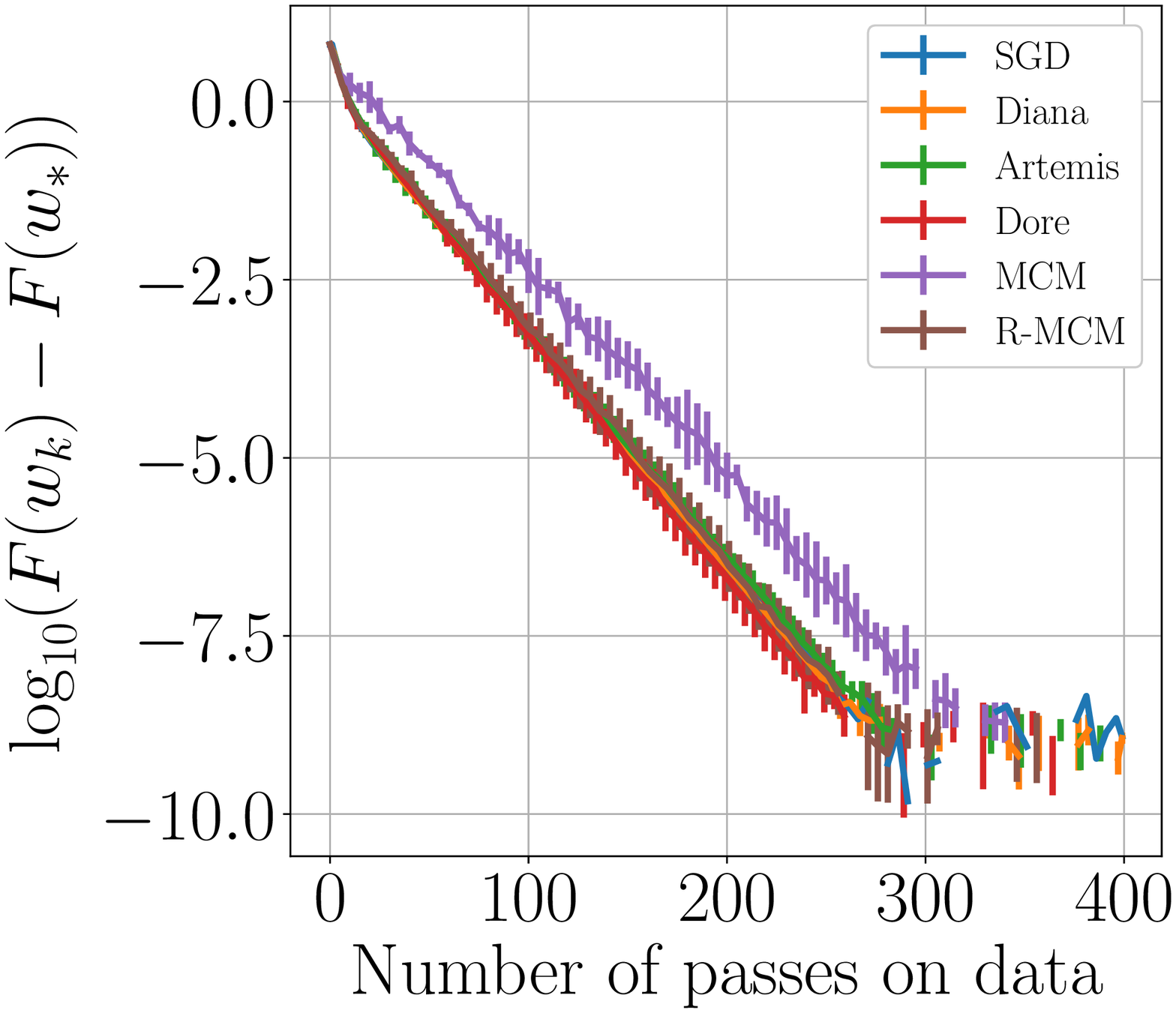}
        \caption{X axis in \# iterations.\vspace{-0.5em}}
        \label{app:fig:LSR_full_iter}
    \end{subfigure}
    \begin{subfigure}{\sizefig\textwidth}
        \centering
        \includegraphics[width=1\textwidth]{pictures/exp/synth_noised/mcm-vs-existing/bits-noavg-full.eps}
        \caption{X axis in \# bits.\vspace{-0.5em}}
        \label{app:fig:LSR_full_bits}
    \end{subfigure}
    \caption{Least-square regression, toy dataset: $\gamma = 1/L$, $\sigmstar^2 = 0$. \vspace{-0.9em}}
    \label{app:fig:LSR_full}
\end{figure}

\begin{figure}
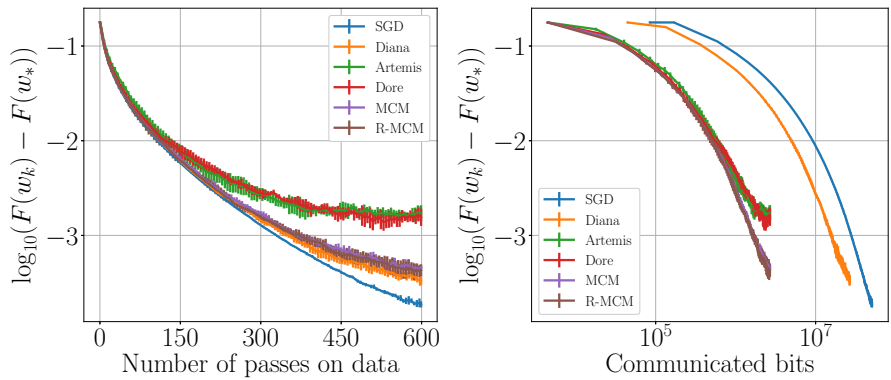

    \centering
    \begin{subfigure}{\sizefig\textwidth}
        \centering
        \includegraphics[width=1\textwidth]{pictures/exp/quantum/mcm-vs-existing/it-noavg-sto-b400.eps}
        \caption{X axis in \# iterations.\vspace{-0.5em}}
        \label{app:fig:quantum_it}
    \end{subfigure}
    \begin{subfigure}{\sizefig\textwidth}
        \centering
        \includegraphics[width=1\textwidth]{pictures/exp/quantum/mcm-vs-existing/bits-noavg-sto-b400.eps}
        \caption{X axis in \# bits.\vspace{-0.5em}}
        \label{app:fig:quantum_bits}
    \end{subfigure}
    \caption{quantum  with $b=400$, $\gamma = 1/L$.  \vspace{-0.9em}}
    \label{app:fig:quantum}
\end{figure}

\begin{figure}
    \centering
    \begin{subfigure}{\sizefig\textwidth}
        \centering
        \includegraphics[width=1\textwidth]{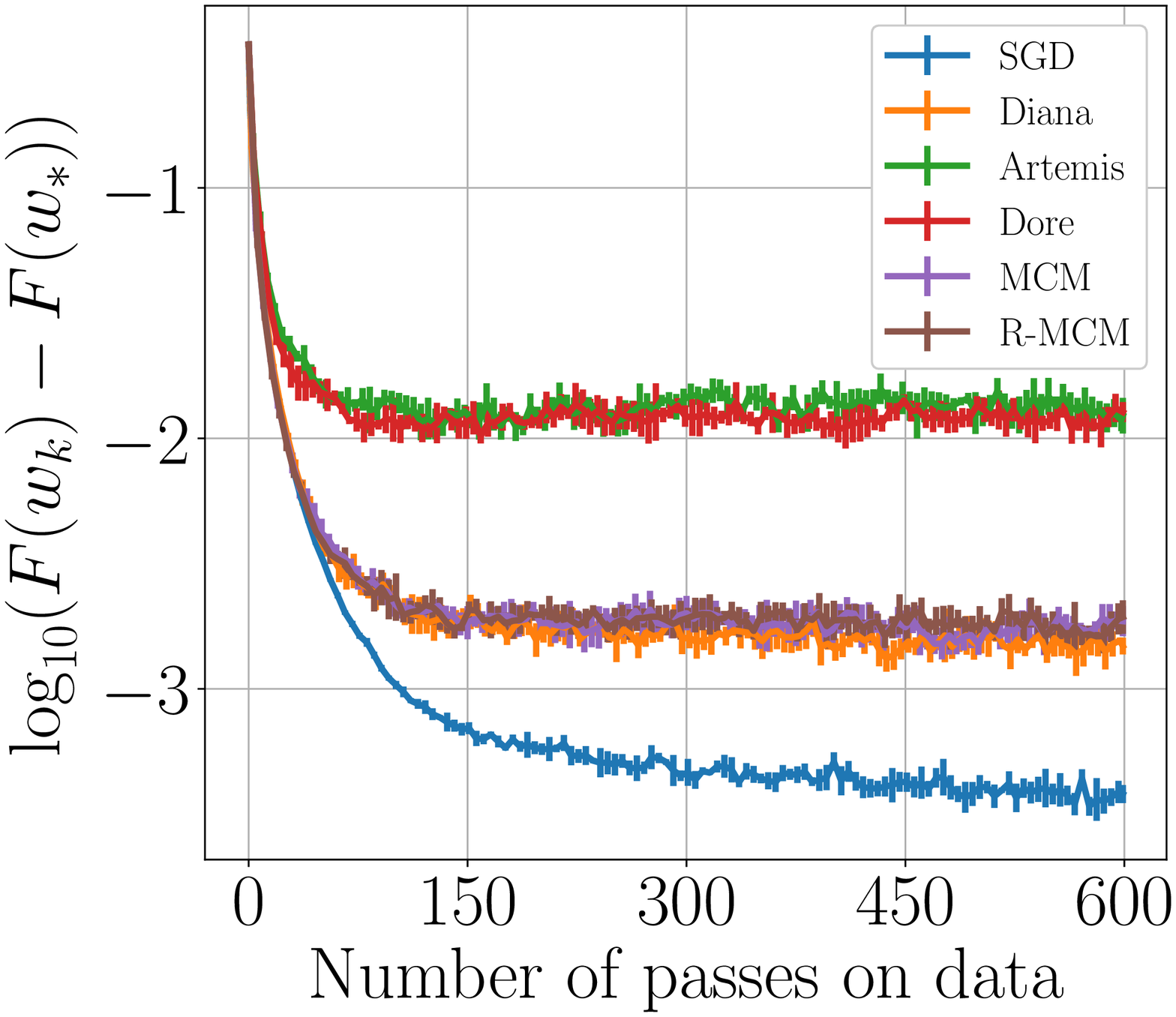}
        \caption{X axis in \# iterations.\vspace{-0.5em}}
        \label{app:fig:a9a_it}
    \end{subfigure}
    \begin{subfigure}{\sizefig\textwidth}
        \centering
        \includegraphics[width=1\textwidth]{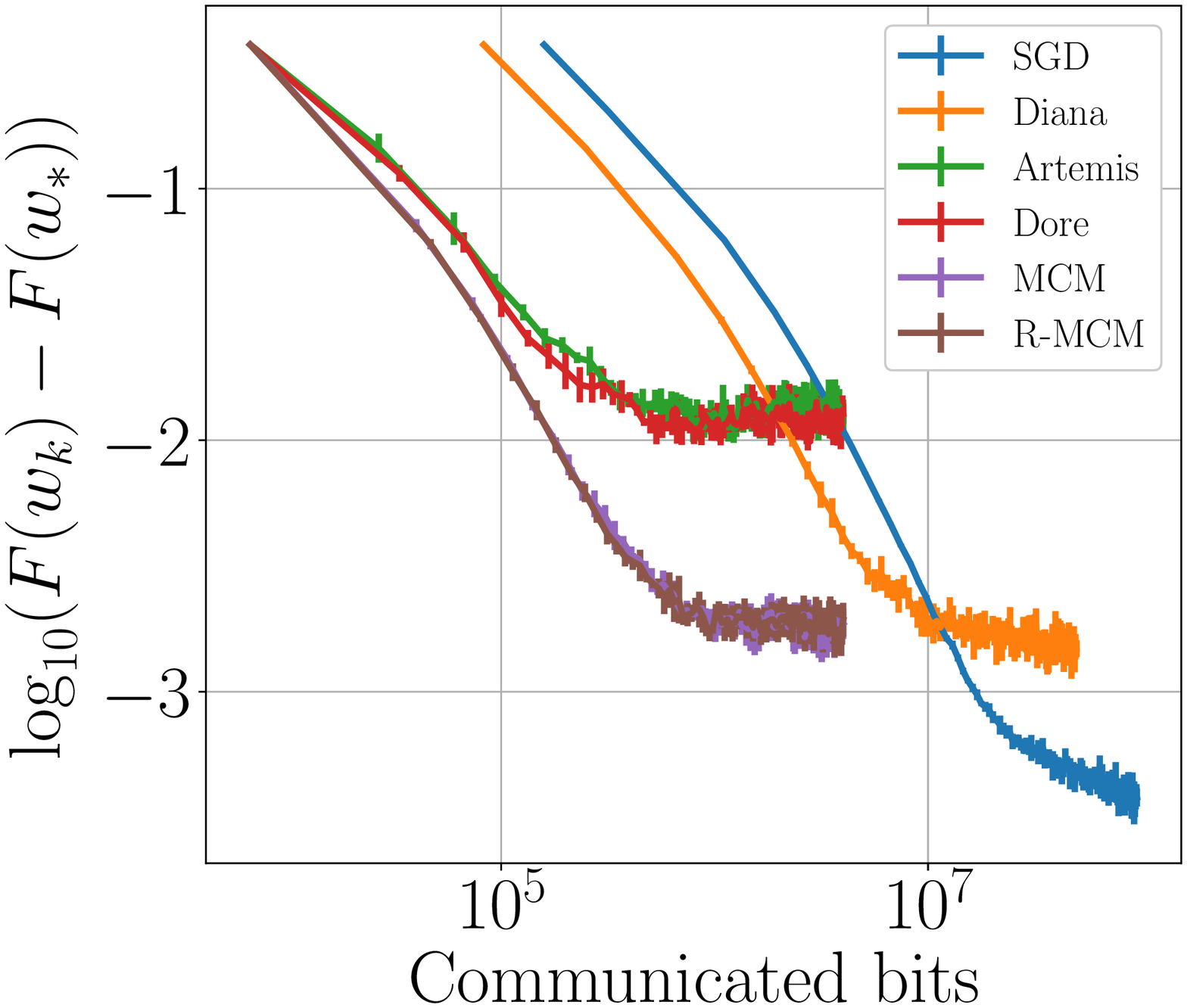}
        \caption{X axis in \# bits.\vspace{-0.5em}}
        \label{app:fig:a9a_bits}
    \end{subfigure}
    \caption{A9A  with $b=50$, $\gamma = 1/L$.  \vspace{-0.9em}}
    \label{app:fig:a9a}
\end{figure}

\begin{figure}
    \centering
    \begin{subfigure}{0.3\textwidth}
        \centering
        \includegraphics[width=1\textwidth]{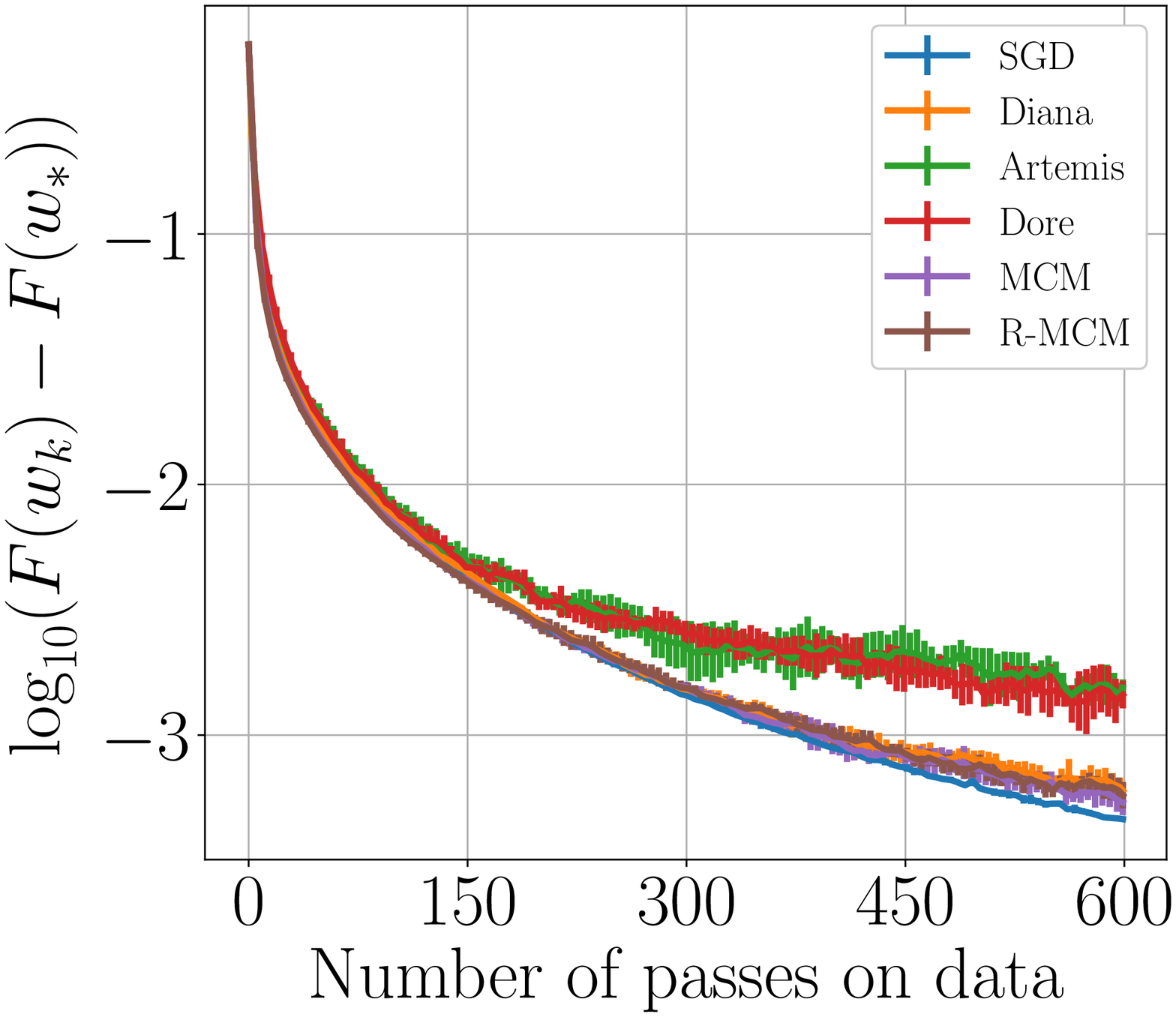}
        \caption{Phishing.\vspace{-0.5em}}
        \label{app:fig:phishing_it}
    \end{subfigure}
    \begin{subfigure}{0.3\textwidth}
        \centering
        \includegraphics[width=1\textwidth]{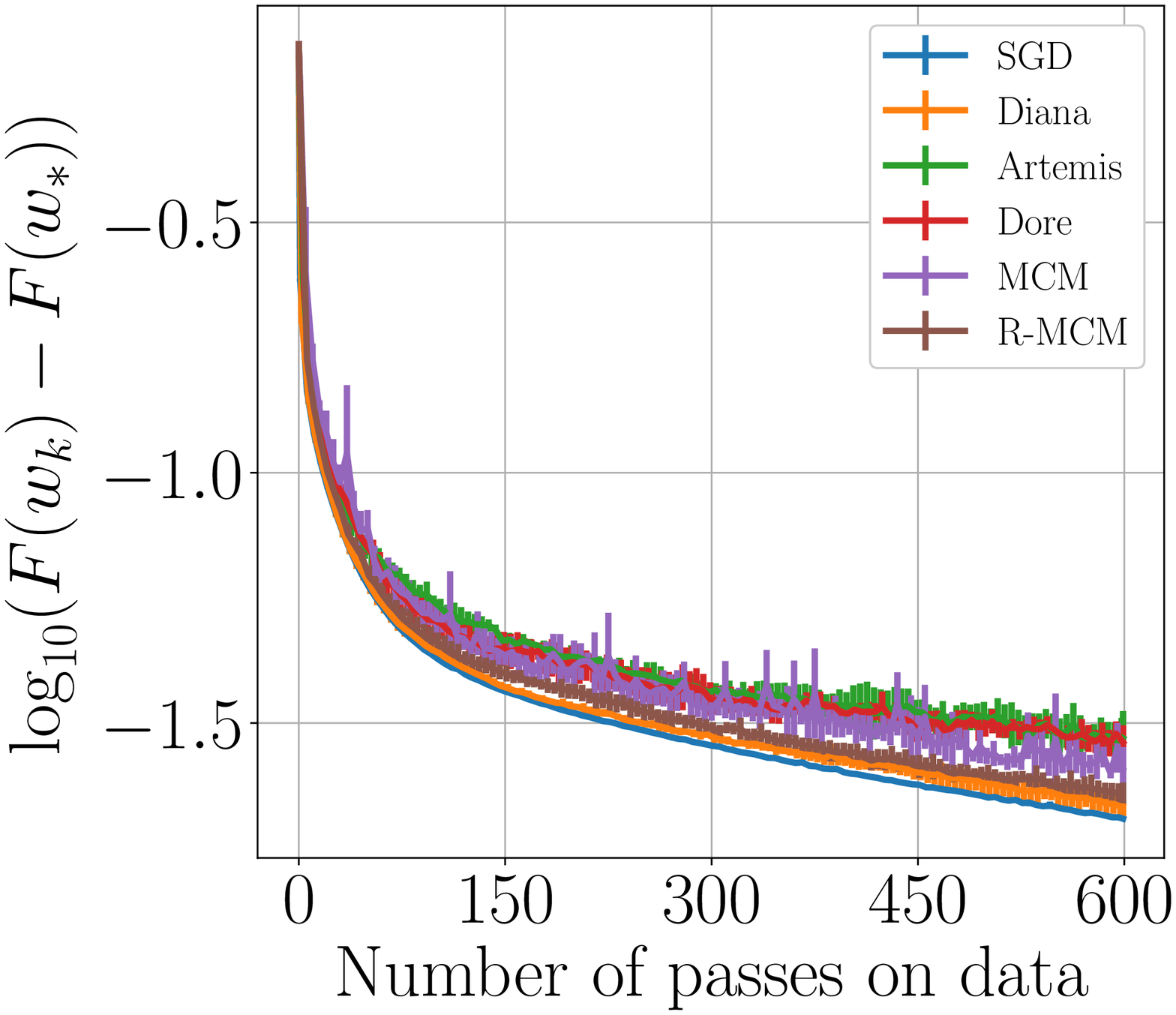}
        \caption{Superconduct.\vspace{-0.5em}}
        \label{app:fig:superconduct_it}
    \end{subfigure}
    \begin{subfigure}{0.3\textwidth}
        \centering
        \includegraphics[width=1\textwidth]{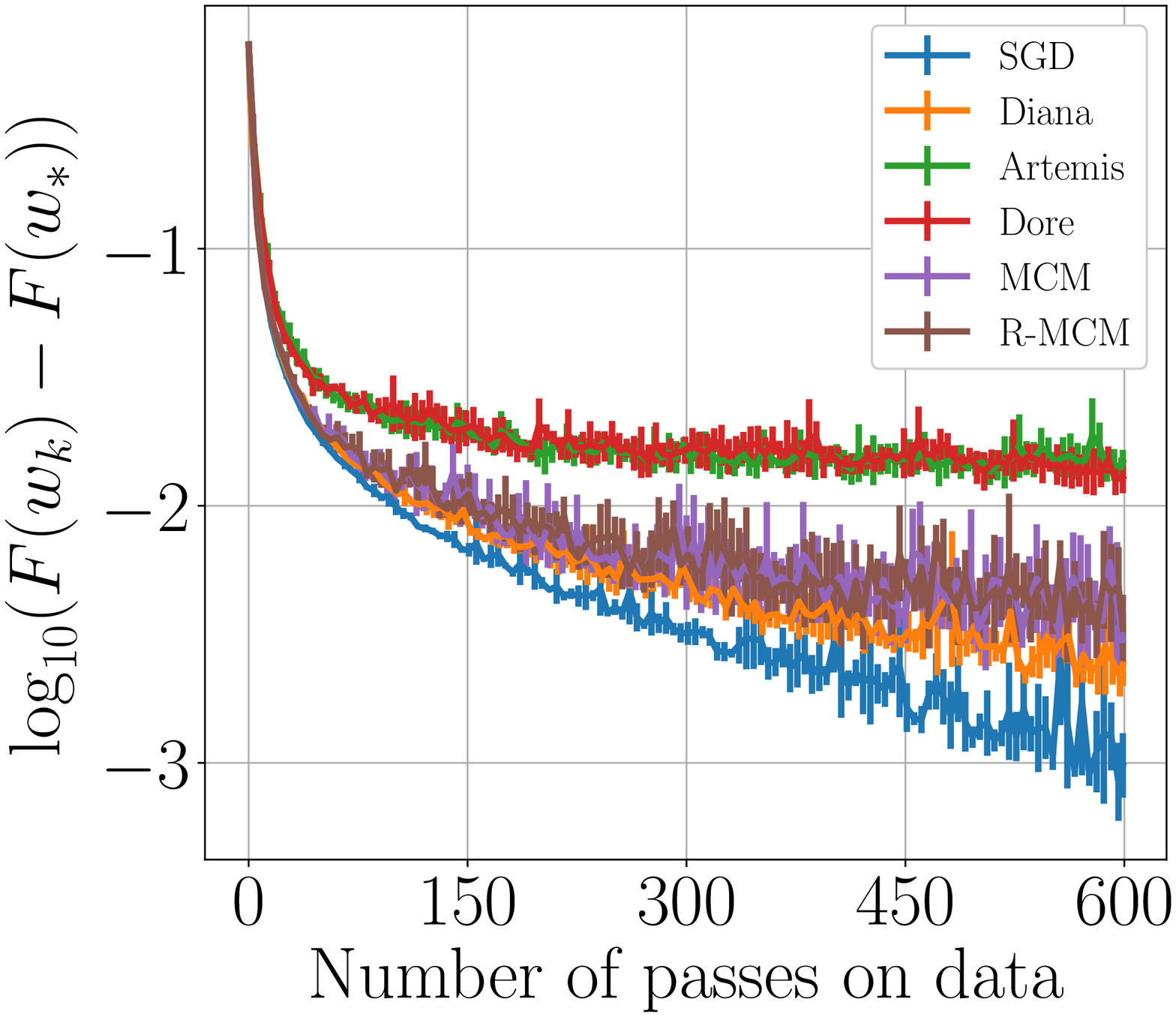}
        \caption{W8A.\vspace{-0.5em}}
        \label{app:fig:w8a_it}
    \end{subfigure}
    \caption{X axis in \# iterations.  \vspace{-0.9em}}
    \label{app:fig:phishing_superconduct_w8a}
\end{figure}

On \Cref{app:fig:random_sparsification}, we present a9a, quantum and w8a with a different operator of compression than in all other experiments. We use random unbiased sparsification: each coordinate has a likelihood $p=0.1$ to be selected.

\begin{figure}
    \centering
    \begin{subfigure}{0.3\textwidth}
        \centering
        \includegraphics[width=1\textwidth]{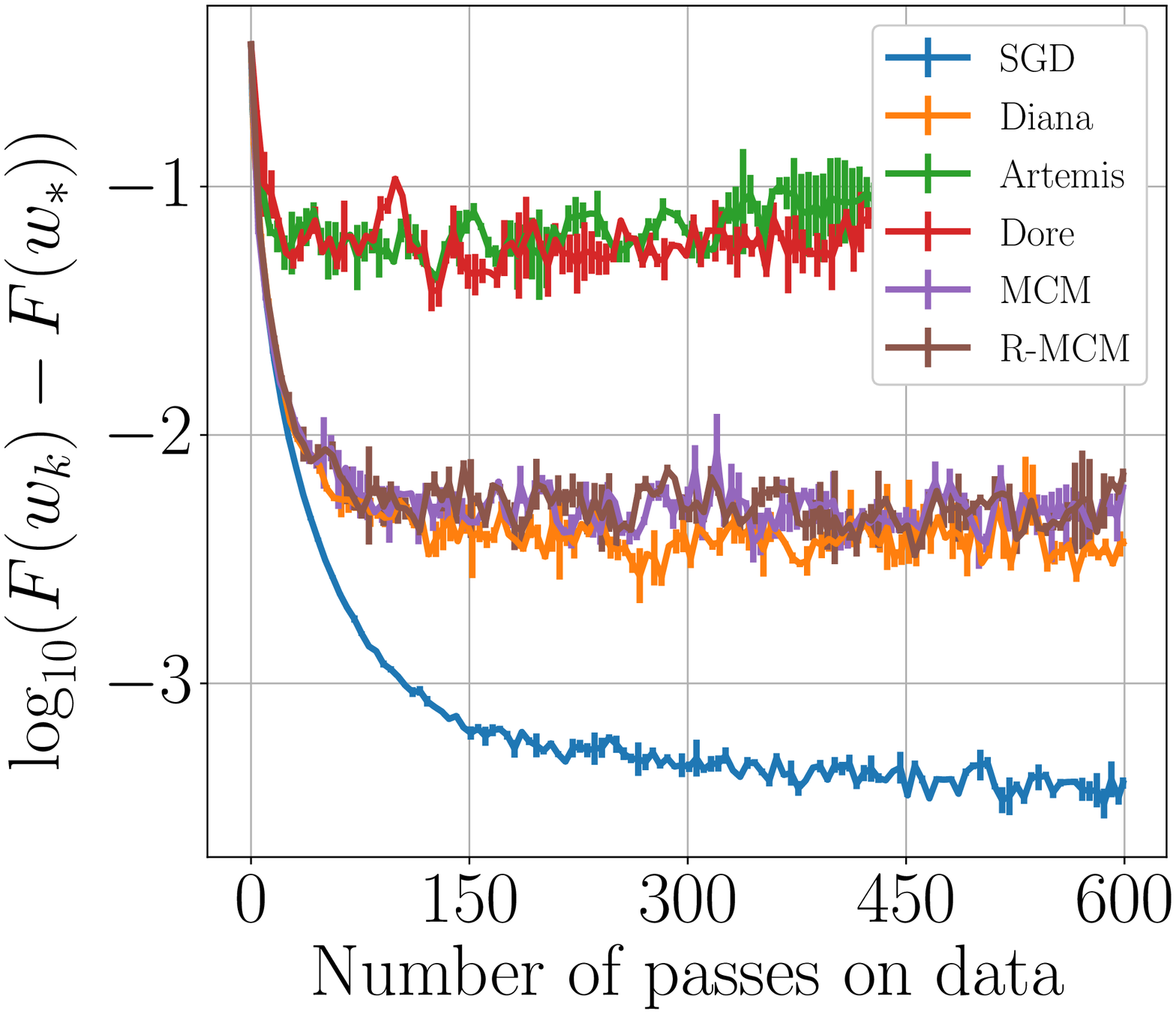}
        \caption{A9A.\vspace{-0.5em}}
        \label{app:fig:a9a_rdk}
    \end{subfigure}
    \begin{subfigure}{0.3\textwidth}
        \centering
        \includegraphics[width=1\textwidth]{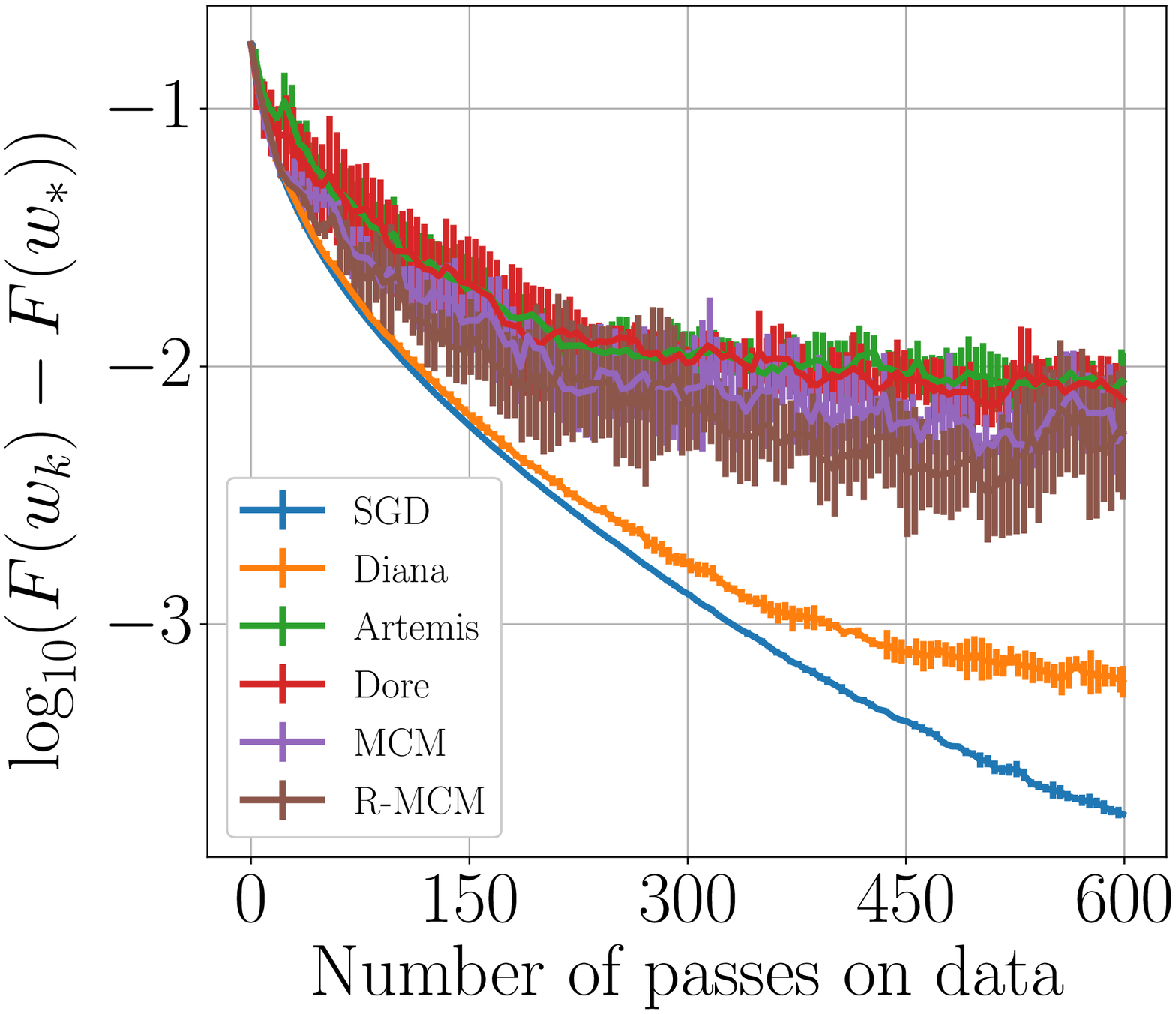}
        \caption{Quantum.\vspace{-0.5em}}
        \label{app:fig:quantum_rdk}
    \end{subfigure}
    \begin{subfigure}{0.3\textwidth}
        \centering
        \includegraphics[width=1\textwidth]{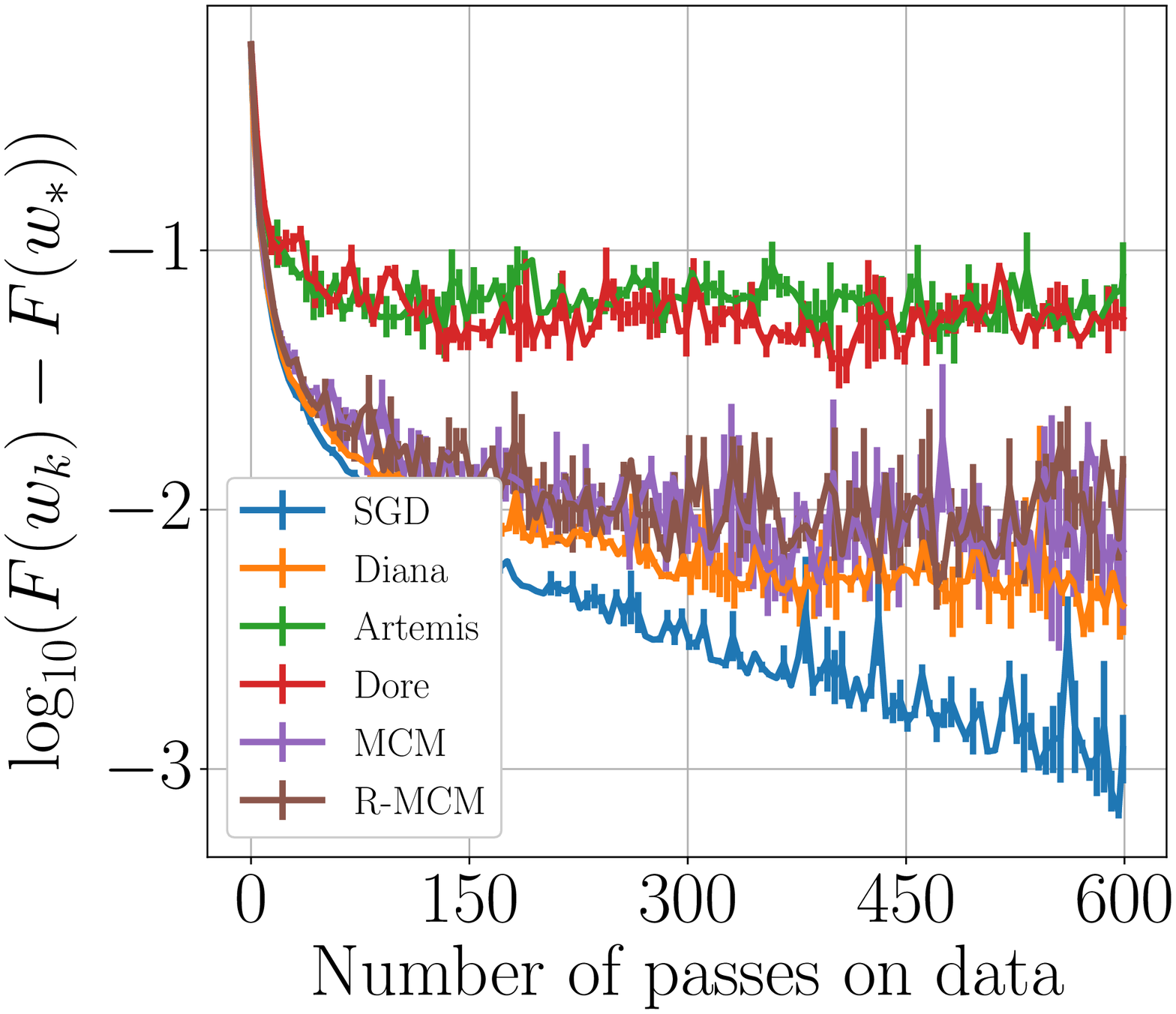}
        \caption{W8A.\vspace{-0.5em}}
        \label{app:fig:w8a_rdk}
    \end{subfigure}
    \caption{X axis in \# iterations using random sparsification with $p=0.1$.  \vspace{-0.9em}}
    \label{app:fig:random_sparsification}
\end{figure}

\subsubsection{Experiments on partial participation}
\label{app:subsec:expe_partial_part}

In this subsection, we run the experiments in a setting where only \textit{half of devices} (independently picked at each iteration) are available at each iteration, thus simulating a setting of partial participation. \Cref{app:fig:a9a_pp,app:fig:quantum_pp} present the results for respectively quantum and A9A. For these experiments, we used a $2$-quantization compression. We do not plot \MCM~on these figures because in a context of partial participation, \RMCM~is the natural thing to do. Indeed in this context, we must hold a memory for each worker, and thus the compressed vector sent to each worker is unique. 

We observe that partial participation leads to an increase of the variance for all algorithms.
Furthermore, we can observe on both \Cref{app:fig:a9a_bits_pp,app:fig:quantum_bits_pp} that \RMCM~outperforms \Artemis~and \Dore~not only in term of convergence but also in term of communication cost. This is because \RMCM~does not require the synchronization step, at which any active nodes receive any update it has missed. This saves a few communication rounds. In these settings, the level of saturation of \texttt{SGD}, \Diana~and \RMCM~seems to be almost identical, this fact stresses again the benefit of our designed algorithm.

\begin{figure}
    \centering
    \begin{subfigure}{\sizefig\textwidth}
        \centering
        \includegraphics[width=1\textwidth]{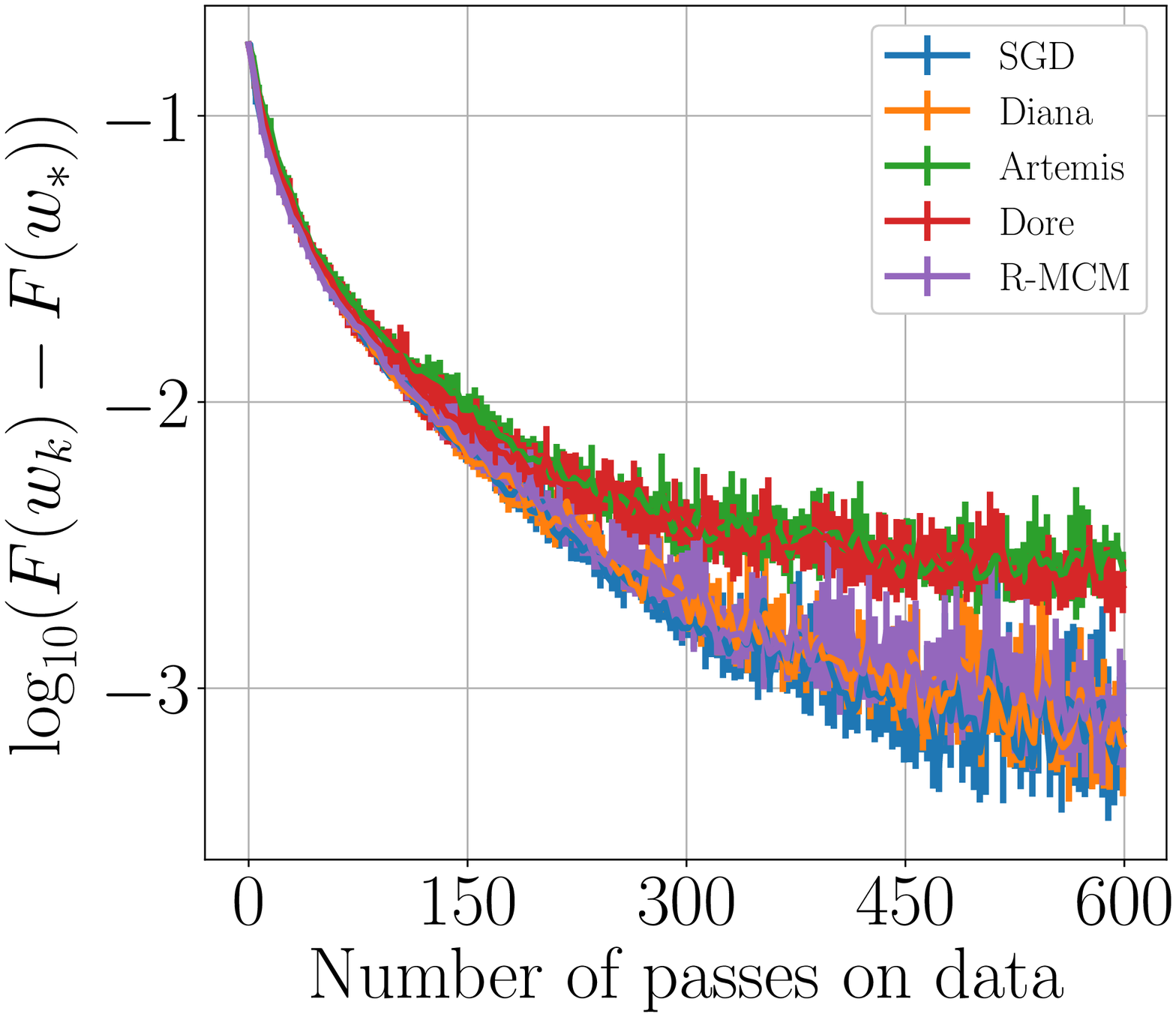}
        \caption{X axis in \# iterations.\vspace{-0.5em}}
        \label{app:fig:quantum_it_pp}
    \end{subfigure}
    \begin{subfigure}{\sizefig\textwidth}
        \centering
        \includegraphics[width=1\textwidth]{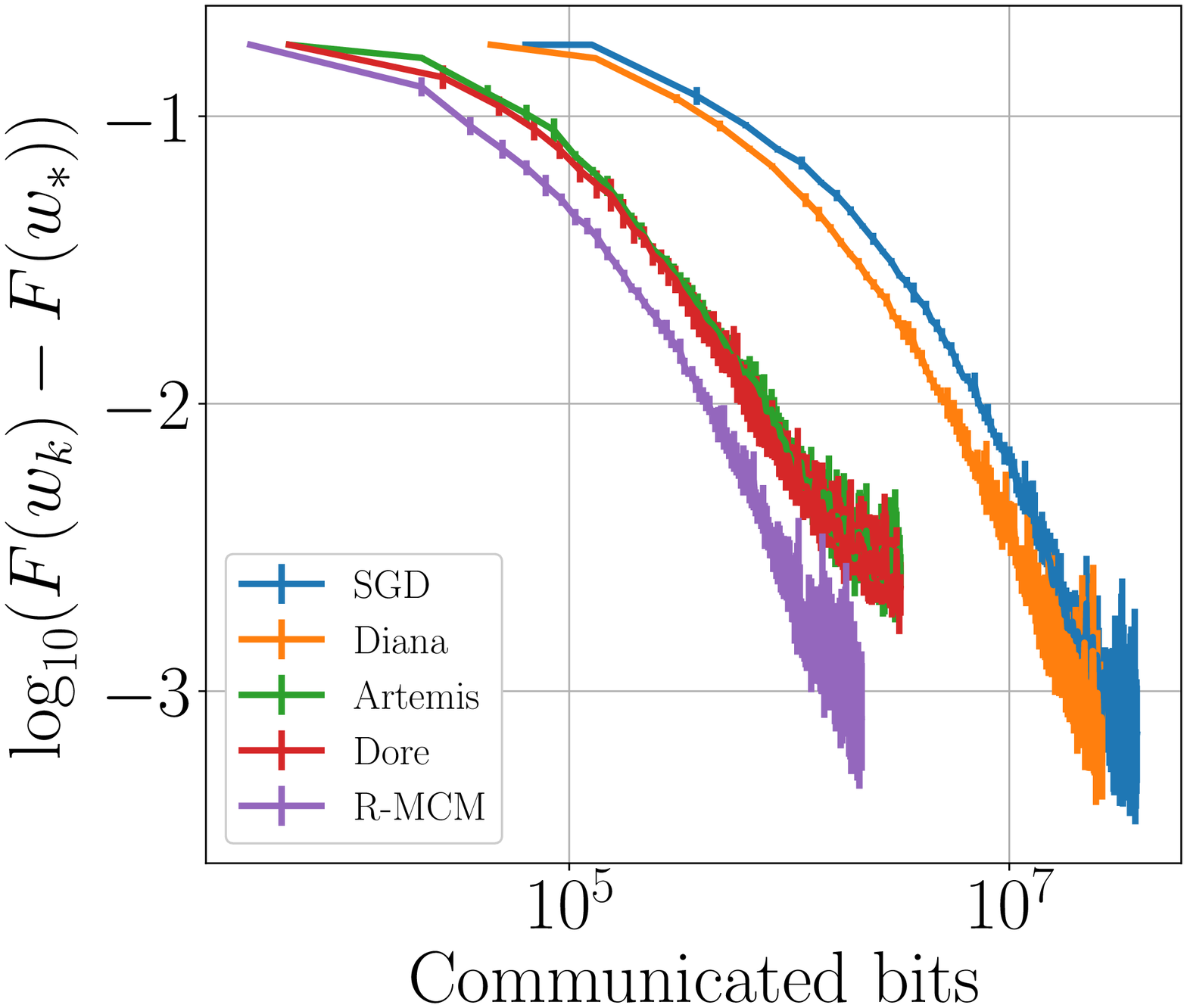}
        \caption{X axis in \# bits.\vspace{-0.5em}}
        \label{app:fig:quantum_bits_pp}
    \end{subfigure}
    \caption{quantum  with $b=400$, $\gamma = 1/L$ and a $2$-quantization. Only half of the devices are participating at each round. \vspace{-0.9em}}
    \label{app:fig:quantum_pp}
\end{figure}

\begin{figure}
    \centering
    \begin{subfigure}{\sizefig\textwidth}
        \centering
        \includegraphics[width=1\textwidth]{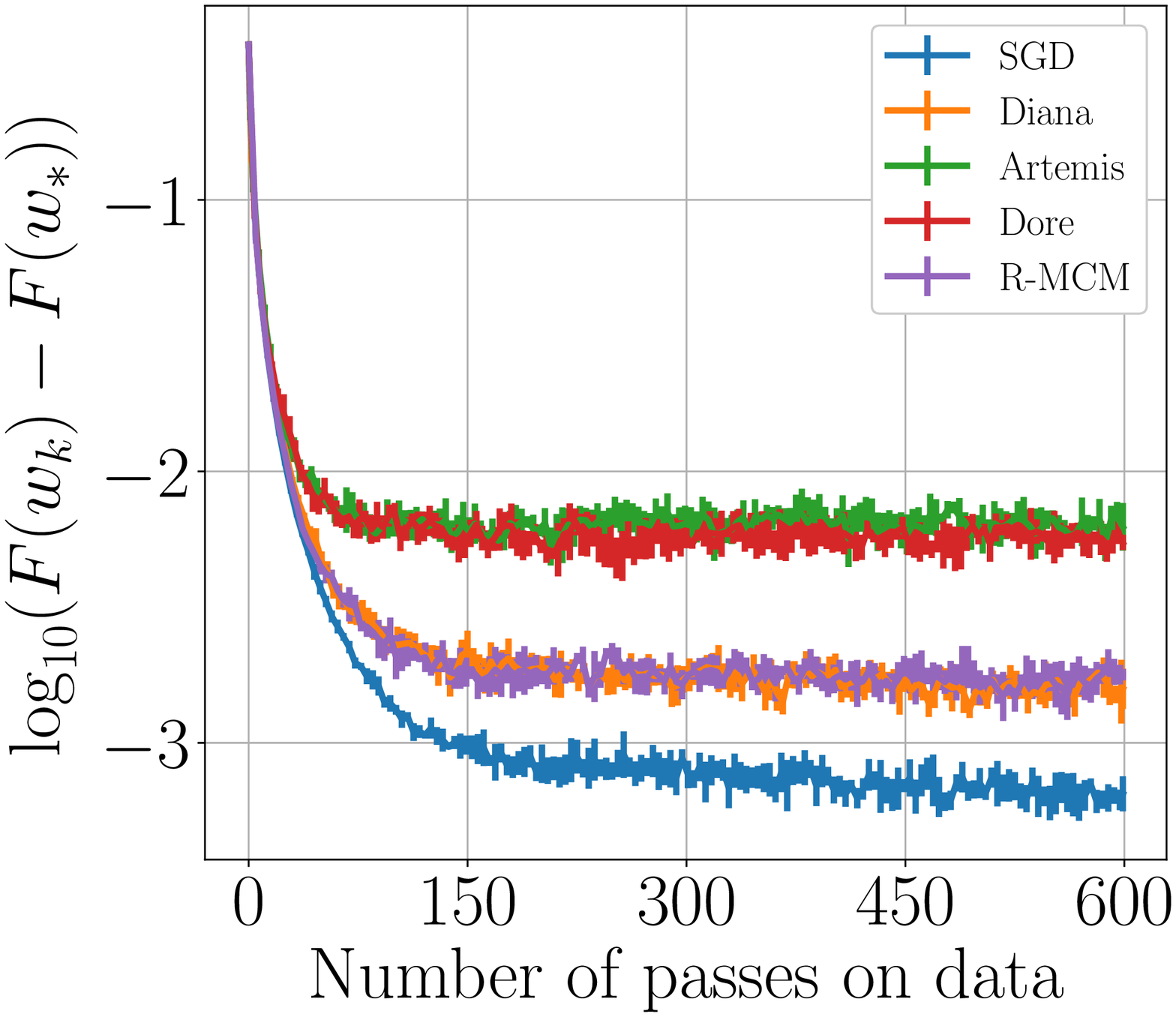}
        \caption{X axis in \# iterations.\vspace{-0.5em}}
        \label{app:fig:a9a_it_pp}
    \end{subfigure}
    \begin{subfigure}{\sizefig\textwidth}
        \centering
        \includegraphics[width=1\textwidth]{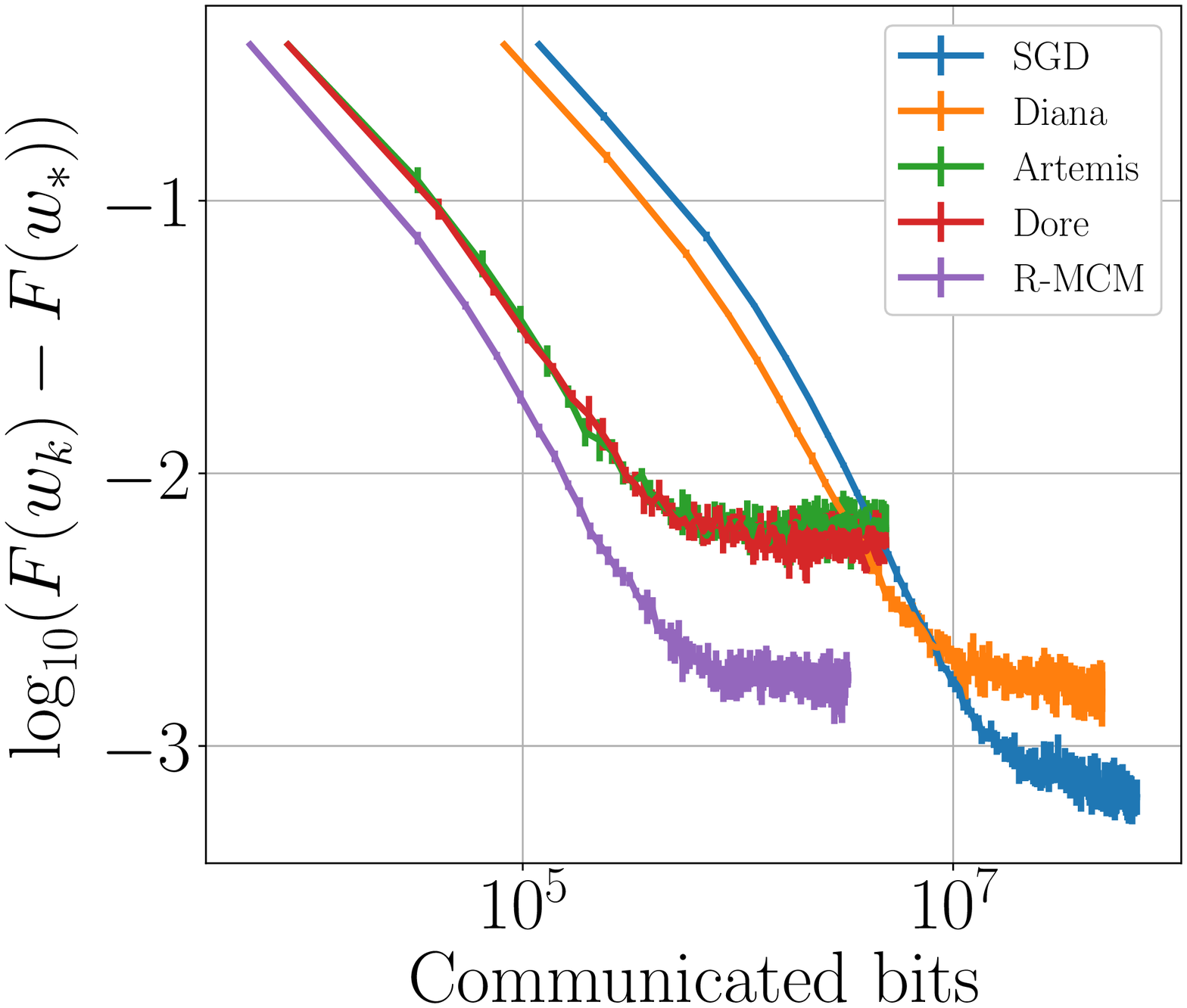}
        \caption{X axis in \# bits.\vspace{-0.5em}}
        \label{app:fig:a9a_bits_pp}
    \end{subfigure}
    \caption{A9A  with $b=50$, $\gamma = 1/L$ and a $2$-quantization. Only half of the devices are participating at each round. \vspace{-0.9em}}
    \label{app:fig:a9a_pp}
\end{figure}

Additionally, we present on \Cref{app:fig:a9a_pp_one_mem,app:fig:quantum_pp_one_mem} the impact of only using a single averaged downlink memory term instead of $N$ distinct memories. More details about update equations are given in \Cref{app:eq:rand_one_vs_N_mem}. We display three versions of \RMCM~that we compare to the SGD-baseline and to \Artemis:
\begin{enumerate}
    \item The standard \RMCM, using $N$ downlink memories,
    \item \RMCM~with a single memory, without any periodically reset.
    \item \RMCM~with both a single memory and a reset of the downlink memory every $4 \sqrt d$ iterations, where $d$ is the dimension of the optimization problem. This allows to limit the increase of communicated bits. Indeed as we use quantization with $s=1$, each communication costs $32\times \sqrt d \log(d)$ bits instead of $32\times d$. Because every $4\sqrt d$ iterations we send the uncompressed downlink memory term, there is an additional cost of $\ffrac{32 d}{4 \sqrt d}$. At the end, the memory reset leads to send $32\times \sqrt d (\log(d) + 1/4)$ bits by iterations instead of $32\times \sqrt d \log(d)$ bits for \RMCM (without reset). The increase is thus marginal.
\end{enumerate}

 For sake of clarity, we present below the two versions of \RMCM. In the first version, the central server holds $N$ memories that exactly correspond to those kept on the $N$ remote devices. In the second version, the central server holds a single memory $\bar H_k = \frac{1}{N}\sum_{i=1}^N H_k^i$ and each worker $i$ holds there own memory $H_k^i$.
\begin{align}
\begin{array}{ll}
\textbf{\text{$N$ memories}}     &  \textbf{\text{$1$ memories}}    \\
    \begin{aligned}\label{app:eq:rand_one_vs_N_mem}
\left\{
    \begin{array}{ll}
        \Omega_{k+1}^i = w_{k+1} - H_{k}^i \,, \\
    \widehat{w}_{k+1}^i = H_{k}^i + \C_{\mathrm{dwn}, i}(\Omega_{k+1}^i) \\
    H_{k+1}^i = H_k^i + \alpha\dwn \C_{\mathrm{dwn}, i}({\Omega}_{k+1}^i) .
    \end{array}
\right. 
\end{aligned}  & \begin{aligned}
\left\{
    \begin{array}{ll}
        \Omega_{k+1} = w_{k+1} - \bar H_{k} \,, \\
    \widehat{w}_{k+1}^i = H_{k}^i + \C_{\mathrm{dwn}, i}(\Omega_{k+1}) \\
    H_{k+1}^i = H_k^i + \alpha\dwn \C_{\mathrm{dwn}, i}({\Omega}_{k+1}) \\
    \bar H_{k+1} = \bar H_k + \ffrac{\alpha\dwn}{N} \sum_{i=1}^N \C_{\mathrm{dwn}, i}({\Omega}_{k+1}) .
    \end{array}
\right. 
\end{aligned}
\end{array}
\end{align}

In this experiments, it is noticeable that using single-downlink-memory-\RMCM~without periodic reset makes the algorithms saturate at a high level with an important variance. But as soon as we introduce the reset, we recover previous rates.

\begin{figure}
    \centering
    \begin{subfigure}{\sizefig\textwidth}
        \centering
        \includegraphics[width=1\textwidth]{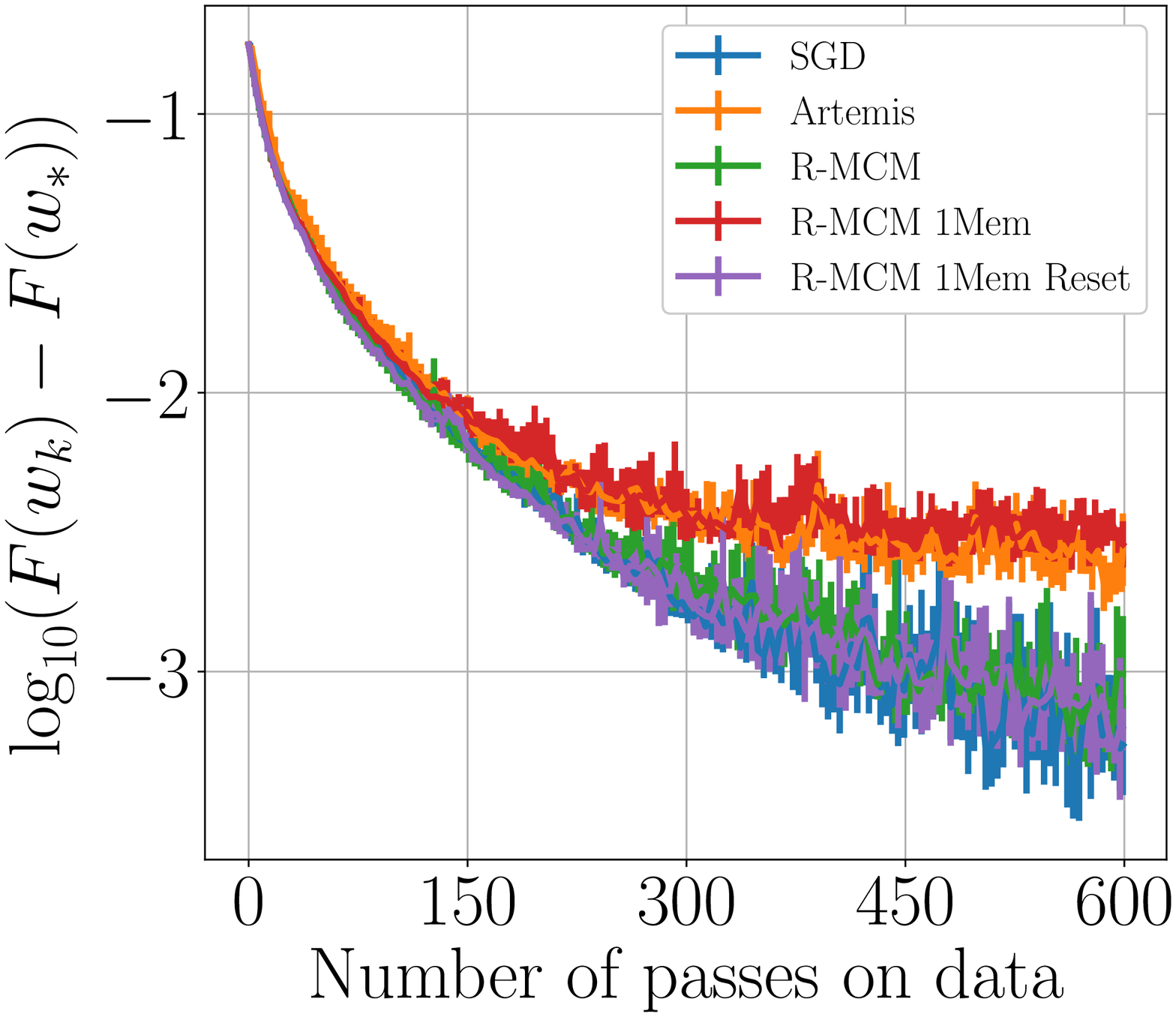}
        \caption{X axis in \# iterations.\vspace{-0.5em}}
        \label{app:fig:quantum_it_pp_one_mem}
    \end{subfigure}
    \begin{subfigure}{\sizefig\textwidth}
        \centering
        \includegraphics[width=1\textwidth]{pictures/exp/quantum/mcm-1-mem/bits-noavg-sto-b400.eps}
        \caption{X axis in \# bits.\vspace{-0.5em}}
        \label{app:fig:quantum_bits_pp_one_mem}
    \end{subfigure}
    \caption{quantum  with $b=400$, $\gamma = 1/L$ and a $2$-quantization. Only half of the devices are participating at each round. \vspace{-0.9em}}
    \label{app:fig:quantum_pp_one_mem}
\end{figure}

\begin{figure}
    \centering
    \begin{subfigure}{\sizefig\textwidth}
        \centering
        \includegraphics[width=1\textwidth]{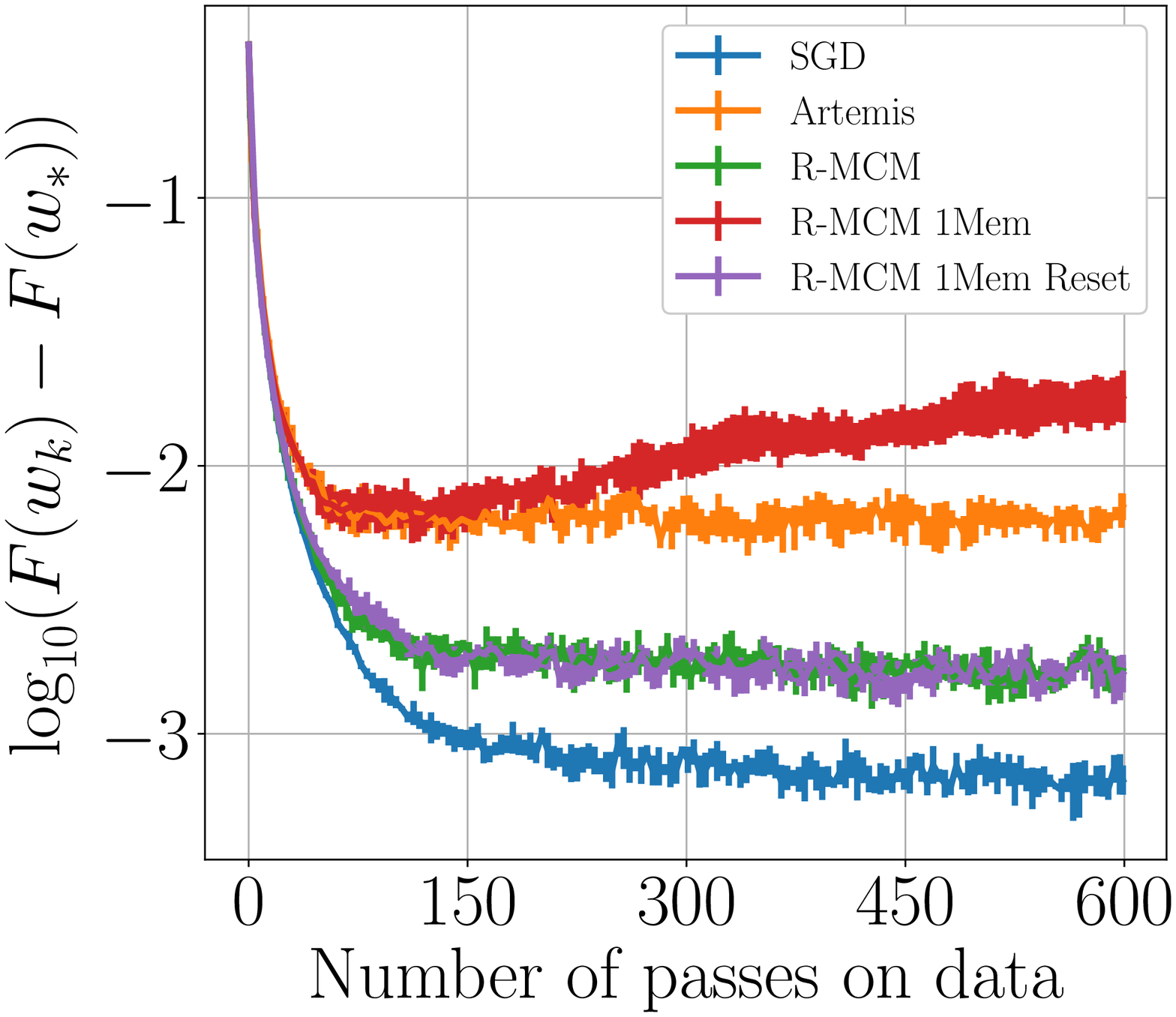}
        \caption{X axis in \# iterations.\vspace{-0.5em}}
        \label{app:fig:a9a_it_pp_one_mem}
    \end{subfigure}
    \begin{subfigure}{\sizefig\textwidth}
        \centering
        \includegraphics[width=1\textwidth]{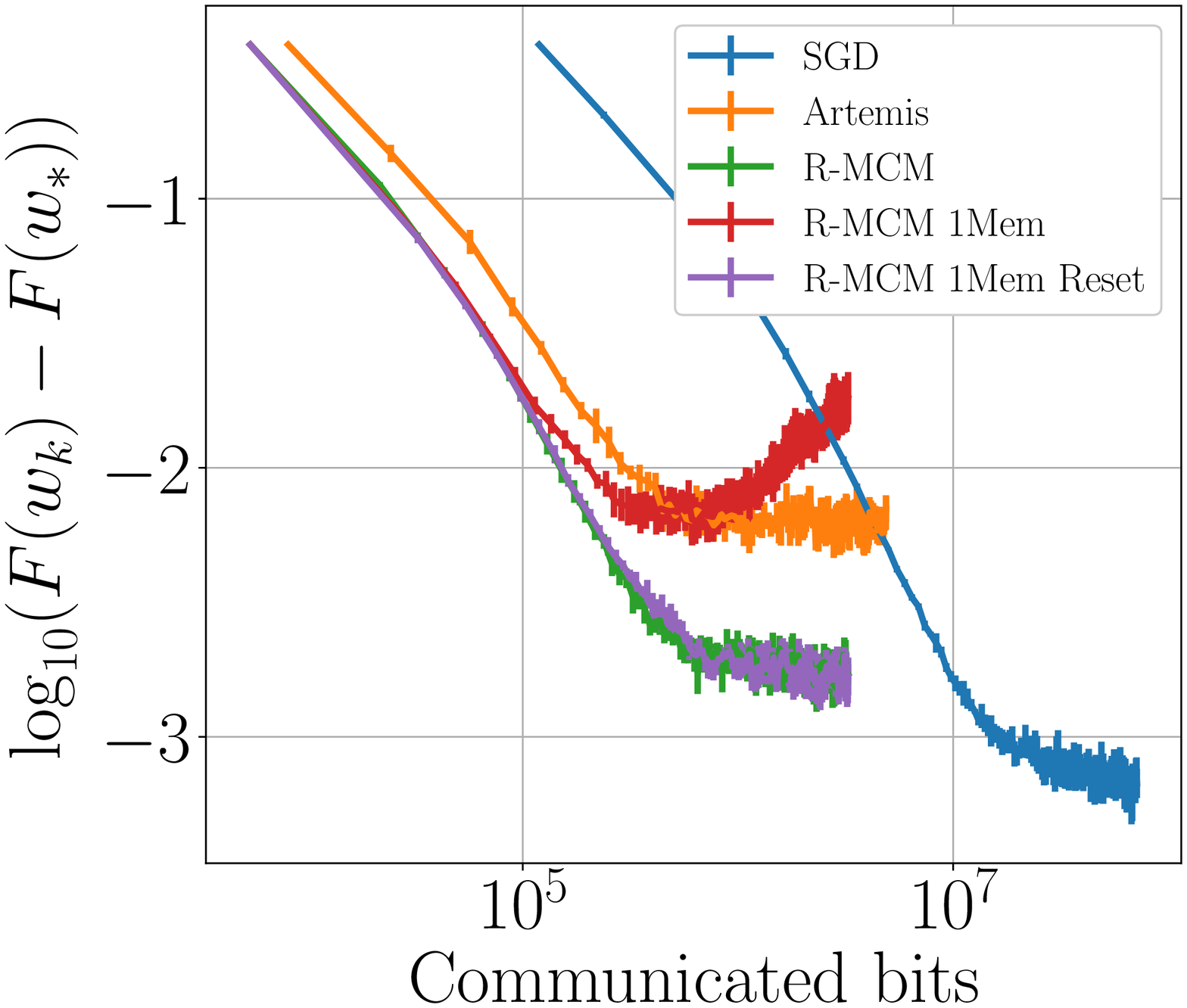}
        \caption{X axis in \# bits.\vspace{-0.5em}}
        \label{app:fig:a9a_bits_pp_one_mem}
    \end{subfigure}
    \caption{A9A  with $b=50$, $\gamma = 1/L$ and a $2$-quantization. Only half of the devices are participating at each round. \vspace{-0.9em}}
    \label{app:fig:a9a_pp_one_mem}
\end{figure}

\subsubsection{Comparing \MCM~with other algorithm using non-degraded update}
\label{app:subsec:exp_non_degraded_update}

The aim of this section is to show the importance to set $\alpha < 1$, for this purpose we compare \MCM~with three other algorithms:
\begin{enumerate}
    \item \Artemis~with a non-degraded update i.e. unlike the version proposed by \citet{philippenko_artemis_2020}, we do not update the global model with the compression sent to all remote nodes. \textit{It means that we compress only the update that has already been performed on the global server.} It corresponds to:
    \begin{align*}
    \left\{
    \begin{array}{l}
    \forall i \in \llbracket1, N \rrbracket, \Delta_{k}^i = \g_{k+1}^i(\widehat{w}_{k}) - h_k^i  \\
    w_{k+1} = w_k - \ffrac{\gamma}{N} \sum_\iN \C\up(\Delta_{k}^i) + h_k^i\\
    \widehat{w}_{k+1} = \widehat{w}_k - \gamma \C\dwn\bigpar{\ffrac{1}{N} \sum_\iN \C\up(\Delta_{k}^i) + h_k^i}\\
    h_{k+1}^{i} = h_k^i + \alpha\up \C\up({\Delta}_k^i).
    \end{array}
    \right. 
    \end{align*}
    \item \MCM~with $\alpha = 0$, \textit{thus without memory.}
    \item \MCM~with $\alpha=1$, in other words, for $k$ in $\N^*$ it corresponds to the case $H_{k+1} = \widehat{w}_{k+1}$. Indeed by definition we have $H_{k+1} = H_k + \alpha \widehat{\Omega}_{k+1}$, and furthermore, when we rebuild the compressed model on remote device, we have: $\widehat{w}_{k+1} = \widehat{\Omega}_{k+1} + H_k$. \textit{In this case, we use the compressed model as memory.}
\end{enumerate}

\Cref{app:fig:quantum_mcm_various_option,app:fig:a9a_mcm_various_option} clearly show the superiority of \MCM~over the three other variants. Some conclusions can be drawn from the observation of these figures.

\begin{itemize}
    \item \MCM~without downlink memory (orange curve, $\alpha = 0$) does not converge. As stressed in \Cref{sec:bidirectional_framework}, this mechanism is crucial to control the variance of the local model $w_{k+1}$, for $k$ in $\N$.
    \item Intuitively, while it appears reasonable to consider as memory the model that has been compressed at the previous step, experiments (green curves) show that this is not the case in practice and that $\alpha$ must be small enough to ensure convergence. This is the \textit{noise explosion} phenomenon that was mentioned earlier in the paper. 
    \item Compressing only the update gives reasonable results (blue curve). However, the convergence saturates at a higher level than for \MCM. 
\end{itemize}

\begin{figure}
    \centering
    \begin{subfigure}{\sizefig\textwidth}
        \centering
        \includegraphics[width=1\textwidth]{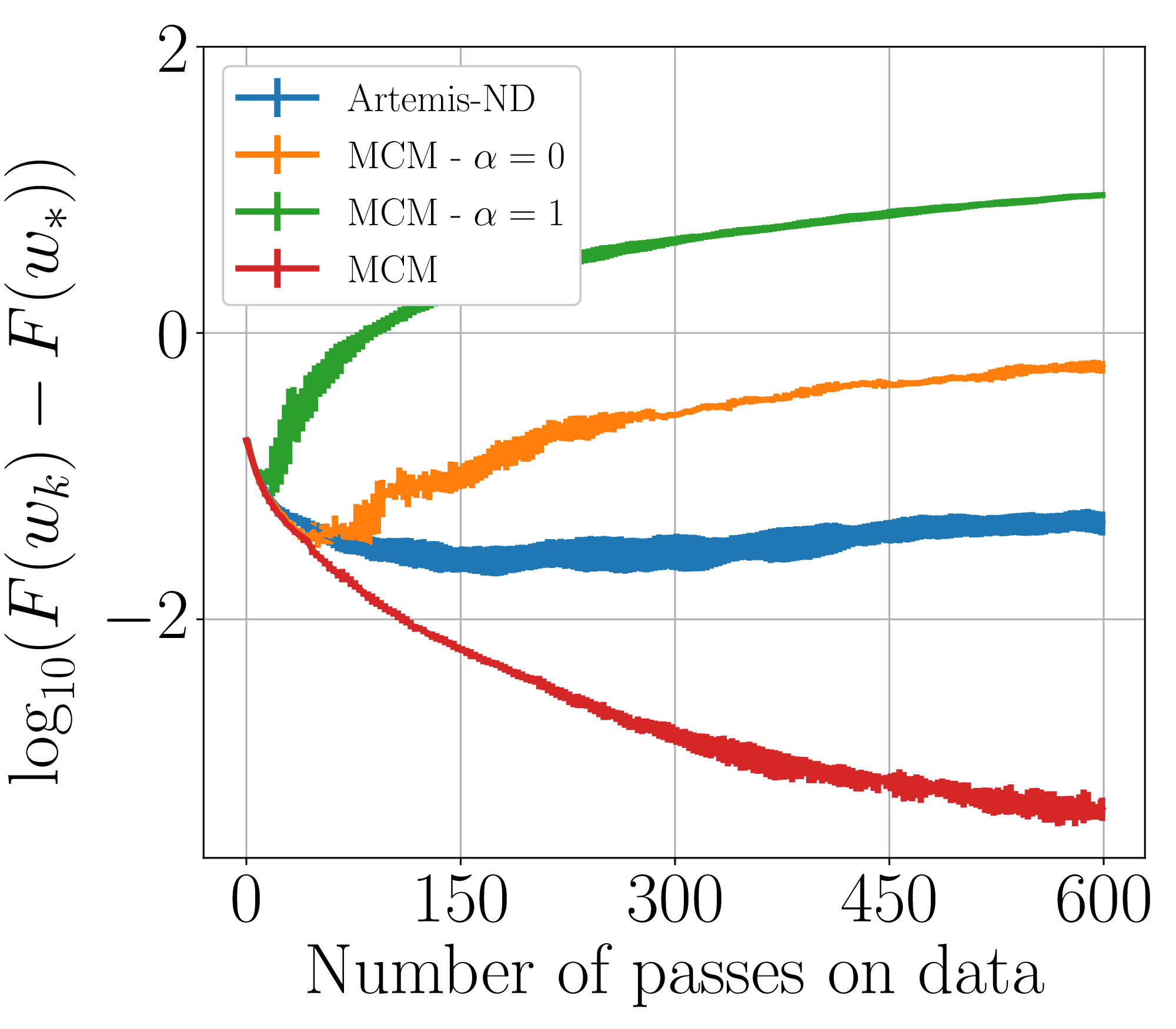}
        \caption{quantum  - $b=400$.\vspace{-0.5em}}
        \label{app:fig:quantum_mcm_various_option}
    \end{subfigure}
    \begin{subfigure}{\sizefig\textwidth}
        \centering
        \includegraphics[width=1\textwidth]{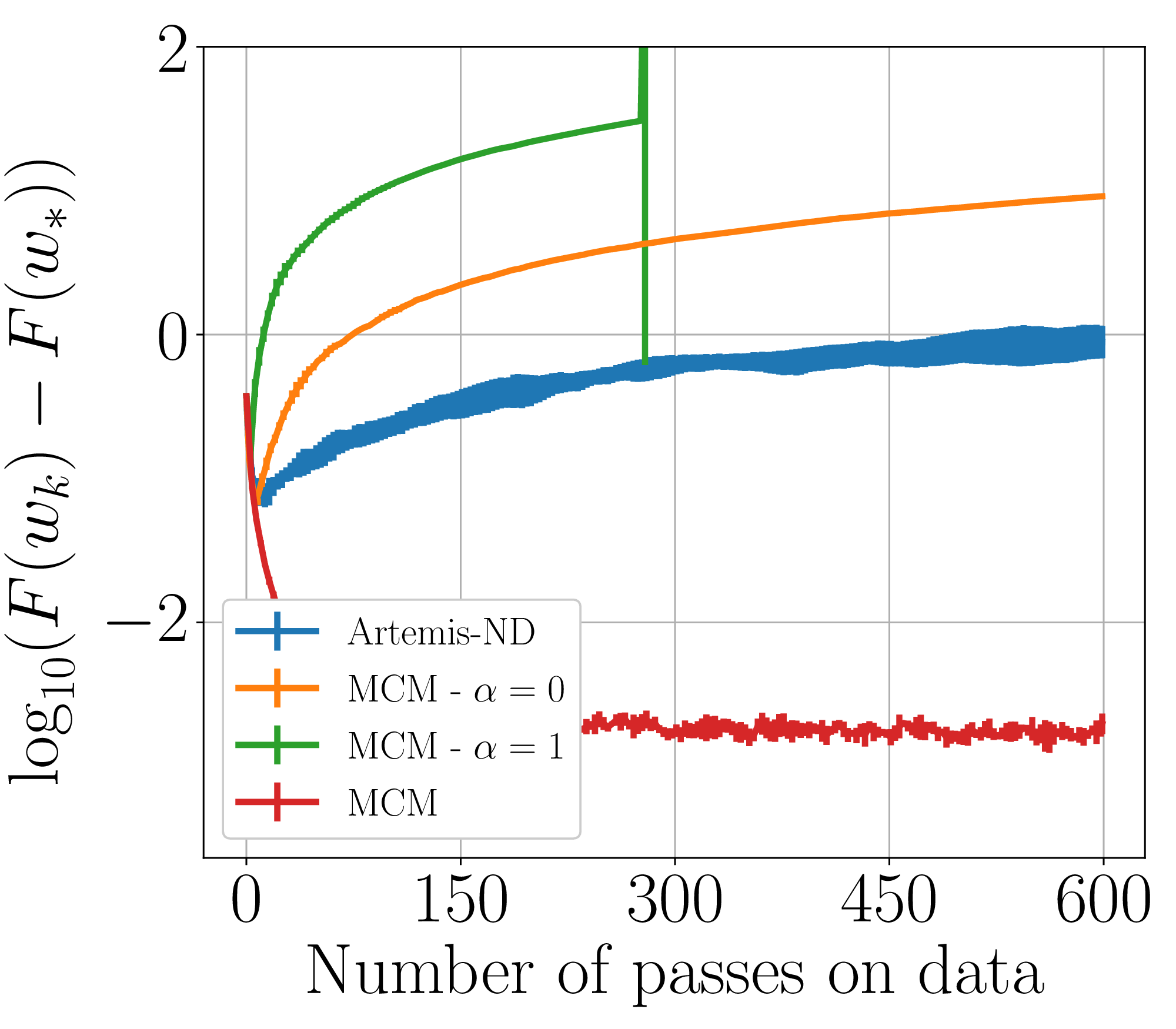}
        \caption{\textit{A99}  - $b=50$.\vspace{-0.5em}}
        \label{app:fig:a9a_mcm_various_option}
    \end{subfigure}
    \caption{Comparing \MCM~with three other algorithms using a non-degraded update, $\gamma = 1/L$. \texttt{Artemis-ND} stands for \Artemis~with a non-degraded update. \vspace{-0.9em}}
    \label{app:fig:mcm_various_option}
\end{figure}

\subsubsection{Impact of the learning rate $\alpha$}

On \Cref{app:fig:alpha}, we plot the value of the excess loss obtained after $250$ epochs w.r.t. to the value of $\frac{1}{2(1+\omgC_{\mathrm{up}/\mathrm{dwn}})}$.
We observe that if $\alpha$ is too big, \MCM~converges slowly; but after reaching a threshold, the value of $\alpha$ does not impact anymore the rate of convergence. This confirms theory that suggests to use the largest possible $\alpha_{dwn}$ but smaller than a given value.  The condition $\alpha\dwn \leq \frac{1}{ 4 (\omega\dwn +1)}$ results from the proofs of \Cref{app:thm:contraction_mcm,app:thm:contraction_mcm_heterog}. But because the constant $4$ is partially an artifact of the proof, in experiments we used $\alpha\dwn = \frac{1}{2(\omega\dwn +1)}$ as in \cite{philippenko_artemis_2020} (see condition S19 in Theorem S7), and this choice is confirmed by \Cref{app:fig:alpha}.

\begin{figure}
    \centering
    \begin{subfigure}{0.3\textwidth}
        \centering
        \includegraphics[width=1\textwidth]{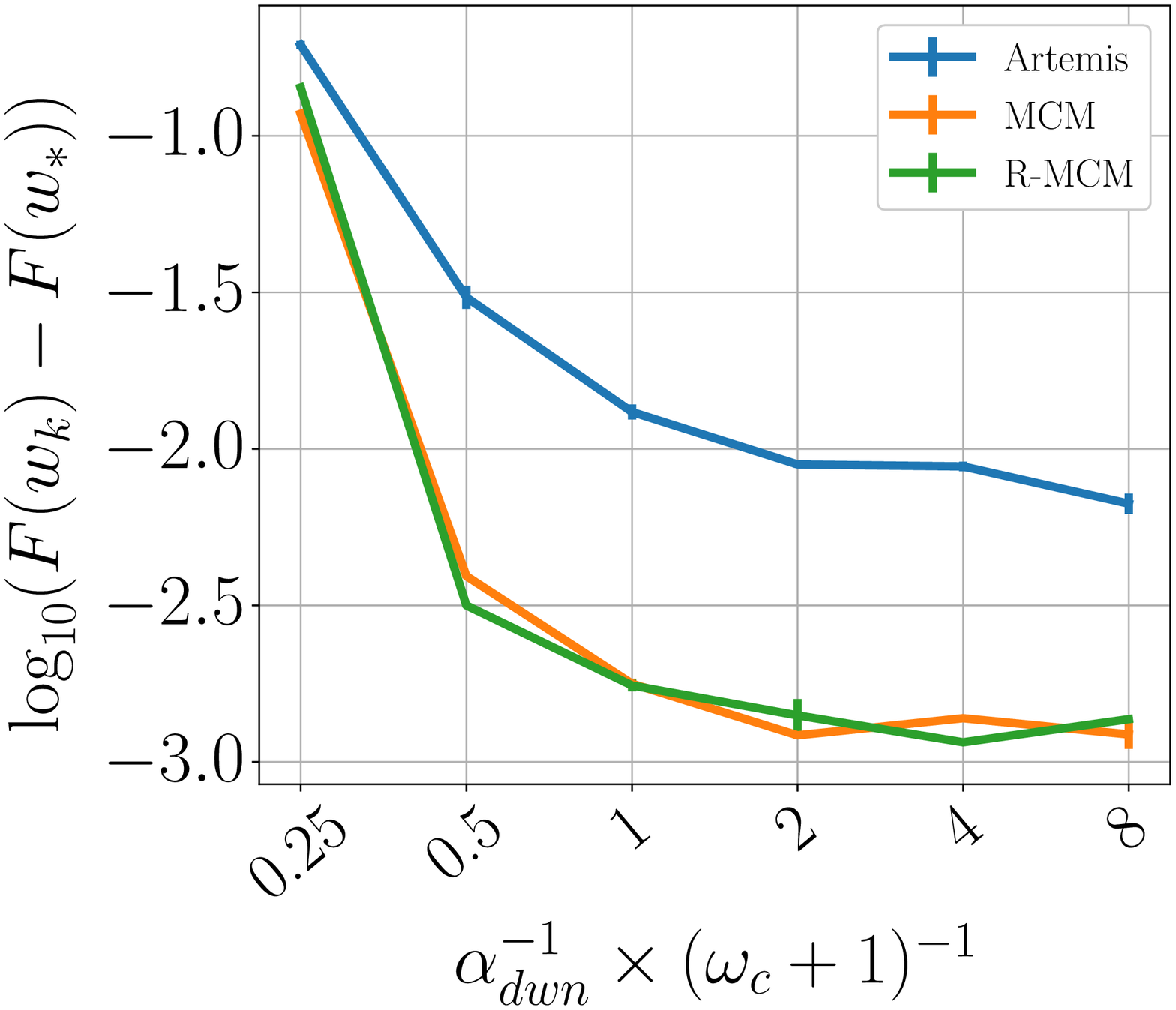}
        \caption{A9A.\vspace{-0.5em}}
        \label{app:fig:a9a_alpha}
    \end{subfigure}
    \begin{subfigure}{0.3\textwidth}
        \centering
        \includegraphics[width=1\textwidth]{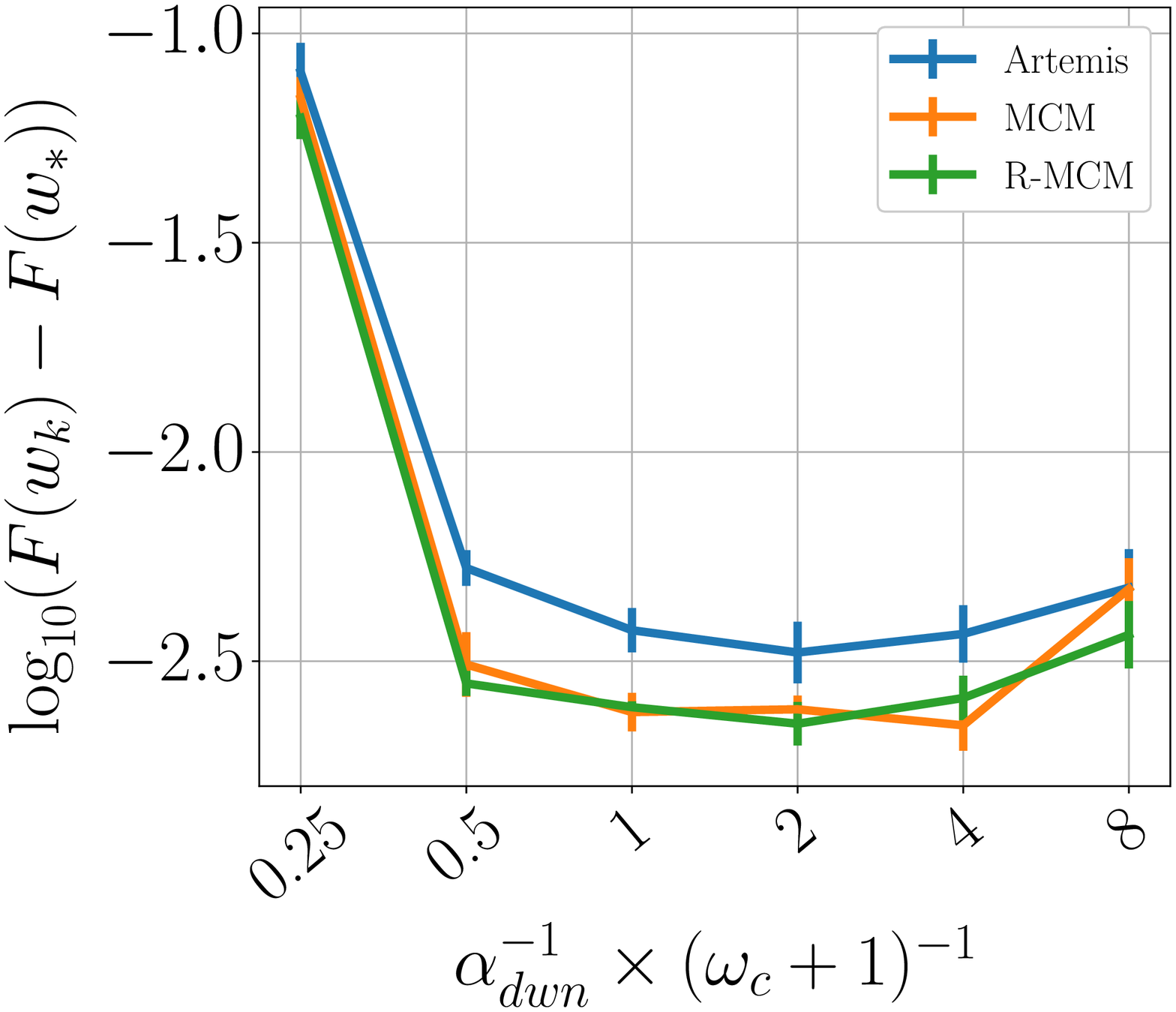}
        \caption{Quantum.\vspace{-0.5em}}
        \label{app:fig:quantum_alpha}
    \end{subfigure}
    \begin{subfigure}{0.3\textwidth}
        \centering
        \includegraphics[width=1\textwidth]{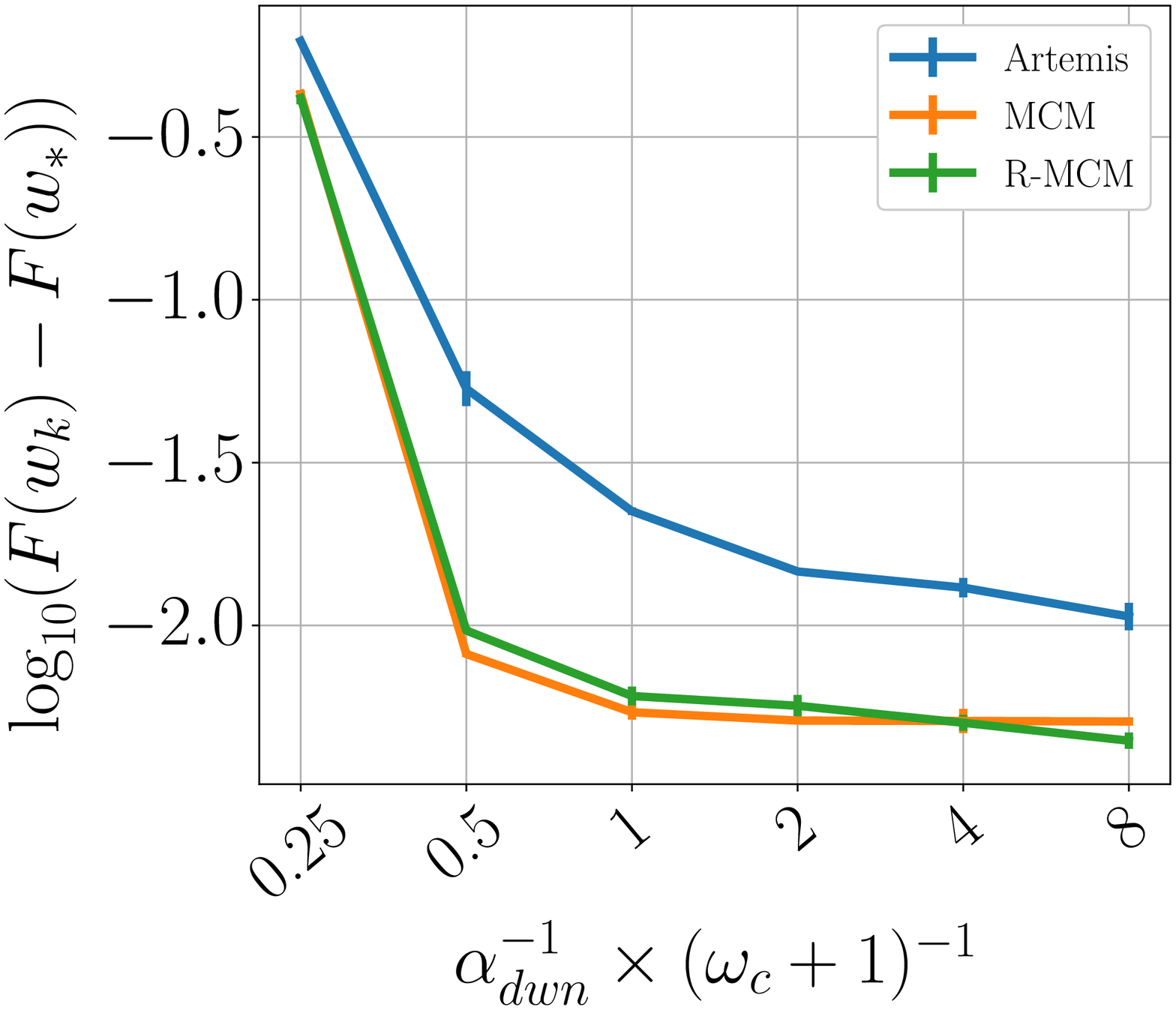}
        \caption{W8A.\vspace{-0.5em}}
        \label{app:fig:w8a_alpha}
    \end{subfigure}
    \caption{On X axis is displayed different values of $\ffrac{1}{\alpha (\omgC\dwn + 1)}$. On Y axis is given the excess loss after $250$ epochs. In all other experiments, we choose $\alpha\dwn = \ffrac{1}{2(\omgC\dwn + 1)} (=\alpha\up)$. \vspace{-0.9em}}
    \label{app:fig:alpha}
\end{figure}

\subsection{Experiments in deep learning}
\label{app:subsec:deep_learning}

In this section, we show the robustness of \MCM~in high dimension using more complex data and applying the algorithm to non-convex problems (see \Cref{app:thm:mcm_non_convex} for a guarantee of convergence in this scenario). We carried out experiments on MNIST/FE-MNIST/Fashion-MNIST  using a CNN (\Cref{app:fig:FE_MNIST,app:fig:Fashion_MNIST,app:fig:MNIST}), and on CIFAR using the LeNet model (\Cref{app:fig:CIFAR10}).
We plot the logarithm of the train loss w.r.t the number of iterations and the number of communicated bits. The accuracy has been given in \Cref{sec:experiments}, see \Cref{tab:exp_nonconvex}. Settings of the experiments can be found in \Cref{app:tab:settings_nonconvex}, all experiments are averaged over $2$ runs.

As for experiments in convex case, \MCM~presents identical rates of convergence than \Diana~but with a small shift that makes \Artemis~better during the first iterations. 

\begin{figure}
    \centering
    \begin{subfigure}{\sizefig\textwidth}
        \centering
        \includegraphics[width=1\textwidth]{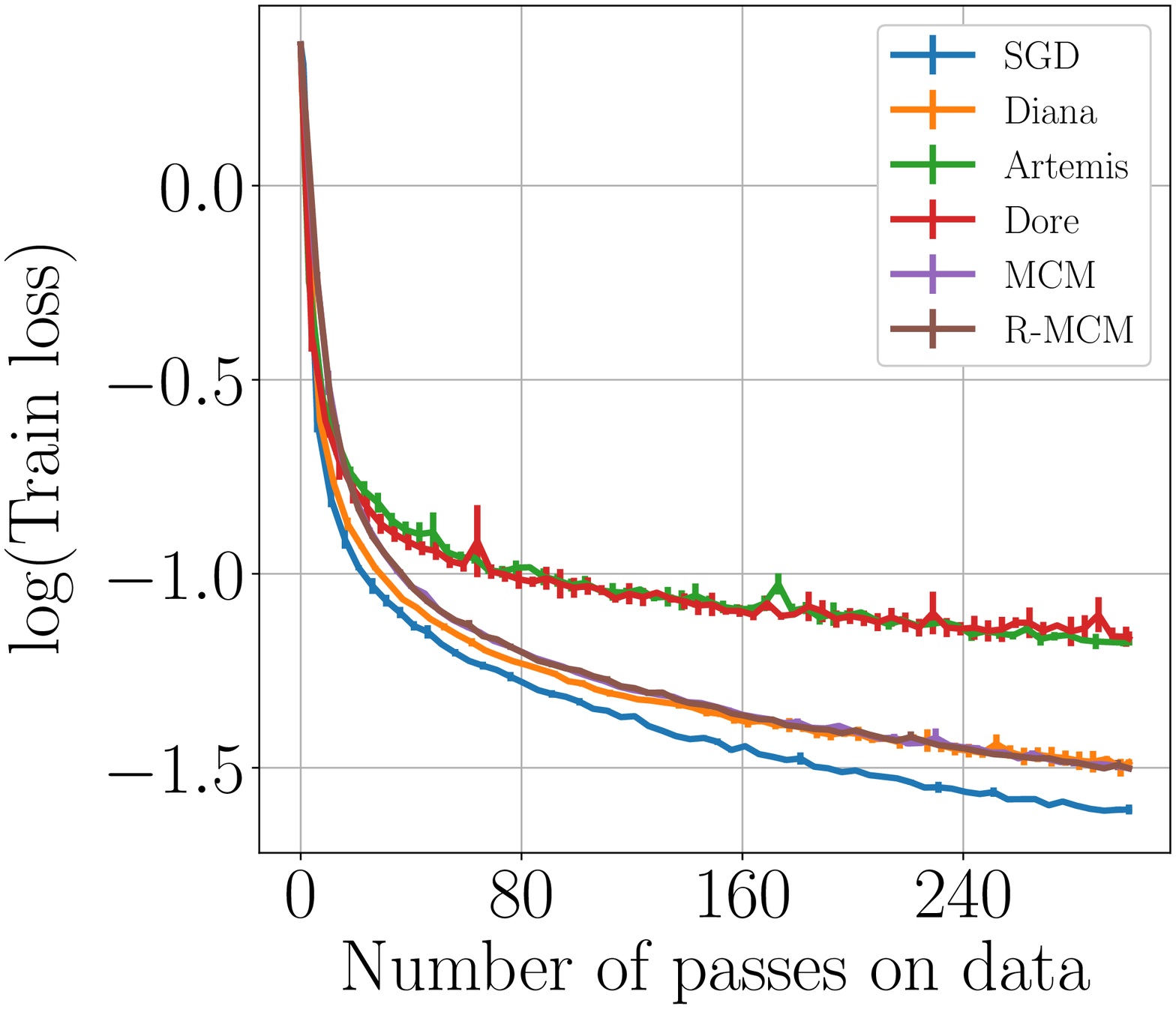}
        \caption{X axis in \# iterations.\vspace{-0.5em}}
        \label{app:fig:MNIST_its}
    \end{subfigure}
    \begin{subfigure}{\sizefig\textwidth}
        \centering
        \includegraphics[width=1\textwidth]{pictures/exp/mnist/MNIST_CNN_m0_lr0.1_sup4_sdwn4_b128_wd0_norm-2_train_losses_bits.eps}
        \caption{X axis in \# bits.\vspace{-0.5em}}
        \label{app:fig:MNIST_bits}
    \end{subfigure}
    \caption{Convergence on MNIST using a CNN.}
    \label{app:fig:MNIST}
\end{figure}

\begin{figure}
    \centering
    \begin{subfigure}{\sizefig\textwidth}
        \centering
        \includegraphics[width=1\textwidth]{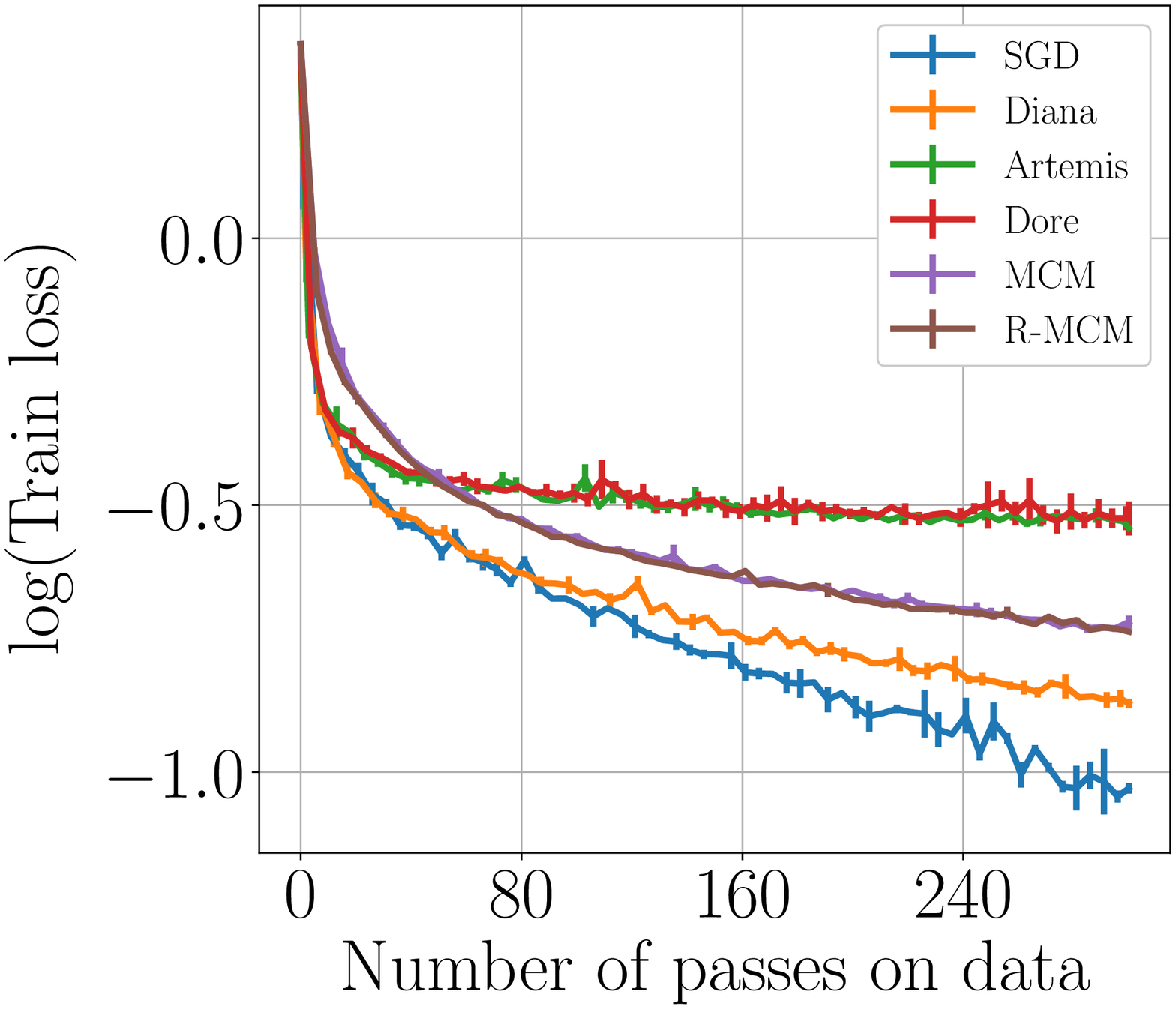}
        \caption{X axis in \# iterations.\vspace{-0.5em}}
        \label{app:Fashion_MNIST_its}
    \end{subfigure}
    \begin{subfigure}{\sizefig\textwidth}
        \centering
        \includegraphics[width=1\textwidth]{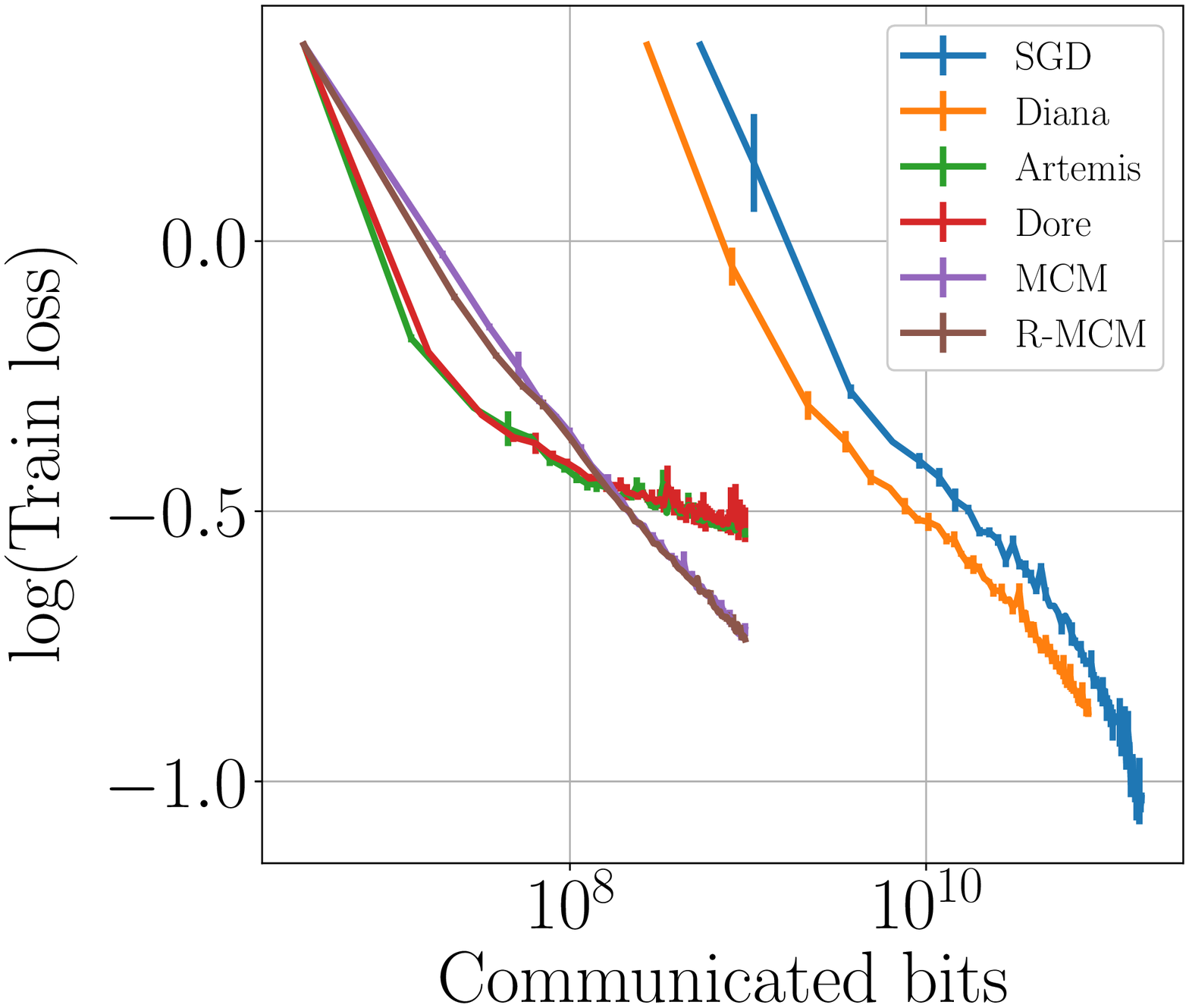}
        \caption{X axis in \# bits.\vspace{-0.5em}}
        \label{app:fig:Fashion_MNIST_bits}
    \end{subfigure}
    \caption{Convergence on Fashion-MNIST.}
    \label{app:fig:Fashion_MNIST}
\end{figure}

\begin{figure}
    \centering
    \begin{subfigure}{\sizefig\textwidth}
        \centering
        \includegraphics[width=1\textwidth]{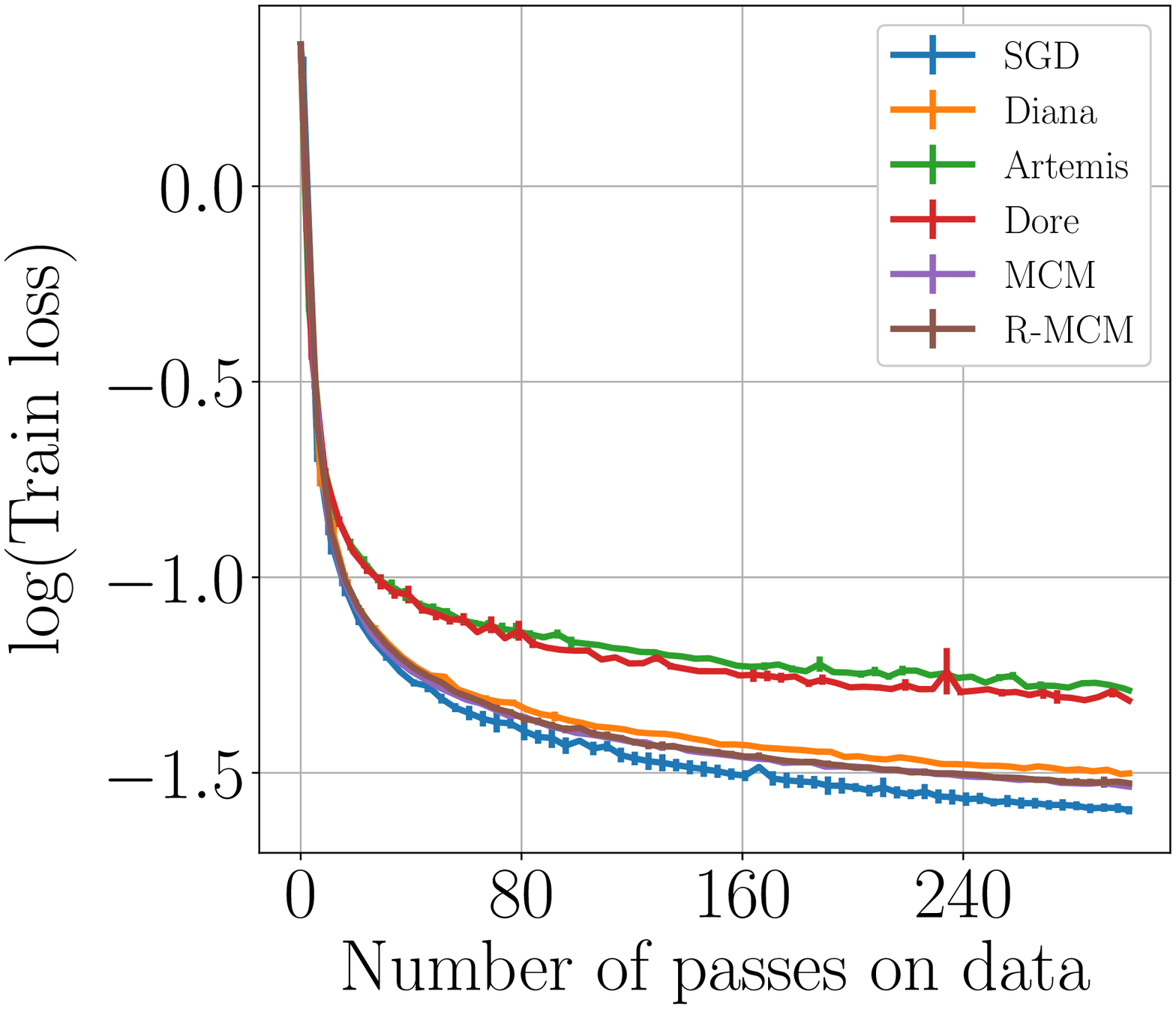}
        \caption{X axis in \# iterations.\vspace{-0.5em}}
        \label{app:fig:FE_MNIST_its}
    \end{subfigure}
    \begin{subfigure}{\sizefig\textwidth}
        \centering
        \includegraphics[width=1\textwidth]{pictures/exp/femnist/MNIST_CNN_m0_lr0.1_sup4_sdwn4_b128_wd0_norm-2_train_losses_bits.eps}
        \caption{X axis in \# bits.\vspace{-0.5em}}
        \label{app:fig:FE_MNIST_bits}
    \end{subfigure}
    \caption{Convergence on FE-MNIST.}
    \label{app:fig:FE_MNIST}
\end{figure}

\begin{figure}
    \centering
    \begin{subfigure}{\sizefig\textwidth}
        \centering
        \includegraphics[width=1\textwidth]{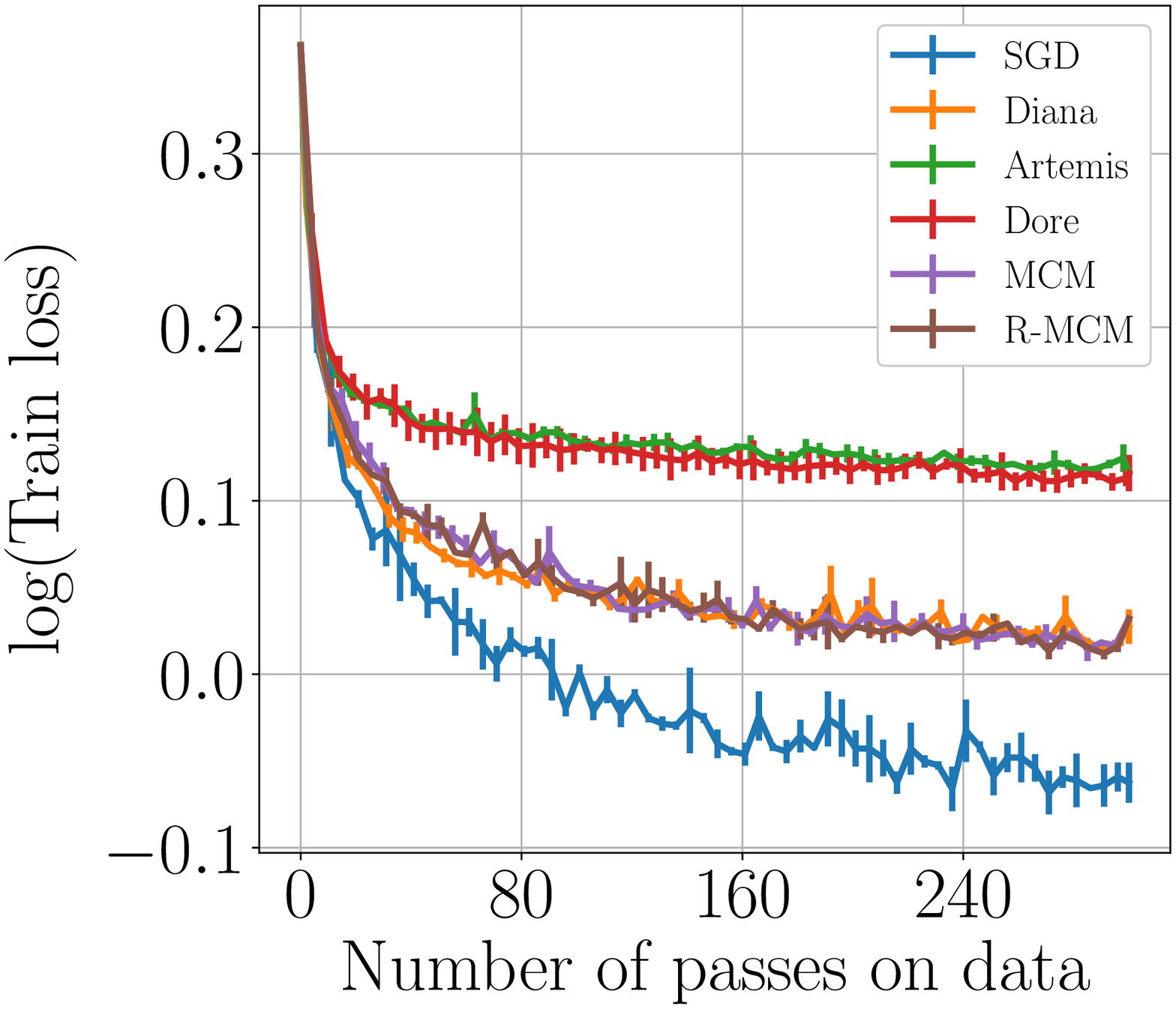}
        \caption{X axis in \# iterations.\vspace{-0.5em}}
        \label{app:fig:CIFAR10_its}
    \end{subfigure}
    \begin{subfigure}{\sizefig\textwidth}
        \centering
        \includegraphics[width=1\textwidth]{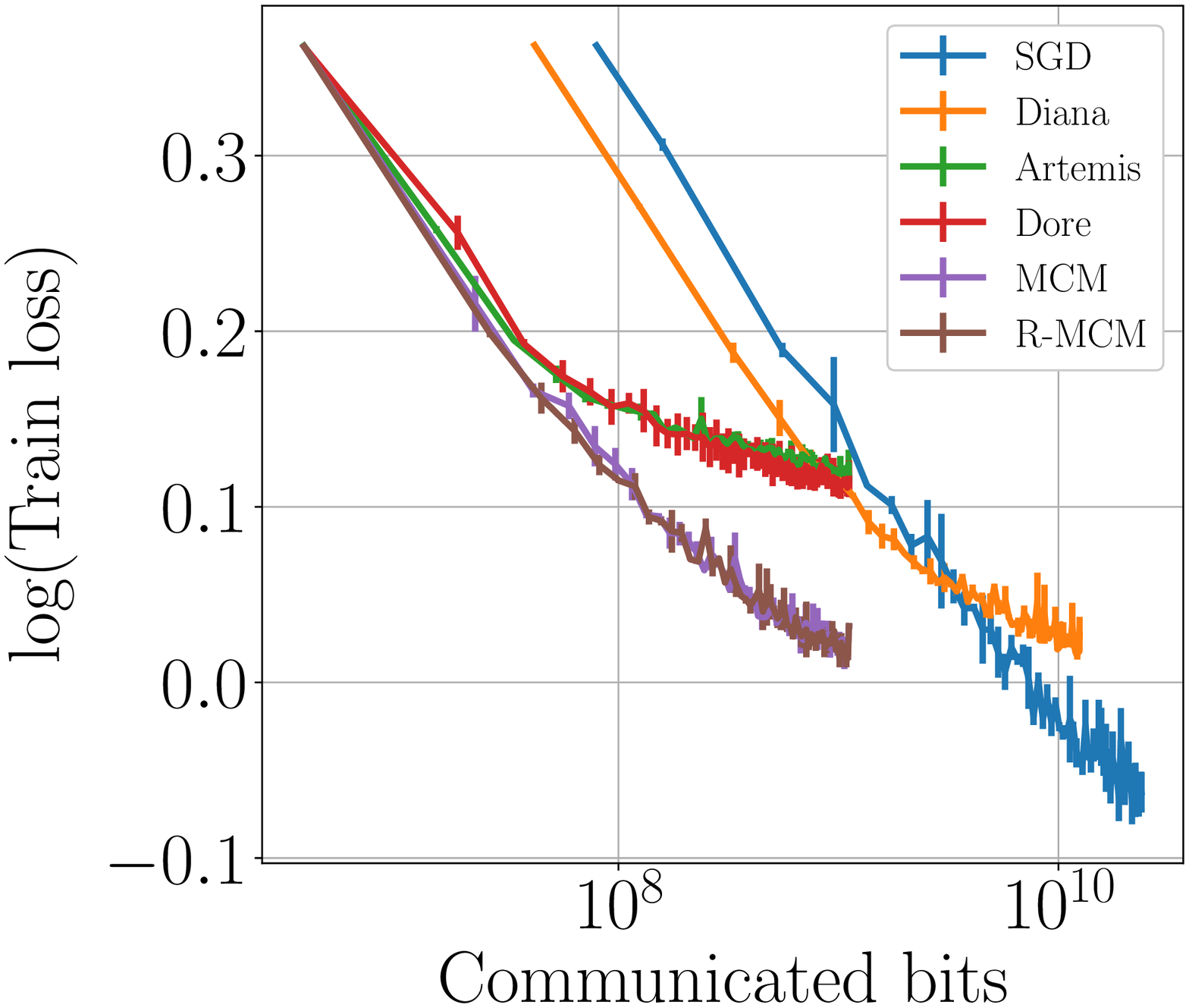}
        \caption{X axis in \# bits.\vspace{-0.5em}}
        \label{app:fig:CIFAR10_bits}
    \end{subfigure}
    \caption{Convergence on CIFAR10.}
    \label{app:fig:CIFAR10}
\end{figure}

\subsection{Wall clock time} 
\label{app:subsec:wall_clock_time}
We verified in our experiments that the downlink compression of $w_k-H_{k-1}$ on the central server does not lead to a noticeable overhead w.r.t. gradients computation and communications. Here, as experiments are performed in a \textit{simulated environment} there is no communication cost. In \Cref{app:tab:wall_clock_time} we report the computation time when training on FE-MNIST, this allows to highlight that compression only marginally increases the computation cost.

\begin{table}[!htp]
\caption{Wall clock time on FE-MNIST with $b=128$ and $s=2^2$.}
\label{app:tab:wall_clock_time}
\centering
\begin{tabular}{lc}
Compression regime & Computation time for 150 epoch  \\
\hline
 No compression (SGD) & 15421s   \\
 Compression on uplink & 16773s, ratio: 1.08   \\
 Compression on uplink and downlink & 16769s, ratio: 1.08 \\
\bottomrule
\end{tabular}
\end{table}

\subsection{Hardware and Carbon footprint}
\label{app:subsec:carbon_footprint}

As part as a community effort to report the carbon footprint of experiments, we describe in this subsection the hardware used and the total computation time.

We have two kind of experiments : for deep learning models we ran experiments on a GPU, and for linear/logistic regression on a CPU. We used an Intel(R) Xeon(R) CPU E5-2667 processor with 16 cores; and we used an Nvidia Tesla V100 GPU with $4$ nodes.

To generate all figures in this paper, our code ran (if run in a sequential mode) for $150$ hours on a CPU. In overall, we consider that the whole paper writing process required (code development, debugging, exploring settings ...) at least $6000$ hours end to end on the CPU. The carbon emissions caused by this work were subsequently evaluated with the \texttt{Green Algorithm}, built by \citet{lannelongue_green_2021}. It estimates our computations to generate around $100$kg of CO2, requiring $2.5$MWh. To compare, this corresponds to about $570$km by car. 

On the GPU, experiments require to be ran for around $140$ hours (if run in a sequential mode). In overall, we consider that the full paper writing process required at least $2800$ hours end to end on the GPU.
The \texttt{Green Algorithm} estimates our computations to generate $220$kg of CO2, requiring $5.7$MWh. To compare, this corresponds to about $1,270$km by car. 

\section{Technical results}
\label{app:sec:technical_results}

In this section, we provide some technical results required by our demonstration. In \Cref{app:subsec:basic_ineq} we recall classical inequalities and in \Cref{app:subsec:somme_lemmas} we present two preliminary lemmas.

In \Cref{app:sec:proof_for_ghost,app:sec:proofs_mcm,app:sec:technical_results}, for ease of notation we denote, for $k$ in $\N^*$, $\gwktilde = \ffrac{1}{N} \sum_\iN \gwkiHAT$. Furthermore we use the convention $\nabla F(w_{-1}) = 0$.

\subsection{Basic inequalities}
\label{app:subsec:basic_ineq}

In this subsection, we recall some very classical inequalities, for all $a, b \in \R^d$, $\beta > 0$ we have:
\begin{align}
    &\PdtScl{a}{b} \leq \ffrac{\SqrdNrm{a}}{2 \beta} +\ffrac{\beta \SqrdNrm{b}}{2} \label{app:basic_ineq:pdtscl} \,,\\
    &\SqrdNrm{a + b} \leq (1 + \ffrac{1}{\beta})\SqrdNrm{a} + (1 + \beta)\SqrdNrm{a}  \label{app:basic_ineq:trick_bound_norm} \,,\\
    &\SqrdNrm{a + b} \leq 2 \bigpar{ \SqrdNrm{a} + \SqrdNrm{b } } \,,\label{app:basic_ineq:nrm_sum} \\
    &|\PdtScl{a}{b}| \leq \| a \| \cdot \|b \| \text{\qquad(Cauchy-Schwarz inequality) \,,} \label{app:basic_ineq:cauchy_schwarz} \\
    &\PdtScl{a}{b} \leq \frac{1}{2} \bigpar{\SqrdNrm{a} + \SqrdNrm{b} - \SqrdNrm{a - b}} \text{\qquad(Polarization identity).} \label{app:basic_ineq:polarization}
\end{align}

Below, we recall Jensen's inequality.
\paragraph{Jensen inequality}
Let a probability space $(\Omega, \mathcal{A},\mathbf{P})$ with $\Omega$ a sample space, $\mathcal{A}$ an event space, and $\mathbf{P}$ a probability measure. Suppose that $X: \Omega \longrightarrow \R^d$ is a random variable, then for any convex function $f: \R^d \longrightarrow \R$ we have:
\begin{align}
\label{app:basic_ineq:jensen}
f\bigpar{\E(X)} \leq \E f(X) \,.
\end{align}

The next lemma will be used several times in the proofs.
\begin{lemma}
\label{app:lem:trickVariance}
Let a probability space $(\Omega, \mathcal{A},\mathbf{P})$ with $\Omega$ a sample space, $\mathcal{A}$ an event space, $\mathbf{P}$ a probability measure, and $\mathcal F$ a $\sigma-$algebra. For any $a \in \R^d$ and for any random vector in $\R^d$ we have:
\[
\E\left[\SqrdNrm{X-\E X }\right] \le \E\left[\SqrdNrm{X-a }\right]  \]
indeed $\E[X] = \argmin_{a\in \R^d} \E\left[\SqrdNrm{X-a }\right]$.
Similarly, for any random vector $Y$ in $\R^d$ which is $\mathcal{F}$-measurable, we have:
\[\Expec{\SqrdNrm{X-\Expec{X}{\mathcal F} }}{\mathcal F} \le \Expec{\SqrdNrm{X- Y}}{\mathcal F} \,.
\]
\end{lemma}

\begin{assumption}[Cocoercivity]
\label{app:asu:cocoercivity}
We suppose that for all $k$ in $\N$, stochastic gradients functions $(\g^i_k)_{i \in \llbracket 1,N \rrbracket}$ are $L$-cocoercive in quadratic mean.
That is, for  $k$ in $\N$,   $i$ in $\llbracket1, N \rrbracket$ and for all vectors $w_1, w_2$ in $\WW$, we have:
\[
\E [ \| \g^i_{k}(w_1) - \g_{k}^i(w_2) \|^2 ] \leq 
  L  \PdtScl{\nabla F_i (w_1) - \nabla F_i(w_2)}{w_1 - w_2} \,.
\]
This assumption is stronger than supposing convexity and $L$-smoothness of $F$.
\end{assumption}

The final proposition of this subsection presents two inequalities used in our demonstrations when invoking convexity or strong-convexity. They follow from \Cref{asu:cvx_or_strongcvx} and can be found in \cite{nesterov_introductory_2004}.
\begin{proposition}
\label{app:ineq_convex_demi_and_demi}
If a function $F$ is convex, then it satisfies for all $w$ in $\R^d$:
\begin{align}
\label{app:eq:convex}
\PdtScl{\nabla F(x)}{w-w_*} \geq \frac{1}{2} (F(w) - F(w_*)) + \ffrac{1}{2L} \SqrdNrm{\nabla F(w)} \,.
\end{align}

If a function $F$ is strongly-convex, then it satisfies for all $w$ in $\R^d$:
\begin{align}
\label{app:eq:strgly_convex}
\PdtScl{\nabla F(x)}{w-w_*} \geq \frac{1}{2} (F(w) - F(w_*)) + \ffrac{1}{2}\left( \mu \SqrdNrm{w - w_*} + \ffrac{1}{L} \SqrdNrm{\nabla F(w)} \right) \,.
\end{align}
\end{proposition}

\subsection{Two lemmas}
\label{app:subsec:somme_lemmas}

In this subsection, we give two lemmas required to prove the convergences of \ghost\footnote{\ghost~is defined in \Cref{app:sec:ghost_def}.}, \MCM~and \RMCM. 

The first lemma will be used to show that \MCM~indeed satisfies \Cref{thm:contraction_mcm}. The proof is straightforward from the definition of $w_k$ and $H_{k-1}$.

\begin{lemma}[Expectation of $w_k - H_{k-1}$]
For any $k$ in $\N^*$, the expectation of $(w_k - H_{k-1})$ conditionally to $\wkm$ can be decomposed as follows:
\label{app:lem:w_k_minus_H_k}
\begin{align*}
    \Expec{w_k - H_{k-1}}{\wkm} = (1- \alpha\dwn) (\wkm - H_{k-2}) - \gamma \Expec{\nabla F (\wkmhat)}{\wkm} \,.
\end{align*}
\end{lemma}

\begin{proof}
Let $k$ in $\N^*$, by definition and with \Cref{asu:expec_quantization_operator}:
\begin{align*}
    \Expec{w_k - H_{k-1}}{\wkm} &= \Expec{\wkm - \gamma \widehat{\g}_k(\wkmhat) - \bigpar{H_{k-2} + \alpha\dwn \mathcal{C}(\wkm - H_{k-2})}}{\wkm} \\
    &= (\wkm - H_{k-2}) - \alpha\dwn\Expec{\C\bigpar{\wkm - H_{k-2}}}{\wkm} - \gamma \Expec{\gwkHAT}{\wkm} \,,
\end{align*}

from which the result follows.

\end{proof}

The following lemma provides a control of the impact of the uplink compression. It decomposes the squared-norm of stochastic gradients into two terms: 1) the true gradient 2) the variance of the stochastic gradient $\sigma^2$. 
\begin{lemma}[Squared-norm of stochastic gradients]
For any $k$ in $\N^*$, the second moment and variance of the compressed gradients can be bounded a.s.:
\label{app:lem:grad_sto_to_grad}
\begin{align*}
    &\Expec{\SqrdNrm{\gwkHAT}}{\wkmhat} \leq \bigpar{1 + \ffrac{\omgC\up}{N}} \SqrdNrm{\nabla F (\wkmhat)} + \ffrac{\sigma^2(1 + \omgC\up)}{Nb} \,,\\
    &\Expec{\SqrdNrm{\gwkHAT - \nabla F (\wkmhat)}}{\wkmhat} \leq \ffrac{\omgC\up}{N} \SqrdNrm{\nabla F (\wkmhat)} + \ffrac{\sigma^2(1 + \omgC\up)}{Nb} \,.
\end{align*}
Interpretation:
\begin{itemize}
    \item If $\omgC\up = 0$ (i.e. no up compression), the variance corresponds to a mini-batch.
    \item If $\sigma = 0$ and $N=1$ (i.e. full batch descent with a single device), it becomes: $\E \left[ \SqrdNrm{\C(\nabla F(\wkm)) - \nabla F(\wkm) } \right] \leq \omgC\up \SqrdNrm{\nabla F(\wkm)}$ which is consistent with \Cref{asu:expec_quantization_operator}.
\end{itemize}
\end{lemma}

\begin{proof}
Let $k$ in $\N^*$, then $\Expec{\SqrdNrm{\gwkHAT}}{\wkmhat} = \SqrdNrm{\nabla F(\wkmhat)} + \Expec{\SqrdNrm{\gwkHAT - \nabla F(\wkmhat)}}{\wkmhat}$.

Secondly:
\begin{align*}
    &\Expec{\SqrdNrm{\gwkHAT - \nabla F(\wkmhat)}}{\wkmhat} \\
    &\qquad= \Expec{\SqrdNrm{\ffrac{1}{N} \sum_\iN \bigpar{\gwkiHAT - \nabla F(\wkmhat)}}}{\wkmhat}\\
    &\qquad= \Expec{\SqrdNrm{\ffrac{1}{N} \sum_\iN \bigpar{\gwkiHAT -  \gwkihat + \gwkihat  - \nabla F(\wkmhat)}}}{\wkmhat} \\
    &\qquad= \Expec{\SqrdNrm{\ffrac{1}{N} \sum_\iN \bigpar{\gwkiHAT -  \gwkihat}}}{\wkmhat} \\
    &\qqquad+ \Expec{\SqrdNrm{\ffrac{1}{N} \sum_\iN \bigpar{\gwkihat  - \nabla F(\wkmhat)}}}{\wkmhat}\,,
\end{align*}
the inner product being null.

Next expanding the squared norm again, and because the two sums of inner products are null as the stochastic oracle and uplink compressions are independent:
\begin{align*}
    \Expec{\SqrdNrm{\gwkHAT - \nabla F(\wkmhat)}}{\wkmhat} &= \frac{1}{N^2} \sum_\iN \Expec{\SqrdNrm{\gwkiHAT -  \gwkihat}}{\wkmhat} \\
    &\qquad+ \frac{1}{N^2} \sum_\iN \Expec{\SqrdNrm{\gwkihat  - \nabla F(\wkmhat)}}{\wkmhat} \,.
\end{align*}

Then, for any $i$ in $\llbracket 1, N \rrbracket$ as $\Expec{\SqrdNrm{\gwkiHAT -  \gwkihat}}{\wkmhat} = \Expec{\Expec{\SqrdNrm{\gwkiHAT -  \gwkihat}}{\g_k^i}}{\wkmhat}$, and using \Cref{asu:expec_quantization_operator} we have:
\begin{align*}
    \Expec{\SqrdNrm{\gwkHAT - \nabla F(\wkmhat)}}{\wkmhat} &= \frac{\omgC\up}{N^2} \sum_\iN \Expec{\SqrdNrm{\gwkihat}}{\wkmhat} \\
    &\qquad+ \frac{1}{N^2} \sum_\iN \Expec{\SqrdNrm{\gwkihat  - \nabla F(\wkmhat)}}{\wkmhat} \,.
\end{align*}

Furthermore $\Expec{\SqrdNrm{\gwkihat}}{\wkmhat} = \Expec{\SqrdNrm{\gwkihat - \nabla F(\wkmhat)}}{\wkmhat} + \SqrdNrm{\nabla F(\wkmhat)} $, and using \Cref{asu:noise_sto_grad}:
\begin{align*}
    \Expec{\SqrdNrm{\gwkHAT - \nabla F(\wkmhat)}}{\wkmhat} &= \ffrac{\omgC\up}{N} \SqrdNrm{\nabla F (\wkmhat)} + \ffrac{\sigma^2 (1 + \omgC\up)}{Nb} \,,
\end{align*}
from which we derive the two inequalities of the lemma.

\end{proof}

\section{The \ghost~algorithm}
\label{app:sec:proof_for_ghost}

\subsection{Motivation, definition of \ghost~and proof sketch}\gs
\label{app:sec:ghost_def}

In this section, to convey the best understanding of the theorems and the spirit of the proof, we define a \textit{ghost} algorithm (that is impossible to implement in practice). \ghost~is introduced only to get some intuition of the theoretical insight. 
\begin{definition}[\ghost~algorithm]
\label{def:ghost}
The \ghost~algorithm is defined as follows, for $k \in \N$, for all $i \in \llbracket 1, N \rrbracket$ we have:
\begin{align}
\label{eq:ghost}
          w_{k+1} = w_k - \gamma \ffrac{1}{N} \sum_{i=1}^N \widehat{\g}_{k+1}^i(\widehat{w}_k) \text{\quad and\quad} \widehat{w}_{k+1} = w_k - \gamma \C\dwn \left(\ffrac{1}{N} \sum_{i=1}^N \widehat{\g}_{k+1}^i(\widehat{w}_k) \right) \,.
\end{align}
While the global model is unchanged ($1^{\text{st}}$ line),  the local model $\widehat{w}_{k+1}$ ($2^{\text{nd}}$ line) is updated using the global model~$w_k$ at the previous step, which is not available locally. 
\end{definition}

In the following, we give the main results for \ghost~and complete them with a sketch of proof. Demonstrations are all in the next subsection.

The following Proposition, provides the control of the variance of the local model for \ghost.

\begin{proposition}
\label{app:prop:ghost_contraction}
Consider the \ghost~update in \cref{eq:ghost}, under \Cref{asu:expec_quantization_operator,asu:smooth,asu:noise_sto_grad}, for all $k$ in $\N$ with the convention $\nabla F(w_{-1}) = 0$:
\begin{align*}
    \Expec{\SqrdNrm{w_k - \widehat{w}_k}}{\wkmhat} &\leq \gamma^2\omgC\dwn \bigpar{1 + \ffrac{\omgC\up}{N}}\SqrdNrm{\nabla F(\wkmhat)} + \ffrac{\gamma^2 \omgC\dwn (1 + \omgC\up)\sigma^2}{Nb} \,.
\end{align*}
\end{proposition}

\begin{proof}
The proof of \Cref{app:prop:ghost_contraction} is straightforward using \Cref{def:ghost}.
Let $k$ in $\N$, by \Cref{def:ghost} we have:
\begin{align*}
    \SqrdNrm{w_k - \widehat{w}_k} &= \SqrdNrm{\bigpar{\wkm - \gamma \C\dwn \left(\ffrac{1}{N} \sum_{i=1}^N \gwkiHAT \right)} - \bigpar{\wkm - \gamma \ffrac{1}{N} \sum_{i=1}^N \gwkiHAT }} \\
    &= \gamma^2 \SqrdNrm{\C\dwn \left(\ffrac{1}{N} \sum_{i=1}^N \gwkiHAT \right) - \ffrac{1}{N} \sum_{i=1}^N \gwkiHAT} \,.
\end{align*}

Taking expectation w.r.t. down compression, as $\frac{1}{N} \sum_\iN \gwkiHAT$ is $w_k$-measurable:
\begin{align*}
    \Expec{\SqrdNrm{w_k - \widehat{w}_k}}{w_k} &= \gamma^2 \omgC\dwn \Expec{\SqrdNrm{\ffrac{1}{N} \sum_{i=1}^N \gwkiHAT}}{w_k} = \gamma^2 \omgC\dwn \SqrdNrm{\gwkHAT} \,,
\end{align*}

and \Cref{app:lem:grad_sto_to_grad} gives the upper bound $\Expec{\SqrdNrm{\gwkHAT}}{\wkmhat}$.
\end{proof}
 
The takeaway from this Proposition is that we are able to bound the variance of the local model by an affine function of the squared norm of the \textit{previous} stochastic gradients $\nabla F(\wkmhat)$. For \ghost~only the previous gradient is involved, while for \MCM, we obtain an additional recursive process.
 
To obtain the convergence, we then follow the classical approach~\cite{mania_perturbed_2016}, expanding $\E \SqrdNrm{w_k - w_*} $ as $\E\SqrdNrm{\wkm - w_*} - 2\gamma\E \PdtScl{\nabla F (\wkmhat)}{\wkm - w_*}    
    + \gamma^2 \E\left[\SqrdNrm{\widehat{\g}_k(\widehat w_{k-1})}\right] $. The  critical aspect is that the inner product does not directly result in a contraction, as the support point of the gradient differs from $\wkm$. Using the fact that $\Expec{\wkmhat}{\wkm}=\wkm $, we further decompose it as 
    \begin{align}
         -2\gamma\E\PdtScl{\nabla F (\wkmhat)}{\wkmhat - w_*} +2\gamma \E \PdtScl{\nabla F (\wkmhat) - \nabla F(\wkm)}{\wkm - \wkmhat} \,. \label{eq:ligne2}
    \end{align}
The first part of \cref{eq:ligne2}, corresponds to a ``strong contraction'': by (strong-)convexity, we can upper bound it by $-2\gamma (\mu \SqrdNrm{\wkmhat-w_*} + F(\wkmhat)-F_* ) $, which is on average larger than $-2\gamma (\mu \SqrdNrm{\wkm-w_*} + F(\wkm)-F_* ) $ (Jensen's inequality). Moreover, as the function is smooth and convex, it can also be upper bounded by $-2\gamma \SqrdNrm{\nabla F(\wkmhat)}/L$. This is a crucial term: we ``gain'' something of the order of a squared norm of the gradient at $\wkmhat$, which will \textit{in fine} compensate the variance of the local model. The second part of \cref{eq:ligne2}, corresponds to a positive residual term,  proportional to the variance of the compressed model, that can be controlled thanks to \Cref{app:prop:ghost_contraction} (at $\wkm$!). Putting things together, we get, in the convex case ($\mu=0$):

\begin{theorem}[Contraction for \ghost, convex case]
\label{app:thm:contraction_ghost}
  Under \Cref{asu:cvx_or_strongcvx,asu:expec_quantization_operator,asu:smooth,asu:noise_sto_grad}, with $\mu=0$, if $\gamma L (1+\omgC\up/N) \le \frac{1}{2}$.
  \begin{align*}
      \E{\SqrdNrm{w_k - w_*}} &\leq  \E\SqrdNrm{\wkm - w_*} - \gamma \E (F(\wkm) - F_*) - \frac{\gamma}{2 L}  \E\left[\SqrdNrm{\nabla F(\wkmhat)}\right]  \\ 
    &  + 2\gamma^3 \omgC\dwn L \left(1 + \frac{\omgC\up}{N} \right)\E \SqrdNrm{\nabla F(\hat w_{k-2})} + \gamma^2 \frac{ (1 + \omgC\up)\sigma^2}{Nb} \left(1+ 2 \gamma L \omgC\dwn\right).
  \end{align*}
\end{theorem}
We can make the following observations:
\begin{enumerate}[topsep=0pt,itemsep=1pt,leftmargin=*,noitemsep]
\item At step $k$, the residual can be upper bounded by a constant times squared norm of the gradient at point $\hat w_{k-2}$. When using recursively this upper bound, if $ 2\gamma^3 \omgC\dwn L (1 + \omgC\up/N ) \le  {\gamma}/(2 L) $, then these terms cancel out. This is equivalent to $2\gamma  L \sqrt{\omgC\dwn \left(1 + \omgC\up/N \right)} \le 1$. It is natural to chose $ \gamma \le 1/ (2L \max( 1+\omgC\up/N, 1+\omgC\dwn) )$. 
    \item The bound is in fact proved conditionally to $\wkm$, recursive conditioning is required to propagate the inequality. We carefully handle conditioning in the proofs.
\end{enumerate}

\subsection{Convergence of \ghost, complete proof}
\label{app:subsec:proof_thms_ghost}

In this subsection, we provide the complete proof of convergence for \ghost. Thus in the following demonstration, we give the key concepts required to later prove the convergence of \MCM. 

\fbox{
\begin{minipage}{\textwidth}
\begin{theorem}[Convergence of \ghost, convex case]
\label{app:thm:cvgce_ghost}
Under \Cref{asu:cvx_or_strongcvx,asu:expec_quantization_operator,asu:smooth,asu:noise_sto_grad} with $\mu=0$ (convex case),  for all $k$ in $\N$, defining $V_k := \FullExpec{w_k - w_*} + \ffrac{\gamma}{2 L} \FullExpec{\SqrdNrm{\nabla F(\wkmhat)}} + 2 \gamma L \FullExpec{\SqrdNrm{\widehat{w}_k - w_k}}$, we have:
\begin{align*}
V_k \leq V_{k-1} - \gamma \FullExpec{F(\wkm) - F(w_*)} + \frac{\gamma^2 \sigma^2 \Phi^\GG(\gamma)}{Nb}\,,
\end{align*}
with $\Phi^\GG(\gamma) := (1 + \omgC\up) (1 + 2 \gamma L \omgC\dwn) $.
\end{theorem}
\end{minipage}
}

\begin{remark}
This result is similar to \cref{eq:Lyapunov-convex} but with a different function $\Phi^\GG$ that has a weaker dependency on $\omgC\dwn$.
\end{remark}

\begin{proof}
Let $k$ in $\N^*$, by definition:
\begin{align*}
    \SqrdNrm{w_k - w_*} \leq \SqrdNrm{\wkm - w_*} - 2\gamma \PdtScl{\gwkHAT}{\wkm - w_*} + \gamma^2 \SqrdNrm{\gwkHAT} \,.
\end{align*}
Next, we expend the inner product as following:
\begin{align*}
    \SqrdNrm{w_k - w_*} \leq \SqrdNrm{\wkm - w_*} - 2\gamma \PdtScl{\gwkHAT}{\wkmhat - w_*} - 2\gamma \PdtScl{\gwkHAT}{\wkm - \wkmhat} + \gamma^2 \SqrdNrm{\gwkHAT} \,.
\end{align*}

Taking expectation conditionally to $\wkm$, and using $\Expec{ \gwkHAT}{\wkm}= \Expec{\Expec{\gwkHAT}{\wkmhat}}{\wkm} =\Expec{\nabla F(\wkmhat)}{\wkm}$, we obtain:
\begin{align*}
    \Expec{\SqrdNrm{w_k - w_*}}{\wkm} &\leq \SqrdNrm{\wkm - w_*} - \Expec{2\gamma \PdtScl{\nabla F (\wkmhat)}{\wkmhat - w_*}}{\wkm} \\
    &\qquad- 2\gamma \Expec{\PdtScl{\nabla F (\wkmhat)}{\wkm - \wkmhat}}{\wkm} \\
    &\qquad+ \gamma^2 \Expec{\SqrdNrm{\gwkHAT}}{\wkm} \,.
\end{align*}

Then invoking \Cref{app:lem:grad_sto_to_grad} to upper bound the squared norm of the stochastic gradients, and noticing that $\Expec{\PdtScl{\nabla F(\wkm)}{\wkmhat-\wkm}}{\wkm}=0$ leads to:
\begin{align}
    \Expec{\SqrdNrm{w_k - w_*}}{\wkm} &\leq \SqrdNrm{\wkm - w_*} - 2\gamma\Expec{ \PdtScl{\nabla F (\wkmhat)}{\wkmhat - w_*}}{\wkm} \nonumber \\
    &\qquad - 2\gamma\Expec{ \PdtScl{\nabla F (\wkmhat) - \nabla F(\wkm)}{\wkm - \wkmhat}}{\wkm} \label{eq:decomp} \\
    &\qquad + \gamma^2 \bigpar{ \Big(1 + \ffrac{\omgC\up}{Nb} \Big) \Expec{\SqrdNrm{\nabla F (\wkmhat)}}{\wkm} + \ffrac{\sigma^2 \bigpar{1 + \omgC\up}}{Nb} } \,. \nonumber
\end{align}

In the upper inequality:
\begin{enumerate}
    \item the term $\Expec{ \PdtScl{\nabla F (\wkmhat)}{\wkmhat - w_*}}{\wkm}$ allows the ``strong contraction''
    \item the terms $\Expec{ \PdtScl{\nabla F (\wkmhat) - \nabla F(\wkm)}{\wkm - \wkmhat}}{\wkm}$ and $\Expec{\SqrdNrm{\nabla F (\wkmhat)}}{\wkm}$ are two positives terms that we treat as residuals.
    \item the last term $\sigma^2 \bigpar{1 + \omgC\up } / (Nb)$ is due to the stochastic noise.
\end{enumerate}

Now using Cauchy-Schwarz inequality (\cref{app:basic_ineq:cauchy_schwarz}) and smoothness: 
\begin{align*}
    &- \Expec{2\gamma \PdtScl{\nabla F (\wkmhat) - \nabla F(\wkm) }{\wkm - \wkmhat}}{\wkm} \\
    &\qqquad= 2\gamma \Expec{\PdtScl{\nabla F (\wkmhat) - \nabla F(\wkm)}{\wkmhat - \wkm}}{\wkm} \\
    &\qqquad\leq 2 \gamma L \Expec{\SqrdNrm{\wkmhat - \wkm}}{\wkm} \,,
\end{align*}

and thus:
\begin{align}
\label{app:eq:ghost_before_convexity}
\begin{split}
    \Expec{\SqrdNrm{w_k - w_*}}{\wkm} &\leq \SqrdNrm{\wkm - w_*} - 2\gamma\Expec{ \PdtScl{\nabla F (\wkmhat)}{\wkmhat - w_*}}{\wkm} \\
    &\qquad + 2\gamma L \Expec{\SqrdNrm{\wkmhat - \wkm}}{\wkm} \\
    &\qquad + \gamma^2 \Big(1 + \ffrac{\omgC\up}{N} \Big) \Expec{\SqrdNrm{\nabla F (\wkmhat)}}{\wkm} + \ffrac{\gamma^2 \sigma^2(1 + \omgC\up)}{Nb}  \,. 
\end{split}
\end{align}

Now, using convexity with \cref{app:ineq_convex_demi_and_demi}:
\begin{align*}
    \Expec{\SqrdNrm{w_k - w_*}}{\wkm} &\leq \SqrdNrm{\wkm - w_*} \\
    &\qquad- \gamma\Expec{ \bigpar{F(\wkmhat) - F(w_*) + \frac{1}{L} \SqrdNrm{\nabla F(\wkmhat)}}}{\wkm} \\
    &\qquad + 2\gamma L \Expec{\SqrdNrm{\wkmhat - \wkm}}{\wkm} \\
    &\qquad + \gamma^2 \Big(1 + \ffrac{\omgC\up}{N} \Big) \Expec{\SqrdNrm{\nabla F (\wkmhat)}}{\wkm} + \ffrac{\gamma^2 \sigma^2(1 + \omgC\up)}{Nb}  \,. 
\end{align*}

Taking the full expectation (without conditioning over any random vectors), and because invoking Jensen inequality (\ref{app:basic_ineq:jensen}) leads to $\FullExpec{F(\wkmhat)}\geq \FullExpec{F(\wkm)}$, we finally obtain this intermediate result:
\begin{align}
\label{app:eq:equation_common_to_all}
\begin{split}
    \FullExpec{\SqrdNrm{w_k - w_*}} &\leq \FullExpec{\SqrdNrm{\wkm - w_*}} - \gamma \bigpar{\FullExpec{F(\wkm)} - F(w_*)} \\
    &\qquad - \frac{\gamma}{2L} \FullExpec{\SqrdNrm{\nabla F(\wkmhat)}} \\
    &\qquad + 2\gamma L \FullExpec{\SqrdNrm{\wkmhat - \wkm}} + \ffrac{\gamma^2\sigma^2(1 + \omgC\up)}{Nb} \,,
\end{split}
\end{align}

where we considered that  $\gamma L(1+\omgC\up/N) \le 1/2$, which implies that $\gamma\bigpar{1 - \gamma L \bigpar{1 + \ffrac{\omgC\up}{N}}} \geq \ffrac{\gamma}{2}$.

Remark that \cref{app:eq:equation_common_to_all} is valid for both \ghost~and \MCM, and that the proof of \MCM~will follow the same initial line.

With \Cref{app:prop:ghost_contraction}:
\begin{align}
\label{app:eq:ghost_contraction_used_in_lyapunov}
    \Expec{\SqrdNrm{w_k - \widehat{w}_k}}{\wkmhat} &\leq \gamma^2\omgC\dwn \bigpar{1 + \ffrac{\omgC\up}{N}}\SqrdNrm{\nabla F(\wkmhat)} + \ffrac{\gamma^2 \omgC\dwn (1 + \omgC\up)\sigma^2}{Nb} \,.
\end{align}

Defining $V_k := \FullExpec{w_k - w_*} + \ffrac{\gamma}{2 L} \FullExpec{\SqrdNrm{\nabla F(\wkmhat)}} + \cst \FullExpec{\SqrdNrm{\widehat{w}_k - w_k}}$ with $C = 2\gamma L$, and combining this two equations as following $(\ref{app:eq:equation_common_to_all}) + C (\ref{app:eq:ghost_contraction_used_in_lyapunov})$ leads to:
\begin{align*}
    &\FullExpec{\SqrdNrm{w_k - w_*}} + \cst\FullExpec{\SqrdNrm{\wkmhat - \wkm}} + \frac{\gamma}{2L} \FullExpec{\SqrdNrm{\nabla F(\wkmhat)}}\\
    &\qquad\leq \FullExpec{\SqrdNrm{\wkm - w_*}} - \gamma \bigpar{\FullExpec{F(\wkm)} - F(w_*)}  \\
    &\qqquad+ 2\gamma L \FullExpec{\SqrdNrm{\wkmhat - \wkm}} + \ffrac{\gamma^2\sigma^2(1 + \omgC\up)}{Nb} \\
    &\qqquad+ 2 \gamma L \times \gamma^2\omgC\dwn \bigpar{1 + \ffrac{\omgC\up}{N}} \SqrdNrm{\nabla F(\wkmhat)} + 2 \gamma L \times \ffrac{\gamma^2 \omgC\dwn (1 + \omgC\up) \sigma^2}{Nb} \,.
\end{align*}

To ensure a contraction of the Lyapunov function we require:
\[
\gamma^2\omgC\dwn \bigpar{1 + \ffrac{\omgC\up}{N}} \leq \frac{\gamma}{2L} \Longleftrightarrow \gamma L \leq \ffrac{1}{2 \sqrt{\omgC\dwn \bigpar{1 + \ffrac{\omgC\up}{N}} }}
\]

Under this condition, we obtain:
\begin{align*}
V_k \leq V_{k-1} - \gamma \FullExpec{F(\wkm) - F(w_*)} + \frac{\gamma^2 \sigma^2 \Phi^\GG(\gamma)}{Nb}\,,
\end{align*}
with $\Phi^\GG(\gamma) := (1 + \omgC\up) (1 + 2 \gamma L \omgC\dwn) $.

By recurrence and for $k=K$:
\begin{align*}
    V_K &\leq V_{0} - \sum_{k=1}^K \gamma  \FullExpec{F(\wkm) - F(w_*)} + \sum_{k=1}^K \frac{\gamma^2 \sigma^2 \Phi^\GG(\gamma)}{Nb} \,,
\end{align*}

which leads to:
\begin{align*}
    \ffrac{1}{K} \sum_{k=1}^K \FullExpec{F(\wkm) - F(w_*)} &\leq \ffrac{V_{0} - V_k}{\gamma K} + \frac{\gamma \sigma^2 \Phi^\GG(\gamma)}{Nb}  \,.
\end{align*}

Finally, for any $K$ in $\N^*$, with $\gamma L \leq \min \bigg\{ \ffrac{1}{2 \bigpar{ 1 + \ffrac{\omgC\up}{N}}}, \ffrac{1}{2 \sqrt{ \omgC\dwn \bigpar{ 1 + \ffrac{\omgC\up}{N}}} }\bigg\}$ we have:
\begin{align*}
    \ffrac{\gamma}{K} \sum_{t=1}^K \FullExpec{F(w_{t}) - F(w_*)} \leq \ffrac{\SqrdNrm{w_0 - w_*}}{K} + \frac{\gamma \sigma^2 \Phi^\GG(\gamma)}{Nb} \,.
\end{align*}

Note that the bound of $\gamma L$ encompass the case $\omgC\dwn = 0$ (i.e. no downlink compression), but in the general case of bidirectional compression, we nearly always have $\omgC\dwn > 1$, and thus the dominant term is in fact $\ffrac{1}{2 \sqrt{\omgC\dwn \bigpar{ 1 + \ffrac{\omgC\up}{N}}}}$.

And by Jensen, it implies that:
\begin{align*}
    \FullExpec{F(\bar{w}_K) - F(w_*)} \leq \ffrac{\SqrdNrm{w_{0} - w_*}}{\gamma K} + \ffrac{\gamma \sigma^2 \Phi(\gamma) }{N b}  \text{\quad with $\Phi^\GG(\gamma) := (1 + \omgC\up)(1 + 2\omgC\dwn \gamma L)$} \,.
\end{align*}

\end{proof}

\section{Proofs for \MCM~(and \RMCM)}
\label{app:sec:proofs_mcm}

In this section, we provide the proofs for \MCM~in the convex, strongly-convex, and non-convex cases in respectively \Cref{app:thm:cvgce_mcm_strongly_convex,app:thm:mcm_non_convex,app:thm:cvgce_mcm_convex}. The proofs for \RMCM~(see \Cref{thm:cvgce_rmcm}) are identical and only require to adapt notations as explained in \cref{app:subsec:proofs_rmcm}. 

We denote for $\gamma$ in $\R$, $\Phi(\gamma)~:=~(1 + \omgC\up) \bigpar{1 + \frac{8 \gamma L \omgC\dwn}{\alpha\dwn}}$, for $k$ in $\N$, $\dwnMemTerm_k =  \SqrdNrm{w_k - H_{k-1}}$ and we define $\gamma_{\max}$ such that:
$$\gamma_{\max} L \leq \min \bigg\{ \ffrac{1}{8 \omgC\dwn},  \ffrac{1}{2 \bigpar{ 1 + \ffrac{\omgC\up}{N}}}, \ffrac{1}{4\sqrt{ \ffrac{\omgC\dwn}{\alpha\dwn} \bigpar{ \ffrac{1}{\alpha\dwn} + \ffrac{\omgC\up}{N}}}} \bigg\} \,.$$
 Note that this is equivalent to notations given in \Cref{sec:theory} if we take $\alpha\dwn = 1/8\omgC\dwn$.

\subsection{Control of the Variance of the local model for \MCM~(Theorem 3)}
\label{app:subsec:proof_prop_mcm_assumption}

In this section, we provide a control of the variance of the local model for \MCM, as done previously in \Cref{app:prop:ghost_contraction} for \ghost: this corresponds to \Cref{thm:contraction_mcm}. The demonstration is more complex than for \ghost~and it highlights the trade-offs for the learning rate $\alpha\dwn$. The demonstration builds a bias-variance decomposition of $\SqrdNrm{\Omega_k} = \SqrdNrm{w_k - H_k}$. The variance is then decomposed in three terms, as a result we will need to compute four terms:
\begin{align}
\label{app:eq:bias_var_decomposition}
 \SqrdNrm{w_k - H_{k-1}} = \text{Bias}^2 + 2\gamma^2(\text{Var}_{11} + \text{Var}_{12}) + 2 \alpha\dwn^2 \text{Var}_2\,.
\end{align}

\fbox{
\begin{minipage}{\textwidth}
\begin{theorem}
\label{app:thm:contraction_mcm}
Consider the \MCM~update as in \cref{eq:model_compression_eq_statement}. Under \Cref{asu:expec_quantization_operator,asu:smooth,asu:noise_sto_grad} with $\mu=0$,  if $\gamma \leq ({8\omgC\dwn L})^{-1}$ and $\alpha\dwn \leq (8 \omgC\dwn)^{-1}$, then for all $k$ in $\N$:
\begin{align*}
    \FullExpec{\dwnMemTerm_{k}} &\leq \bigpar{1 - \ffrac{\alpha\dwn}{2}} \FullExpec{\dwnMemTerm_{k-1}} + 2 \gamma^2\bigpar{\frac{1}{\alpha\dwn} + \ffrac{\omgC\up}{N}} \FullExpec{\SqrdNrm{\nabla F(\wkmhat)}} + \ffrac{2\gamma^2 \sigma^2(1+\omgC\up)}{Nb} \,.
\end{align*}
\end{theorem}
\end{minipage}
}

\begin{proof}
Let $k$ in $\N$, we recall that by definition: 
\[
\left\{
    \begin{array}{ll}
        \Omega_k = w_k - H_{k-1} \\
    	\widehat{\Omega}_k = \C\dwn(\Omega_k) \\
    	\widehat{w}_k = \widehat{\Omega}_k + H_{k-1} \,.
    \end{array}
\right.
\]

We start the proof by introducing $\SqrdNrm{\Omega_k}$:
\begin{align*}
    \Expec{\SqrdNrm{w_k - \widehat{w}_k}}{w_k} = \Expec{\SqrdNrm{\widehat{\Omega}_k - \Omega_k}}{w_k} \leq \omgC\dwn \SqrdNrm{\Omega_k} \,.
\end{align*}

Next, we perform a bias-variance decomposition:
\begin{align*}
\SqrdNrm{\Omega_k} = \SqrdNrm{w_k - H_{k-1}} &= \SqrdNrm{w_k - H_{k-1} - \Expec{w_k - H_{k-1}}{\wkm}}  \\
&\quad + \SqrdNrm{\Expec{w_k - H_{k-1}}{\wkm}} \\
&\quad + 2 \PdtScl{w_k - H_{k-1} - \Expec{w_k - H_{k-1}}{\wkm}}{\Expec{w_k - H_{k-1}}{\wkm}} \,,
\end{align*}

taking expectation w.r.t.~$\wkm$:
\begin{align*}
\Expec{\dwnMemTerm_{k}}{\wkm} &= \underbrace{\Expec{\SqrdNrm{w_k - H_{k-1} - \Expec{w_k - H_{k-1}}{\wkm}}}{\wkm}}_{\text{Var}} \\
&\qquad + \underbrace{\SqrdNrm{\Expec{w_k - H_{k-1}}{\wkm}}}_{\text{Bias}^2} \,.
\end{align*}

The first term is the variance $\text{Var}$, and the second term corresponds to the squared bias $\text{Bias}^2$.

Let's handle first the variance, by definition:
\begin{align*}
    \text{Var} &= \Expec{\SqrdNrm{w_k - H_{k-1} - \Expec{w_k - H_{k-1}}{\wkm}}}{\wkm} \\
    &= \E \left[ \| \wkm - \gamma \gwkHAT - H_{k-2} - \alpha\dwn \C(\wkm - H_{k-2})  \right. \\    
    &\qquad \left.-  \wkm - \gamma \Expec{\gwkHAT}{\wkm} - H_{k-2} - \alpha\dwn \Expec{\C(\wkm - H_{k-2}}{\wkm}) \|^2 \big | \wkm \right] \,.
\end{align*}

After simplification and using \cref{app:basic_ineq:nrm_sum}:
\begin{align*}
    \text{Var} &= \E \left[ \| - \gamma \bigpar{ \gwkHAT + \Expec{\nabla F(\wkmhat)}{\wkm}} + \alpha\dwn \bigpar{\C (\wkm - H_{k-2} )} \right.\\
    &\qquad \left.- (\wkm - H_{k-2}) \|^2 \big | \wkm \right]  \\
    &\leq 2 \gamma^2 \Expec{ \SqrdNrm{\gwkHAT - \Expec{\nabla F(\wkmhat)}{\wkm}}}{\wkm} \\
    &\qquad + 2\alpha\dwn^2 \Expec{\SqrdNrm{\C (\wkm - H_{k-2} ) - (\wkm - H_{k-2})}} {\wkm} \\
    &\leq 2 \gamma^2 \underbrace{\Expec{ \SqrdNrm{\gwkHAT - \Expec{\nabla F(\wkmhat)}{\wkm}}}{\wkm}}_{\text{Var}_1} + 2\alpha\dwn^2 \underbrace{\omgC\dwn\SqrdNrm{w_{k-1} - H_{k-2}}}_{\text{Var}_2} \\
    &\leq 2 \gamma^2 \text{Var}_1 + 2 \alpha\dwn^2 \text{Var}_2 \,.
\end{align*}

An interpretation of the above decomposition is that:
\begin{itemize}[topsep=0pt,itemsep=1pt,leftmargin=*,noitemsep]  
    \item $\text{Var}_1$ is the part of the downlink compression caused by the increment $\gwkHAT$, it is similar to \ghost.
    \item $\text{Var}_2$ is the impact of the propagation of the previous noise.
\end{itemize}

We compute the first term by introducing $\nabla F(\wkmhat)$, the second being kept as it is:
\begin{align*}
    \text{Var}_1 &= \Expec{ \SqrdNrm{\gwkHAT - \Expec{\nabla F(\wkmhat)}{\wkm}}}{\wkm} \\
    &= \Expec{ \SqrdNrm{\gwkHAT - \nabla F(\wkmhat) + \nabla F(\wkmhat) - \Expec{\nabla F(\wkmhat)}{\wkm} }}{\wkm} \\
    &= \underbrace{\Expec{ \SqrdNrm{\gwkHAT - \nabla F(\wkmhat)}}{\wkm} }_{\text{Var}_{11}}+ \underbrace{\Expec{\SqrdNrm{\nabla F(\wkmhat) - \Expec{\nabla F(\wkmhat)}{\wkm}}}{\wkm}}_{\text{Var}_{12}} \\
    &= \text{Var}_{11} + \text{Var}_{12} \,,
\end{align*}
the inner product is null given that $\Expec{\nabla F(\wkmhat) - \Expec{\nabla F(\wkmhat)}{\wkm}}{\wkm} = 0$.

Moreover:
\begin{align*}
    \text{Var}_{11} &= \Expec{ \SqrdNrm{\gwkHAT - \nabla F(\wkmhat)}}{\wkm} = \Expec{\Expec{ \SqrdNrm{\gwkHAT - \nabla F(\wkmhat)}}{\wkmhat}}{\wkm}\,,
\end{align*}
so, we can use \Cref{app:lem:grad_sto_to_grad}: $\text{Var}_{11} = \Expec{ \ffrac{\sigma^2}{Nb} (1 + \omgC\up) + \ffrac{\omgC\up}{N} \SqrdNrm{\nabla F(\wkmhat)} }{\wkm}\,.$

And now we use smoothness for the second term:
\begin{align*}
    \text{Var}_{12} &= \Expec{\SqrdNrm{\nabla F(\wkmhat) - \Expec{\nabla F(\wkmhat)}{\wkm}}}{\wkm} \\
    &\le  \Expec{\SqrdNrm{\nabla F(\wkmhat) - \nabla F(\wkm)}}{\wkm} \text{\quad by \cref{app:lem:trickVariance},}\\
    &\leq L^2 \Expec{\SqrdNrm{\wkmhat - \wkm}}{\wkm} \text{\quad using smoothness,}\\
    &\leq L^2 \omgC\dwn \Upsilon_{k-1} \text{\quad with \Cref{asu:expec_quantization_operator}}\,. 
\end{align*}

At the end:
\begin{align}
\label{app:eq:var_mcm}
\begin{split}
    \text{Var} &\leq 2\gamma^2 \bigpar{ \ffrac{\sigma^2 (1 + \omgC\up)}{Nb} + \frac{\omgC\up}{N} \Expec{\SqrdNrm{\nabla F(\wkmhat)}}{\wkm} + L^2 \omgC\dwn \Upsilon_{k-1} } \\
    &\qquad+ 2 \alpha\dwn^2 \omgC\dwn \Upsilon_{k-1} \,.
\end{split}
\end{align}

Now we focus on the squared bias $\text{Bias}^2$, with \Cref{app:lem:w_k_minus_H_k}:
\begin{align*}
    \text{Bias}^2 &= \SqrdNrm{\Expec{w_k - H_{k-1}}{w_{k-1}}} \\
    &= \SqrdNrm{(1-\alpha\dwn) (w_{k-1} - H_{k-2}) - \gamma \Expec{\nabla F(\wkmhat)}{\wkm}} \,, \text{\quad and with \Cref{app:basic_ineq:trick_bound_norm}} \,, \\
    &\leq (1-\alpha\dwn)^2 \bigpar{1 + \alpha\dwn} \Upsilon_{k-1} + \gamma^2 (1 + \ffrac{1}{\alpha\dwn}) \SqrdNrm{\Expec{\nabla F(\wkmhat)}{\wkm}} \,.
\end{align*}

And because $(1 - \alpha\dwn) (1 + \alpha\dwn) < 1$, we finally get that:
\begin{align}
\label{app:eq:bias_mcm}
    \text{Bias}^2 &\leq (1-\alpha\dwn) \Upsilon_{k-1} + \gamma^2 \left(1 + \ffrac{1}{\alpha\dwn}\right) \SqrdNrm{\Expec{\nabla F(\wkmhat)}{\wkm}} \,.
\end{align}

Combining all \cref{app:eq:var_mcm,app:eq:bias_mcm} into \cref{app:eq:bias_var_decomposition}:
\begin{align*}
    \Expec{\dwnMemTerm_{k}}{\wkm} &\leq (1 - \alpha\dwn) \Upsilon_{k-1} + \gamma^2 \bigpar{1 + \frac{1}{\alpha\dwn}} \SqrdNrm{\Expec{\nabla F(\wkmhat)}{\wkm}} \\
    &\qquad + 2\gamma^2 \bigpar{ \ffrac{\sigma^2 (1 + \omgC\up)}{Nb} + \frac{\omgC\up}{N} \Expec{\SqrdNrm{\nabla F(\wkmhat)}}{\wkm}} \\
    &\qquad + 2\gamma^2 \bigpar{L^2 \omgC\dwn \Upsilon_{k-1} } \\
    &\qquad + 2 \alpha\dwn^2 \omgC\dwn \Upsilon_{k-1} \\
    &\leq \bigpar{1 - \alpha\dwn + 2\gamma^2 L^2 \omgC\dwn + 2 \alpha\dwn^2 \omgC\dwn} \SqrdNrm{w_{k-1} - H_{k-2}} \\
    &\qquad + \gamma^2 \bigpar{1 + \frac{1}{\alpha\dwn}} \SqrdNrm{\Expec{\nabla F(\wkmhat)}{\wkm}} \\
    &\qquad + \frac{2 \gamma^2 \omgC\up}{N} \Expec{\SqrdNrm{\nabla F(\wkmhat)}}{\wkm} + \ffrac{2\gamma^2 \sigma^2(1+\omgC\up)}{Nb} \,.
\end{align*}

Next, we require:
\[
\left\{
    \begin{array}{ll}
        2\alpha\dwn^2 \omgC\dwn \leq \frac{1}{4} \alpha\dwn \Longleftrightarrow \alpha\dwn \leq \ffrac{1}{8 \omgC\dwn} \,, \\
    	2\gamma^2 L^2 \omgC\dwn \leq \frac{1}{4} \alpha\dwn = \ffrac{1}{32 \omgC\dwn} \,, \text{~by taking $\alpha\dwn=\ffrac{1}{8 \omgC\dwn}$ \quad} \Longleftrightarrow  \gamma \leq \ffrac{1}{8 \omgC\dwn L} \,, \\
    	1 + \frac{1}{\alpha\dwn} \leq \frac{2}{\alpha\dwn} \text{~which is not restrictive if $\omgC\dwn \geq 1$.}
    \end{array}
\right.
\]

Thus, it leads to:
\begin{align*}
    \Expec{\dwnMemTerm_{k}}{\wkm} &\leq \bigpar{1 - \ffrac{\alpha\dwn}{2}} \Upsilon_{k-1} + \frac{2 \gamma^2}{\alpha\dwn} \SqrdNrm{\Expec{\nabla F(\wkmhat)}{\wkm}} \\
    &\qquad + \frac{2 \gamma^2 \omgC\up}{N} \Expec{\SqrdNrm{\nabla F(\wkmhat)}}{\wkm} + \ffrac{2\gamma^2 \sigma^2(1+\omgC\up)}{Nb} \,.
\end{align*}

Next, we bound $\SqrdNrm{\Expec{\nabla F(\wkmhat)}{w_{k-1}}}$ with $\Expec{\SqrdNrm{\nabla F(\wkmhat)}}{w_{k-1}}$, and we obtain:
\begin{align*}
    \Expec{\dwnMemTerm_{k}}{\wkm} &\leq \bigpar{1 - \ffrac{\alpha\dwn}{2}} \Upsilon_{k-1} \\
    &\qquad+ 2 \gamma^2\bigpar{\frac{1}{\alpha\dwn} + \ffrac{\omgC\up}{N}} \Expec{\SqrdNrm{\nabla F(\wkmhat)}}{\wkm} \\
    &\qquad + \ffrac{2\gamma^2 \sigma^2(1+\omgC\up)}{Nb} \,.
\end{align*}

Taking the unconditional expectation gives the result.

\end{proof}

\subsection{Convex case (Theorem 2)}

In this section, we give the demonstration of \MCM~in the convex case (\Cref{thm:cvgce_mcm_convex}).

\fbox{
\begin{minipage}{\textwidth}
\begin{theorem}[Convergence of \MCM~in the homogeneous and convex case]
\label{app:thm:cvgce_mcm_convex}
Under \Cref{asu:cvx_or_strongcvx,asu:expec_quantization_operator,asu:smooth,asu:noise_sto_grad} with $\mu=0$, for a learning rate $\alpha\dwn~\leq~\ffrac{1}{8 \omgC\dwn}$, for all $k>0$, for any $\gamma\leq \gamma_{\max}$,
defining $V_k := \E[ \SqrdNrm{w_k - w_*}] + 32\gamma L \omgC\dwn^2 \E [\Upsilon_k]$, for $\bar{w}_k = \frac{1}{k}\sum_{i=0}^{k-1} w_i$, we have:
\begin{align*}
   \hspace{-1em}\gamma \FullExpec{F(w_{k-1}) - F(w_*)} &\leq V_{k-1} - V_k   + \ffrac{\gamma^2\sigma^2 \Phi(\gamma)}{Nb} \Longrightarrow    \E[    F(\bar{w}_k) - F_*] \leq \frac{V_{0}}{\gamma k}+ \ffrac{\gamma\sigma^2 \Phi(\gamma)}{Nb}\,.
\end{align*}
Consequently, for $K$ in $\N$ large enough, a step-size $\gamma= \sqrt{\frac{ \SqrdNrm{w_{0} - w_*}Nb  }{(1 + \omgC\up)  \sigma^2 K }}$ and a learning rate~$\alpha\dwn~=~\frac{1}{8\omgC\dwn}$, we have:
\begin{align*} 
\E[    F(\bar{w}_K) - F_*] \leq 2 \sqrt{\frac{ \SqrdNrm{w_{0} - w_*} (1 + \omgC\up) \sigma^2 }{Nb K }}  + O(K^{-1}).
\end{align*}
Moreover if $\sigma^2 =0 $ (noiseless case), we recover a faster convergence: $\E[    F(\bar{w}_K) - F_*] =O(K^{-1})$. 
\end{theorem}
\end{minipage}
}

\begin{proof}
Let $k$ in $\N^*$, the proof follows the one for \ghost, and we start from \cref{app:eq:equation_common_to_all}:
\begin{align*}
    \FullExpec{\SqrdNrm{w_k - w_*}} &\leq \FullExpec{\SqrdNrm{\wkm - w_*}} - \gamma \bigpar{\FullExpec{F(\wkm)} - F(w_*)} - \frac{\gamma}{2L} \FullExpec{\SqrdNrm{\nabla F(\wkmhat)}} \\
    &\qquad + 2\gamma L \FullExpec{\SqrdNrm{\wkmhat - \wkm}} + \ffrac{\gamma^2\sigma^2(1 + \omgC\up)}{Nb} \,,
\end{align*}

with \Cref{asu:expec_quantization_operator}, it easily becomes:
\begin{align*}
    \FullExpec{\SqrdNrm{w_k - w_*}} &\leq \FullExpec{\SqrdNrm{\wkm - w_*}} - \gamma \bigpar{\FullExpec{F(\wkm)} - F(w_*)} - \frac{\gamma}{2L} \FullExpec{\SqrdNrm{\nabla F(\wkmhat)}} \\
    &\qquad + 2\gamma L \omgC\dwn \FullExpec{\Upsilon_{k-1}} + \ffrac{\gamma^2\sigma^2(1 + \omgC\up)}{Nb}\,.
\end{align*}

\Cref{thm:contraction_mcm} which is specific to \MCM~gives:
\begin{align*}
    \FullExpec{\dwnMemTerm_{k}} &\leq \bigpar{1 - \ffrac{\alpha\dwn}{2}} \FullExpec{\Upsilon_{k-1}} + 2 \gamma^2\bigpar{\frac{1}{\alpha\dwn} + \ffrac{\omgC\up}{N}} \FullExpec{\SqrdNrm{\nabla F(\wkmhat)}} + \ffrac{2\gamma^2 \sigma^2(1+\omgC\up)}{Nb} \,.
\end{align*}

Defining:
$V_k := \FullExpec{\SqrdNrm{w_k - w_*}} + \gamma L \cst \FullExpec{\dwnMemTerm_{k}}$ with $\cst = \ffrac{4 \omgC\dwn}{\alpha\dwn}$,
and, combining the two last equations:
\begin{align*}
    \FullExpec{\SqrdNrm{w_k - w_*}} + \gamma L \cst \FullExpec{\dwnMemTerm_{k}} &\leq \FullExpec{\SqrdNrm{\wkm - w_*}} - \gamma \FullExpec{F(\wkm) - F(w_*)} \\
    &\quad + 2\gamma L \omgC\dwn \FullExpec{\Upsilon_{k-1}} \\
    &\quad- \frac{\gamma}{2L} \FullExpec{\SqrdNrm{\nabla F(\wkmhat)}}  + \ffrac{\gamma^2\sigma^2(1 + \omgC\up)}{Nb} \\
    &\quad + \bigpar{1 - \ffrac{\alpha\dwn}{2}} \gamma L \cst \FullExpec{\Upsilon_{k-1}} \\
    &\quad + 2 \gamma^3 L \cst \bigpar{\frac{1}{\alpha\dwn} + \ffrac{\omgC\up}{N}} \FullExpec{\SqrdNrm{\nabla F(\wkmhat)}} \\
    &\quad+\ffrac{2\gamma^3 L \sigma^2(1+\omgC\up) \cst}{N} \,,
\end{align*}

and reordering the terms gives:
\begin{align*}
    V_k &\leq \FullExpec{\SqrdNrm{\wkm - w_*}} + \bigpar{2\gamma L \omgC\dwn + \bigpar{1 - \ffrac{\alpha\dwn}{2}} \gamma L \cst } \FullExpec{\SqrdNrm{\wkm - H_{k-1}}} \\
    &\qquad+ \bigpar{2\gamma^3 L \cst \bigpar{\frac{1}{\alpha\dwn} + \ffrac{\omgC\up}{N}} - \frac{\gamma}{2L} } \FullExpec{\SqrdNrm{\nabla F(\wkmhat)}}  \\
    &\qquad - \gamma \FullExpec{F(\wkm) - F(w_*)} \\
    &\qquad + (2 \gamma L \cst + 1 ) \ffrac{\gamma^2\sigma^2(1 + \omgC\up)}{Nb} \,.
\end{align*}

We observe that:
\begin{align*}
    2\gamma L \omgC\dwn + \bigpar{1 - \ffrac{\alpha\dwn}{2} } \gamma L C \leq \gamma L C \Longleftrightarrow C \geq \ffrac{4 \omgC\dwn}{\alpha\dwn} \quad \text{which is true by definition of $\cst$.}
\end{align*}

Secondly, to get the contraction requires
\begin{align*}
    2\gamma^3 L \cst \bigpar{\frac{1}{\alpha\dwn} + \ffrac{\omgC\up}{N}} - \frac{\gamma}{2L} \leq 0 &\Longleftrightarrow \gamma^2 L \leq \ffrac{1}{4 L \cst \bigpar{\frac{1}{\alpha\dwn} + \ffrac{\omgC\up}{N}}} \\ &\Longleftrightarrow \gamma L\leq \ffrac{1}{4 \sqrt{\ffrac{\omgC\dwn}{\alpha\dwn} \bigpar{\frac{1}{\alpha\dwn} + \ffrac{\omgC\up}{N}}}} \,,
\end{align*}

because $\cst = 4 \omgC\dwn / \alpha$.
Thus, we have that:
\begin{align*}
    V_k &\leq V_{k-1} - \gamma \FullExpec{F(\wkm) - F(w_*)} + \ffrac{\gamma^2\sigma^2 \Phi(\gamma)}{Nb} \text{\quad denoting $\Phi(\gamma) := (1 + \omgC\up) \bigpar{1 + \ffrac{8 \gamma L \omgC\dwn}{\alpha\dwn}}$} \,,
\end{align*}

and then for $k = K \in \N^*$, by recurrence:
\begin{align*}
    V_K &\leq V_{0} - \gamma \sum_{k=1}^K \FullExpec{F(\wkm) - F(w_*)} + \ffrac{\gamma^2\sigma^2 \Phi(\gamma)}{Nb} \,,
\end{align*}

which implies:
\begin{align*}
    \ffrac{1}{K} \sum_{k=1}^K \FullExpec{F(\wkm) - F(w_*)}  &\leq \ffrac{V_{0} - V_K}{\gamma K} + \ffrac{\gamma^2\sigma^2 \Phi(\gamma)}{Nb} \,,
\end{align*}

Finally, by Jensen, for any $K$ in $\N^*$ such that $\gamma L \leq \min \bigg\{ \ffrac{1}{8 \omgC\dwn},  \ffrac{1}{2 \bigpar{ 1 + \ffrac{\omgC\up}{N}}}, \ffrac{1}{4\sqrt{ \ffrac{\omgC\dwn}{\alpha\dwn} \bigpar{ \ffrac{1}{\alpha\dwn} + \ffrac{\omgC\up}{N}}}} \bigg\}$, we have:
\begin{align*}
    \FullExpec{F(\bar{w}_K) - F(w_*)} \leq \ffrac{V_0}{\gamma K} + \ffrac{\gamma\sigma^2 \Phi(\gamma) }{ N b}  \,,
\end{align*}

which concludes the proof.

\end{proof}

\subsection{Strongly-convex case (Theorem 1)}
\label{app:subsec:stronglyMCM}

In this section, we give the demonstration for \MCM~in the strongly-convex case (\Cref{thm:cvgce_mcm_strongly_convex}).

\fbox{
\begin{minipage}{\textwidth}
\begin{theorem}[Convergence of \MCM~in the homogeneous and strongly-convex case]
\label{app:thm:cvgce_mcm_strongly_convex}
Under \Cref{asu:cvx_or_strongcvx,asu:expec_quantization_operator,asu:smooth,asu:noise_sto_grad} with $\mu>0$, for $k$ in $\N$, for a learning rate $\alpha\dwn~\leq~\ffrac{1}{8 \omgC\dwn}$, 
for any sequence $(\gamma_k)_{k\geq 0}\le \gamma_{\max}$,
defining $V_k := \E[ \SqrdNrm{w_k - w_*}] + 32\gamma L \omgC\dwn^2 \E [\Upsilon_k]$, we have:
\begin{align*}
    V_k &\leq (1 - \gamma_{k} \mu) V_{k-1}  - \gamma_{k} \FullExpec{F(\wkmhat) - F(w_*)} +  \ffrac{\gamma_{k}^2 \sigma^2 \Phi(\gamma_k)}{Nb}\,,
\end{align*}
Consequently, 
\begin{enumerate}
    \item if $\sigma^2 =0 $ (noiseless case), for $\gamma_k\equiv \gamma_{\max}$ we recover a linear convergence rate: $\E [\SqrdNrm{{w}_K) - w_*}]\le (1-\gamma_{\max} \mu)^k V_0$; 
    \item if $\sigma^2 >0 $, defining $\widetilde{L}$ such that $\gamma_{\max} = (2 \widetilde{L})^{-1}$, taking for all $k$ in $\N$, $\gamma_k =2/(\mu (k+1) + \widetilde{L})$, for the weighted Polyak-Ruppert average $\bar{w}_K = \sum_{k=1}^K \lambda_k \wkm / \sum_{k=1}^K \lambda_k$, with~$\lambda_k~: =~\ffrac{1}{\gamma_{k-1}^{-1}}$, we have:\gs\gs
    \[
    \hspace{-0.3cm}\FullExpec{F(\bar{w}_K) - F(w_*)} \leq \ffrac{ \mu + 2 \widetilde L}{4\mu K^2} \SqrdNrm{w_0 - w_*} + \ffrac{4\sigma^2(1 + \omgC\up)}{\mu K Nb} \bigpar{1 + \ffrac{64 L \omgC\dwn^2}{ \mu K}\ln(\mu K + \widetilde{L}) }. \]
\end{enumerate}  

\end{theorem}
\end{minipage}
}

\begin{proof}
Let $k$ in $\N^*$, the proof starts like the one for \ghost, and we start from \cref{app:eq:ghost_before_convexity} but we consider a variable step size $\gamma_{k} = 2/(\mu (k+1) + \widetilde{L})$ that depends of the iteration $k$ in $\N$.
\begin{align*}
    \FullExpec{\SqrdNrm{w_k - w_*}} &\leq \FullExpec{\SqrdNrm{\wkm - w_*}} - 2\gamma_{k} \FullExpec{ \PdtScl{\nabla F (\wkmhat)}{\wkmhat - w_*}} \\
    &\qquad + 2\gamma_{k} L \FullExpec{\SqrdNrm{\wkmhat - \wkm}} + \gamma_{k}^2 \bigpar{1 + \ffrac{\omgC\up}{N}} \FullExpec{\SqrdNrm{\nabla F (\wkmhat)}} \\
    &\qquad +\ffrac{\gamma_{k}^2 \sigma^2(1 + \omgC\up)}{Nb}  \,. 
\end{align*}

Now we apply strong-convexity (\cref{app:eq:strgly_convex} of \Cref{app:ineq_convex_demi_and_demi}):
\begin{align*}
    \FullExpec{\SqrdNrm{w_k - w_*}} &\leq \FullExpec{\SqrdNrm{\wkm - w_*}} + 2\gamma_{k} L \FullExpec{\SqrdNrm{\wkmhat - \wkm}}  \\
    &\qquad - \gamma_{k} \FullExpec{F(\wkmhat) - F(w_*)} - \gamma_{k} \left( \mu \SqrdNrm{\wkmhat - w_*} + \ffrac{1}{L} \SqrdNrm{\nabla F(\wkmhat)} \right) \\
    &\qquad + \gamma_{k}^2 \bigpar{1 + \ffrac{\omgC\up}{N}} \FullExpec{\SqrdNrm{\nabla F (\wkmhat)}} + \ffrac{\gamma_{k}^2 \sigma^2(1 + \omgC\up)}{Nb}  \,.
\end{align*}

As $\gamma_{k} \leq \ffrac{2}{\widetilde{L}} \leq \ffrac{1}{2L \bigpar{1 + \ffrac{\omgC\up}{N}}}$, and thus $\bigpar{1 - \gamma_{k} L \bigpar{1 + \ffrac{\omgC\up}{N}} } \geq 1/2$; this allows to simplify the coefficient of $\FullExpec{\SqrdNrm{\nabla F(\wkmhat)}}$:
\begin{align*}
    \FullExpec{\SqrdNrm{w_k - w_*}} &\leq (1 - \gamma_{k} \mu) \SqrdNrm{\wkm - w_*}  - \gamma_{k} \FullExpec{F(\wkmhat) - F(w_*)} \\
    &\qquad- \ffrac{\gamma_{k}}{2L} \FullExpec{\SqrdNrm{\nabla F (\wkmhat)}} + 2\gamma_{k} L \FullExpec{\SqrdNrm{\wkmhat - \wkm}} \\
    &\qquad+ \ffrac{\gamma_{k}^2 \sigma^2(1 + \omgC\up)}{Nb}\\
\end{align*}

equivalent to:
\begin{align}
\label{app:eq:mcm_strongly_convex}
\begin{split}
    \FullExpec{\SqrdNrm{w_k - w_*}} &\leq (1 - \gamma_{k} \mu) \SqrdNrm{\wkm - w_*}  - \gamma_{k} \FullExpec{F(\wkmhat) - F(w_*)} \\
    &\qquad- \ffrac{\gamma_{k}}{2L} \FullExpec{\SqrdNrm{\nabla F (\wkmhat)}} + 2\gamma_{k} L \omgC\dwn \FullExpec{\SqrdNrm{\wkm - H_{k-1}}} \\
    &\qquad+ \ffrac{\gamma_{k}^2 \sigma^2(1 + \omgC\up)}{Nb}  \,.
\end{split}
\end{align}

\Cref{thm:contraction_mcm} adapted to the case of decaying steps gives:
\begin{align}
\label{app:eq:contraction_mcm_stgly_cvxe}
\begin{split}
    \FullExpec{\dwnMemTerm_{k}} &\leq \bigpar{1 - \ffrac{\alpha\dwn}{2}} \FullExpec{\Upsilon_{k-1}} + 2 \gamma_{k}^2\bigpar{\frac{1}{\alpha\dwn} + \ffrac{\omgC\up}{N}} \FullExpec{\SqrdNrm{\nabla F(\wkmhat)}} \\
    &\qquad + \ffrac{2\gamma_{k}^2 \sigma^2(1+\omgC\up)}{Nb} \,.
\end{split}
\end{align}

Defining $V_k := \FullExpec{\SqrdNrm{w_k - w_*}} +  \gamma_k L\cst \FullExpec{\dwnMemTerm_{k}}$ with $\cst = 4 \omgC\dwn/\alpha$,
combining the two later equations \eqref{app:eq:strgly_convex} + $\gamma_k L \cst$ \eqref{app:eq:contraction_mcm_stgly_cvxe}:
\begin{align*}
    &\FullExpec{\SqrdNrm{w_k - w_*}} + \gamma_k L \cst \FullExpec{\dwnMemTerm_{k}} \\
    &\qquad\leq (1 - \gamma_{k} \mu) \SqrdNrm{\wkm - w_*}  - \gamma_{k} \FullExpec{F(\wkmhat) - F(w_*)} \\
    &\qqquad- \ffrac{\gamma_{k}}{2L} \FullExpec{\SqrdNrm{\nabla F (\wkmhat)}} + 2\gamma_{k} L \omgC\dwn \FullExpec{\SqrdNrm{\wkm - H_{k-1}}} + \ffrac{\gamma_{k}^2 \sigma^2(1 + \omgC\up)}{Nb} \\
    &\qqquad + \bigpar{1 - \ffrac{\alpha\dwn}{2}} \gamma_k L \cst \FullExpec{\Upsilon_{k-1}} + 2 \gamma_{k}^3 L  \cst \bigpar{\frac{1}{\alpha\dwn} + \ffrac{\omgC\up}{N}} \FullExpec{\SqrdNrm{\nabla F(\wkmhat)}} \\
    &\qqquad+\ffrac{2\gamma_{k}^3 L \sigma^2(1+\omgC\up) \cst}{Nb} \,,
\end{align*}

and reordering the terms gives:
\begin{align*}
    V_k &\leq (1 - \gamma_{k} \mu) \SqrdNrm{\wkm - w_*}  - \gamma_{k} \FullExpec{F(\wkmhat) - F(w_*)} \\
    &\qquad + \bigpar{ 1 - \ffrac{\alpha\dwn}{2} + \ffrac{2 \omgC\dwn}{\cst}} \gamma_k L \cst \FullExpec{\SqrdNrm{\wkm - H_{k-1}}}  \\
    &\qquad + \bigpar{2 \gamma_{k}^3 L \cst \bigpar{\frac{1}{\alpha\dwn} + \ffrac{\omgC\up}{N}} - \ffrac{\gamma_{k}}{2L}  }\FullExpec{\SqrdNrm{\nabla F(\wkmhat)}} \\
    &\qquad+ \bigpar{2 \gamma_k L \cst +1 } \ffrac{\gamma_{k}^2 \sigma^2(1 + \omgC\up)}{Nb}\,,
\end{align*}

To reach a $(1-\gamma \mu)$-convergence we first need
$\bigpar{1 - \ffrac{\alpha\dwn}{2} + \ffrac{2\omgC\dwn}{\cst}} \gamma_k L \cst \leq (1-\gamma_{k} \mu) \gamma_{k-1} L \cst  $ i.e $1 - \ffrac{\alpha\dwn}{2} + \ffrac{2\omgC\dwn}{\cst} \leq \ffrac{(1 - \gamma_k \mu) \gamma_{k-1}}{\gamma_k}$.

We need that for all $k \in \N$, $\ffrac{1 - \gamma_k \mu}{\gamma_k} \leq \ffrac{1}{\gamma_{k-1}}$ i.e., $1 - \gamma_k \mu \leq \ffrac{\gamma_k}{\gamma_{k-1}}$, but:
\begin{align*}
    \ffrac{\gamma_k}{\gamma_{k-1}} = \ffrac{\mu k - \mu + \widetilde{L}}{\mu k +\widetilde{L}} = 1 - \ffrac{\mu}{\mu k + \widetilde{L}} \text{\quad and\quad} 1 - \gamma_k \mu = 1- \ffrac{2\mu}{\mu k + \widetilde{L}} \,,
\end{align*}
and so, the inequality is always true.

Thus we must have $2 \omgC\dwn / \cst \leq \alpha\dwn / 2$ which is true by definition of $\cst$.

Secondly, it requires:
\begin{align*}
    2 \gamma_{k}^2 \cst \bigpar{\frac{1}{\alpha\dwn} + \ffrac{\omgC\up}{N}} - \ffrac{\gamma_{k}}{2L} \leq 0 &\Longleftrightarrow \gamma_{k} L \leq \ffrac{1}{4 \cst \bigpar{\ffrac{1}{\alpha\dwn} + \ffrac{\omgC\up}{N}}} \\
    &\Longleftrightarrow \gamma_{k} L \leq \ffrac{1}{4 \sqrt{ \ffrac{\omgC\dwn}{\alpha\dwn} \bigpar{\ffrac{1}{\alpha\dwn} + \ffrac{\omgC\up}{N}}}} \,,
\end{align*}
by definition of $\cst$. And it follows that the first part of the theorem is proved:
\begin{align*}
    V_k &\leq (1 - \gamma_{k} \mu) V_{k-1}  - \gamma_{k} \FullExpec{F(\wkmhat) - F(w_*)} + \ffrac{\gamma_{k}^2 \sigma^2 \Phi(\gamma_k)}{Nb}\,,
\end{align*}

where $\Phi(\gamma_k) := (1 + \omgC\up) \bigpar{1+\ffrac{8 \gamma_{k} L \omgC\dwn}{\alpha\dwn}}$. 

We now prove the second part, which requires to carefully handle the term of noise. By definition $\gamma_k = \ffrac{2}{\mu(k+1) + L}$, we denote $\lambda_k = \ffrac{1}{\gamma_{k-1}}$ and we sum the above equation weighted with the sequence of $(\lambda_k)_{k=1}^K$:
\begin{align*}
    \frac{1}{\sum_{k=1}^K \lambda_k} \sum_{k=1}^K \lambda_k \FullExpec{F(\wkmhat) - F(w_*)} &\leq \frac{1}{\sum_{k=1}^K \lambda_k} \sum_{k=1}^K \ffrac{(1 - \gamma_{k} \mu)\lambda_k}{\gamma_k} V_{k-1}  - \ffrac{\lambda_k}{\gamma_k} V_k\\
    &\qquad+ \frac{1}{\sum_{k=1}^K \lambda_k} \sum_{k=1}^K \lambda_k \ffrac{\gamma_{k} \sigma^2 \Phi(\gamma_k)}{Nb} \,.
\end{align*}

The weights are chosen to ensure that the sum of $(V_k)_{k=1}^K$ is telescopic. Because $(1 - \gamma_{k} \mu ) / \gamma_k = \gamma_{k-2}^{-1}$, we have:
\begin{align*}
    \frac{1}{\sum_{k=1}^K \lambda_k} \sum_{k=1}^K \lambda_k \FullExpec{F(\wkmhat) - F(w_*)} &\leq \frac{1}{\sum_{k=1}^K \lambda_k} \sum_{k=1}^K \ffrac{1}{\gamma_{k-2} \gamma_{k-1}} V_{k-1}  - \ffrac{1}{\gamma_k \gamma_{k-1}} V_k\\
    &\qquad+ \frac{1}{\sum_{k=1}^K \lambda_k} \sum_{k=1}^K \lambda_k \ffrac{\gamma_{k} \sigma^2 \Phi(\gamma_k)}{Nb} \,,
\end{align*}

and because for $K \in \N^*$ big enough $\frac{1}{\sum_{k=1}^K \lambda_k} = \ffrac{1}{\mu (K +1)K/4 + (\widetilde{L}K)/2} \leq \ffrac{4}{\mu K^2}$, it results that:
\begin{align}
\label{app:eq:strongly_convex_weighted_sum}
    &\frac{1}{\sum_{k=1}^K \lambda_k} \sum_{k=1}^K \lambda_k \FullExpec{F(\wkmhat) - F(w_*)} \leq \ffrac{V_0}{\gamma_0 \gamma_{-1} \mu K^2} + \frac{4}{\mu K^2} \sum_{k=1}^K \lambda_k \ffrac{\gamma_{k} \sigma^2 \Phi(\gamma_k)}{Nb}\,.
\end{align}

At the end, using the Jensen inequality - $\FullExpec{\Expec{F(\wkmhat)}{\wkm}} \leq \FullExpec{F(\wkm)}$, see \Cref{app:basic_ineq:jensen} - we have for all $K$ in~$\N$:
\begin{align*}
    &\frac{1}{\sum_{k=1}^K \lambda_k} \sum_{k=1}^K \lambda_k \FullExpec{F(\wkm) - F(w_*)} \\
    &\qqquad\leq \ffrac{V_0}{\gamma_0 \gamma_{-1} \mu K^2} + \frac{4}{\mu K^2} \sum_{k=1}^K \frac{1}{\gamma_{k-1}} \bigpar{1+\ffrac{8 \gamma_{k} L \omgC\dwn}{\alpha\dwn} } \ffrac{\gamma_{k} \sigma^2(1 + \omgC\up)}{Nb} \\
    &\qqquad\leq \ffrac{V_0}{\gamma_0 \gamma_{-1} \mu K^2} + \frac{4}{\mu K^2} \sum_{k=1}^K \bigpar{1+\ffrac{8 \gamma_{k} L \omgC\dwn}{\alpha\dwn} } \ffrac{\sigma^2(1 + \omgC\up)}{Nb}\,,
\end{align*}

because for all $k$ in $N^*$, $\gamma_k \leq \gamma_{k-1}$. We need to compute the following classical sum:
\[
    \sum_{k=1}^K \ffrac{1}{\mu k + \widetilde{L}} \le \int_{x=0}^K \ffrac{1}{\mu x + \widetilde{L}} \rm d x \le  \ffrac{1}{\mu} \ln \bigpar{\mu K+\tilde L} \,.
\]
At the end, using again the Jensen inequality, defining $\widetilde{L} = \max \left\{ 4 L \sqrt{\ffrac{ \omgC\dwn}{\alpha\dwn} \bigpar{\ffrac{1}{\alpha\dwn} + \ffrac{\omgC\up}{N}}} ,~4L \bigpar{1 + \ffrac{\omgC\up}{N}} \right\}$, taking for all $k$ in $\N$, $\gamma_k = \ffrac{2}{\mu (k+1) + \widetilde{L}}$, for all $k$ in $N^*$, $\lambda_k = \ffrac{1}{\gamma_{k-1}}$ and denoting $\bar{w}_K = \frac{\sum_{k=1}^K \lambda_k \wkm}{\sum_{k=1}^K \lambda_k} $, then for any $K$ in $\N^*$, we have:
\begin{align*}
    &\FullExpec{F(\bar{w}_K) - F(w_*)} \leq \ffrac{ \mu + 2 \widetilde L}{4\mu K^2} \SqrdNrm{w_0 - w_*} + \bigpar{1 + \ffrac{64 L \omgC\dwn^2}{ \mu K}\ln\bigpar{\mu K + \widetilde{L}} } \cdot \ffrac{4\sigma^2(1 + \omgC\up)}{\mu K Nb} \,,
\end{align*}

and the demonstration is completed.

\end{proof}

\subsection{Non-convex case (extra theorem)}
\label{app:subsec:nonconvexMCM}

In this section, we detail the convergence guarantee given for \MCM~in the non-convex case. In this scenario, the theorem will hold on the average of gradients after $K$ in $\N^*$ iterations. The structure of the proof is different from the one used for \ghost~and \MCM~in convex and strongly-convex case. Instead, the demonstration starts from the equation resulting from smoothness and use the polarization identity to handle the inner product of gradients taken at two different points. 

\fbox{
\begin{minipage}{\textwidth}
\begin{theorem}[Convergence of \MCM~in the non-convex case]
\label{app:thm:mcm_non_convex}
Under \Cref{asu:expec_quantization_operator,asu:smooth,asu:noise_sto_grad} (non-convex case), for a learning rate $\alpha\dwn~=~\ffrac{1}{8 \omgC\dwn}$, for any step size $\gamma$ s.t.
$$\gamma L \leq \min \left\{\ffrac{1}{8 \omgC\dwn}, \ffrac{1}{2 \bigpar{ 1 + \ffrac{\omgC\up}{N}}}, \ffrac{1}{8 \sqrt{\omgC\dwn^2 \bigpar{ 8\omgC\dwn + \ffrac{\omgC\up}{N}}}} \right\} \,,$$
after running $K$ in  $\N^*$ iterations, we have: 
\begin{align*}
    \ffrac{1}{K} \sum_{k=1}^K \FullExpec{\SqrdNrm{\nabla F(\wkm)}} &\leq \ffrac{2 \bigpar{F(w_0) - F(w_*)}}{\gamma K} + \ffrac{\gamma L \sigma^2 \Phi^{\mathrm{non-cvx}}(\gamma)}{Nb} \,,
\end{align*}
with $\Phi^{\mathrm{non-cvx}}(\gamma) := (1 + \omgC\up) \bigpar{1 + 32 \gamma L \omgC\dwn^2}$. 
Thus, for $K$ in $\N^*$ large enough, taking $\gamma = \sqrt{\ffrac{2 Nb \bigpar{F(w_0) - F(w_*)}}{ \sigma^2 L (1 + \omgC\up)  K} }$:
\begin{align*}
    \ffrac{1}{K} \sum_{k=1}^K \FullExpec{\SqrdNrm{\nabla F(\wkm)}} &\leq 2 \sqrt{ \ffrac{2 L \sigma^2 ( 1 + \omgC\up) \bigpar{F(w_0) - F(w_*)}}{Nb K} } + O(K^{-1})\,.
\end{align*}

\end{theorem}
\end{minipage}}

\begin{proof}
Let $k$ in $\N^*$, then smoothness (see \Cref{asu:smooth}) implies:
\begin{align*}
    &F(w_k) \leq F(\wkm) + \PdtScl{\nabla F(\wkm)}{w_k - \wkm} + \ffrac{L}{2} \SqrdNrm{w_k - \wkm} \\
    \Longleftrightarrow\quad& F(w_k) \leq F(\wkm) - \gamma \PdtScl{\nabla F(\wkm)}{\gwkHAT} + \ffrac{\gamma^2 L}{2} \SqrdNrm{\gwkHAT} \,.
\end{align*}

The inner product is not easy to handle because it implies two gradients computed at two different points: $\wkm$ and $\wkmhat$. To turn around this difficulty, we use the polarization identity, and so we have: 
\begin{align*}
    - \Expec{\PdtScl{\nabla F(\wkm)}{\gwkHAT}}{\wkm}  &= - \PdtScl{\nabla F(\wkm)}{\Expec{\nabla F(\wkmhat)}{\wkm}} \\
    &= \ffrac{1}{2} \left( - \SqrdNrm{\nabla F(\wkm)} - \Expec{\SqrdNrm{\nabla F(\wkmhat)}}{\wkm} \right. \\
    &\qquad+ \left. \Expec{\SqrdNrm{\nabla F(\wkm) - \nabla F (\wkmhat)}}{\wkm} \right)
\end{align*}

where we used the Polarization identity (\cref{app:basic_ineq:polarization}), and next with smoothness:

\begin{align*}
    - \Expec{\PdtScl{\nabla F(\wkm)}{\gwkHAT}}{\wkm} &\leq \ffrac{1}{2} \left( - \SqrdNrm{\nabla F(\wkm)} - \Expec{\SqrdNrm{\nabla F(\wkmhat)}}{\wkm} \right. \\
    &\qquad \left.+ L^2 \Expec{\SqrdNrm{\wkm - \wkmhat}}{\wkm} \right) \,,
\end{align*}

Combining with \Cref{app:lem:grad_sto_to_grad}, we obtain:
\begin{align*}
    F(w_k) &\leq F(\wkm) - \ffrac{\gamma}{2} \SqrdNrm{\nabla F(\wkm)} - \ffrac{\gamma}{2} \Expec{\SqrdNrm{\nabla F(\wkmhat)}}{\wkm} \\
    &\qquad+ \ffrac{\gamma L^2}{2} \Expec{\SqrdNrm{\wkm - \wkmhat}}{\wkm} \\
    &\qquad+ \ffrac{\gamma^2 L}{2} \bigpar{\bigpar{1 + \ffrac{\omgC\up}{N}} \SqrdNrm{\nabla F (\wkmhat)} + \ffrac{\sigma^2(1 + \omgC\up)}{Nb}} \,.
\end{align*}

Taking the full expectation and re-ordering the terms gives:
\begin{align*}
    \FullExpec{F(w_k)} &\leq \FullExpec{F(\wkm)} - \ffrac{\gamma}{2} \FullExpec{\SqrdNrm{\nabla F(\wkm)}} - \ffrac{\gamma}{2}  \left( 1 - \gamma L \bigpar{1 + \ffrac{\omgC\up}{N}} \right) \FullExpec{\SqrdNrm{\nabla F(\wkmhat)}} \\
    &\qquad + \ffrac{\gamma L^2}{2} \FullExpec{\SqrdNrm{\wkm - \wkmhat}} + \ffrac{\gamma^2 L}{2} \times \ffrac{\sigma^2(1 + \omgC\up)}{Nb} \,.
\end{align*}

Exactly like the convex case, we consider that $\gamma L (1 + \omgC\up / N) \leq 1/2$ and because $\FullExpec{\SqrdNrm{\wkm - \wkmhat}} = \FullExpec{\Expec{\SqrdNrm{\wkm - \wkmhat}}{\widehat{w}_{k-2}}} $ we can use \Cref{asu:expec_quantization_operator}:
\begin{align}
\label{app:eq:non_convex_before_contraction}
\begin{split}
    \FullExpec{F(w_k)} &\leq \FullExpec{F(\wkm)} - \ffrac{\gamma}{2} \FullExpec{\SqrdNrm{\nabla F(\wkm)}} - \ffrac{\gamma}{4} \FullExpec{\SqrdNrm{\nabla F(\wkmhat)}} \\
    &\qquad + \ffrac{\omgC\dwn\gamma L^2}{2} \FullExpec{\dwnMemTerm_k} + \ffrac{\gamma^2 L}{2} \times \ffrac{\sigma^2(1 + \omgC\up)}{Nb} \,.
\end{split}
\end{align}

Next, \Cref{thm:contraction_mcm} gives:
\begin{align*}
    \FullExpec{\dwnMemTerm_{k}} &\leq \bigpar{1 - \ffrac{\alpha\dwn}{2}} \FullExpec{\dwnMemTerm_{k-1}} + 2 \gamma^2\bigpar{\frac{1}{\alpha\dwn} + \ffrac{\omgC\up}{N}} \FullExpec{\SqrdNrm{\nabla F(\wkmhat)}} + \ffrac{2\gamma^2 \sigma^2(1+\omgC\up)}{Nb} \,.
\end{align*}

We iterate over $k$ and compute the resulting geometric sum, it gives:
\begin{align*}
    \FullExpec{\dwnMemTerm_{k}} &\leq \bigpar{1- \frac{\alpha\dwn}{2}}^k \SqrdNrm{\dwnMemTerm_{0}} + 2 \gamma^2 \bigpar{\frac{1}{\alpha\dwn} + \ffrac{\omgC\up}{N}} \sum_{t=1}^{k} \bigpar{1 - \ffrac{\alpha}{2}}^{k-t} \FullExpec{\SqrdNrm{\nabla F(\widehat{w}_{t-1})}} \\
    &\qquad + \ffrac{4\gamma^2 \sigma^2(1+\omgC\up)}{\alpha\dwn N b} \,,
\end{align*}
where we considered for the last term of the above equation that $\sum_{t=1}^k \bigpar{1 - \ffrac{\alpha\dwn}{2}}^k \leq \ffrac{2}{\alpha\dwn}$. This is equivalent to:
\begin{align*}
    \FullExpec{\dwnMemTerm_{k}} &\leq 2 \gamma^2\bigpar{\frac{1}{\alpha\dwn} + \ffrac{\omgC\up}{N}} \sum_{t=1}^{k} \bigpar{1 - \ffrac{\alpha\dwn}{2}}^{k-t} \FullExpec{\SqrdNrm{ \nabla F(\widehat{w}_{t-1})}} + \ffrac{4\gamma^2 \sigma^2(1+\omgC\up)}{\alpha\dwn N b} \,.
\end{align*}

We apply this last result to \cref{app:eq:non_convex_before_contraction}:
\begin{align*}
    \ffrac{\gamma}{2} \FullExpec{\SqrdNrm{\nabla F(\wkm)}} &\leq \FullExpec{F(\wkm) - F(w_k)} - \ffrac{\gamma}{4} \FullExpec{\SqrdNrm{\nabla F(\wkmhat)}} \\
    &\quad + \ffrac{\gamma L^2}{2} \left( \ffrac{4 \omgC\dwn \gamma^2 \sigma^2(1+\omgC\up)}{N b \alpha\dwn} \right.\\ 
    &\qqquad \left. + 2 \omgC\dwn \gamma^2 \bigpar{\ffrac{1}{\alpha\dwn} + \ffrac{\omgC\up}{N}} \sum_{t=1}^k \bigpar{1 - \ffrac{\alpha\dwn}{2}}^{k-t} \FullExpec{\SqrdNrm{\nabla F(\widehat{w}_{t-1})}} \right) \\
    &\quad + \ffrac{\gamma^2 L}{2} \times \ffrac{\sigma^2(1 + \omgC\up)}{Nb} \\
    &\leq \FullExpec{F(\wkm) - F(w_k)} - \ffrac{\gamma}{4} \FullExpec{\SqrdNrm{\nabla F(\wkmhat)}} \\
    &\quad + \gamma^3 L^2 \omgC\dwn \bigpar{\ffrac{1}{\alpha\dwn} + \ffrac{\omgC\up}{N}} \sum_{t=1}^k \bigpar{1 - \ffrac{\alpha\dwn}{2}}^{k-t} \FullExpec{\SqrdNrm{\nabla F(\widehat{w}_{t-1})}} \\
    &\quad + \ffrac{\gamma^2 \sigma^2 L (1 + \omgC\up)}{2Nb} \left( 1 + \ffrac{4 \gamma L \omgC\dwn}{\alpha\dwn} \right) \,.
\end{align*}

Summing this equation, for $k$ in range $1$ to $K$:
\begin{align*}
    \ffrac{\gamma}{2} \sum_{k=1}^K \FullExpec{\SqrdNrm{\nabla F(\wkm)}} &\leq \FullExpec{F(w_0) - F(w_k)} - \ffrac{\gamma}{4} \sum_{k=1}^K \FullExpec{\SqrdNrm{\nabla F(\wkmhat)}} \\
    &\quad + \gamma^3 L^2 \omgC\dwn\bigpar{\ffrac{1}{\alpha\dwn} + \ffrac{\omgC\up}{N}} \sum_{k=1}^K \sum_{t=1}^k \bigpar{1 - \ffrac{\alpha\dwn}{2}}^{k-t} \FullExpec{\SqrdNrm{\nabla F(\widehat{w}_{t-1})}} \\
    &\quad + \ffrac{\gamma^2 \sigma^2 L (1 + \omgC\up)}{2Nb} \left( 1 + \ffrac{4 \gamma L \omgC\dwn}{\alpha\dwn} \right) K \,.
\end{align*}

We need to invert the double-sum and we obtain:
\begin{align*}
    \ffrac{\gamma}{2} \sum_{k=1}^K \FullExpec{\SqrdNrm{\nabla F(\wkm)}} &\leq \gamma{F(w_0) - F(w_k)} - \ffrac{\gamma}{4} \sum_{i=1}^K \FullExpec{\SqrdNrm{\nabla F(\wkmhat)}} \\
    &\qquad + \ffrac{2}{\alpha\dwn} \times \gamma^3 L^2 \omgC\dwn \bigpar{\ffrac{1}{\alpha\dwn} + \ffrac{\omgC\up}{N}} \sum_{k=1}^K \FullExpec{\SqrdNrm{\nabla F(\widehat{w}_{k-1})}} \\
    &\qquad + \ffrac{\gamma^2 \sigma^2 L (1 + \omgC\up)}{2Nb} \left( 1 + \ffrac{4 \gamma L \omgC\dwn}{\alpha\dwn} \right) K \\
    &\leq \FullExpec{F(w_0) - F(w_k)} \\
    &\qquad+ \bigpar{2 \gamma^3 L^2\ffrac{ \omgC\dwn}{\alpha\dwn}\bigpar{\ffrac{1}{\alpha\dwn} + \ffrac{\omgC\up}{N}} - \ffrac{\gamma}{4}} \sum_{k=1}^K \FullExpec{\SqrdNrm{\nabla F(\widehat{w}_{k-1})}} \\
    &\qquad + \ffrac{\gamma^2 \sigma^2 L (1 + \omgC\up)}{2Nb} \left( 1 + \ffrac{4 \gamma L \omgC\dwn}{\alpha\dwn} \right) K \,.
\end{align*}

Now we consider that $2 \gamma^3 L^2\ffrac{ \omgC\dwn}{\alpha\dwn}\bigpar{\ffrac{1}{\alpha\dwn} + \ffrac{\omgC\up}{N}} \leq \gamma/4$, and because for all $k$ in $\N$, $F(w_0) - F(w_k) \leq F(w_0) - F(w_*)$:
\begin{align*}
    \ffrac{1}{K} \sum_{k=1}^K \FullExpec{\SqrdNrm{\nabla F(\wkm)}} &\leq \ffrac{2 \bigpar{F(w_0) - F(w_*)}}{\gamma K} + \ffrac{\gamma \sigma^2 L (1 + \omgC\up)}{Nb} \left( 1 + \ffrac{4 \gamma L \omgC\dwn}{\alpha\dwn} \right)  \,.
\end{align*}

Finally, for any $K$ in $\N^*$, such that  $\gamma L \leq \min \left\{\ffrac{1}{8 \omgC\dwn}, \ffrac{1}{2 \bigpar{ 1 + \ffrac{\omgC\up}{N}}}, \ffrac{1}{2 \sqrt{ 2\ffrac{ \omgC\dwn}{\alpha\dwn} \bigpar{ \ffrac{1}{\alpha\dwn} + \ffrac{\omgC\up}{N}}}} \right\}$ and~$\alpha\dwn~\leq~\ffrac{1}{8 \omgC\dwn}$, we have:
\begin{align*}
    \ffrac{1}{K} \sum_{k=1}^K \FullExpec{\SqrdNrm{\nabla F(\wkm)}} &\leq \ffrac{2 \bigpar{F(w_0) - F(w_*)}}{\gamma K} + \ffrac{\gamma L \sigma^2 \Phi^{\mathrm{non-cvx}}(\gamma)}{Nb} \,,
\end{align*}

denoting $\Phi^{\mathrm{non-cvx}}(\gamma) := (1 + \omgC\up) \bigpar{1 + \ffrac{4 \gamma L \omgC\dwn}{\alpha\dwn}}$.

Thus, for $K$ in $\N^*$ large enough, taking $\gamma = \sqrt{\ffrac{2 Nb \bigpar{F(w_0) - F(w_*)}}{ \sigma^2 L (1 + \omgC\up)  K} }$ and $\alpha\dwn = 1/(8 \omgC\dwn)$:
\begin{align*}
    \ffrac{1}{K} \sum_{k=1}^K \FullExpec{\SqrdNrm{\nabla F(\wkm)}} &\leq 2 \sqrt{ \ffrac{2 L \sigma^2 ( 1 + \omgC\up) \bigpar{F(w_0) - F(w_*)}}{Nb K} }  + O(K^{-1})\,.
\end{align*}

\end{proof}

\subsection{Proof for \RMCM~(Theorem 4)}
\label{app:subsec:proofs_rmcm}

The proof for \RMCM~is almost identical to the \MCM-scenario. It only requires to modify some notations because each device $i$ in $\llbracket 1, N \rrbracket$ holds a unique model $\wkmhati$.

For $k$ in $\N$:
\begin{enumerate}
    \item $\gwkHAT$ is now defined as $\gwkHAT = \ffrac{1}{N} \sum_\iN \widehat{\g}_k^i(\wkmhati)$,
    \item for all $i$ in $\llbracket 1, N \rrbracket$, $\gwkiHAT$ and $\nabla F(\wkmhat)$ must be replaced by $\gwkiHATRdmizd$ and $\nabla F(\wkmhati)$,
    \item instead of having a unique memory $H_k$, there is $N$ memories $(H_k^i)_\iN$ that keep track of the updates done on each worker,
    \item furthermore the notation $w_{k-1} - H_{k-2}$ is no more correct as we have $N$ different memories. Thus, it must be replaced by $\ffrac{1}{N} \sum_\iN \wkm - H_{k-2}^i$.
\end{enumerate}

\section{Proofs in the quadratic case for \MCM~and \RMCM}
\label{app:sec:proofs_quad}

In this section, for ease of notation we denote for $k$ in $\N^*$, $\gwkHAT = \ffrac{1}{N} \sum_\iN \gwkiHATRdmizd$. 

\MCM~has a unique memory $H_k$, and \RMCM~has $N$ different memories $(H_k^i)_\iN$. But for the sake of factorization, we will consider that both algorithm have $N$ memories, thus we will always consider the quantity $\ffrac{1}{N} \sum_\iN \SqrdNrm{\wkm - H_{k-2}^i}$, while we should consider the quantity $\ffrac{1}{N} \sum_\iN \SqrdNrm{\wkm - H_{k-2}}$ for \MCM. However this notation is correct considering that for \MCM, for all $i$ in $\llbracket 1, N \rrbracket$, $H_k^i = H_k$. And it follows that we have $\ffrac{1}{N} \sum_\iN \SqrdNrm{\wkm - H_{k-2}^i} = \ffrac{1}{N} \sum_\iN \SqrdNrm{\wkm - H_{k-2}^i}$.

Unlike the previous sections where the proofs for \MCM~and \RMCM~do not require any distinction, here in the quadratic case, we will on the contrary stress on the difference between the two.
The difference appears in \Cref{app:lem:rand_or_not} and comes from the way we handle the expectation of $\SqrdNrm{\ffrac{1}{N} \sum_\iN \nabla F(\wkmhati) - \nabla F(\wkm)}$  for $k$ in $\N^*$. For this purpose we define a constant $\RandOrNot$ such that $\RandOrNot = 1$ in the \MCM-case and $\RandOrNot=N$ in the \RMCM-case.

The proofs for quadratic functions relies on the fact that for any $k$ in $\N^*$, $\Expec{\nabla F(\wkmhat}{\wkm} = \nabla F(\wkm)$. 

\begin{definition}[Quadratic function]
A function $f: \R^d \mapsto \R$ is said to be quadratic if there exists a symmetric matrix $A$ in~$\mathcal{M}_{d,d}(\R)$ such that for all $x$ in $\R^d$: $f(x) - f(x_*) = \ffrac{1}{2} (x-x_*)^T A (x-x_*)$.
And then its gradient is defined for all $x$ in $\R^d$ as: $\nabla f(x) = A (x-x_*)$.

\end{definition}
\subsection{Two other lemmas}

In this section, we detail two lemmas required to prove the convergence of \MCM~and \RMCM~in the case of quadratic functions.

The first lemma allows to factorize all the results obtained for both \MCM~and \RMCM~algorithms. For $k$ in $\N^*$ and $i$ in $\llbracket 1, N \rrbracket$, the difference between the \MCM-case and the \RMCM-case results from the tigher control of $\SqrdNrm{\sum_\iN \nabla F (\wkmhati) - \nabla F(\wkm)}$. 
\begin{lemma}
\label{app:lem:rand_or_not}
 We define $\RandOrNot$ such that $\RandOrNot = 1$ in the \MCM-case and $\RandOrNot=N$ in the \RMCM-case. Then for any $k$ in $\N^*$, we have:
\begin{align*}
    \Expec{\SqrdNrm{\frac{1}{N} \sum_\iN \nabla F (\wkmhati) - \nabla F(\wkm)}}{\wkm}  \leq \frac{L^2 \omgC\dwn}{\RandOrNot} \wkmHSqrd \,.
\end{align*}
\end{lemma}
\begin{proof}
Let $k$ in $\N^*$, we apply smoothness (see \Cref{asu:smooth}), and then we upper bound the variance of the quantization operator with \Cref{asu:expec_quantization_operator}. But we must distinguish \MCM~and \RMCM~because in the first case we have $\wkmhati$ equal to $\wkmhat$ for all $i$ in $\llbracket 1, N \rrbracket$.

In the \MCM-case: 
\begin{align*}
    \Expec{\SqrdNrm{\frac{1}{N} \sum_\iN \nabla F (\wkmhati) - \nabla F(\wkm)}}{\wkm} &= \Expec{\nabla F(\wkmhat) - F(\wkm)}{\wkm} \\
    &\leq L^2 \Expec{\SqrdNrm{\wkmhat - \wkm}}{\wkm} \\
    &\leq L^2 \omgC\dwn \SqrdNrm{\Omega_{k-1}} \\
    &\leq L^2 \omgC\dwn \wkmHSqrd \,,
\end{align*}
because we consider that $\SqrdNrm{\Omega_{k-1}} = \SqrdNrm{w_{k-1} - H_{k-2}} = \wkmHSqrd$. 

In the \RMCM-case, by independence of the compressions on the downlink direction:
\begin{align*}
    \Expec{\SqrdNrm{\frac{1}{N} \sum_\iN \nabla F (\wkmhati) - \nabla F(\wkm)}}{\wkm} &=  \frac{1}{N^2} \sum_\iN\Expec{\SqrdNrm{ \nabla F (\wkmhati) - \nabla F(\wkm)}}{\wkm} \\
    &\leq \ffrac{L^2}{N^2}  \sum_\iN \SqrdNrm{\wkmhati - \wkm} \\
    &\leq \ffrac{L^2 \omgC\dwn}{N} \times \ffrac{1}{N} \sum_\iN \SqrdNrm{\wkm - H_{k-2}^i} \\
    &\leq \frac{L^2 \omgC\dwn}{N} \wkmHSqrd \,.
\end{align*}
    	
We factorize the two results and define $\RandOrNot$ such that $\RandOrNot = 1$ in the \MCM-case and $\RandOrNot=N$ in the \RMCM-case, and the result follows.
\begin{align*}
    \Expec{\SqrdNrm{\frac{1}{N} \sum_\iN \nabla F (\wkmhati) - \nabla F(\wkm)}}{\wkm}  \leq \frac{L^2 \omgC\dwn}{\RandOrNot} \wkmHSqrd \,.
\end{align*}
\end{proof}

The next lemma replaces \Cref{app:lem:grad_sto_to_grad} in the context of randomization and quadratic functions. Note that the conditioning in \Cref{app:lem:grad_sto_to_grad} is w.r.t.~to $\wkmhat$ while here we take the expectation w.r.t.~$\wkm$. This is because we remove $\wkmhat$ from the gradient and give a result which depends of $\SqrdNrm{\nabla F(\wkm)}$ instead of $\SqrdNrm{\nabla F(\wkmhat)}$. This is made possible by the fact that for all $k$ in $\N$, for quadratic functions, we have $\FullExpec{\nabla F(\wkmhat)} = \nabla F(\wkm)$. 
\begin{lemma}[Squared-norm of stochastic gradients]
\label{app:lem:quad:grad_sto_grad}
For any $k$ in $\N^*$, the squared-norm of gradients can be bounded a.s.:
\begin{align}
\label{app:eq:quad:sqrd_norm_diff_sto_and_full}
 \begin{split}
    \Expec{\SqrdNrm{\ffrac{1}{N} \sum_\iN \gwkiHATRdmizd - \nabla F(\wkmhati)}}{\wkm} &\leq \ffrac{\omgC\up}{N} \SqrdNrm{\nabla F (\wkm)} + \ffrac{\sigma^2(1 + \omgC\up)}{Nb}\\
    &\quad+\ffrac{\omgC\up \omgC\dwn L^2}{N} \wkmHSqrd \,,
\end{split}
\end{align}
\begin{align}
 \label{app:eq:quad:sqrd_norm_sto}
 \begin{split}
    \Expec{\SqrdNrm{\gwkHAT}}{\wkm} &\leq \bigpar{1 + \ffrac{\omgC\up}{N}} \SqrdNrm{\nabla F(\wkm)} + \ffrac{\sigma^2 (1 + \omgC\up)}{Nb} \\
    &\qquad+  L^2 \omgC\dwn \bigpar{\ffrac{1}{\RandOrNot} + \ffrac{\omgC\up}{N}} \wkmHSqrd \,.
\end{split}
\end{align}
\end{lemma}

The demonstration will be in two stages. We first show \cref{app:eq:quad:sqrd_norm_diff_sto_and_full}, and in a second time, we show \cref{app:eq:quad:sqrd_norm_sto}.

\begin{proof}
Let $k$ in $\N^*$.

\paragraph{First part (\cref{app:eq:quad:sqrd_norm_diff_sto_and_full}).} We can decompose the squared-norm in two terms:
\begin{align*}
    &\Expec{\SqrdNrm{\ffrac{1}{N} \sum_\iN \bigpar{\gwkiHATRdmizd - \nabla F(\wkmhati)}}}{\wkm} \\
    &\qquad= \Expec{\SqrdNrm{\ffrac{1}{N} \sum_\iN \bigpar{\gwkiHATRdmizd - \gwkihatRdmizd }}}{\wkm} \\
    &\qquad+ \Expec{\SqrdNrm{\ffrac{1}{N} \sum_\iN \bigpar{\gwkihatRdmizd - \nabla F(\wkmhati)}}}{\wkm} \,,
\end{align*}

the first term is bounded by \Cref{asu:expec_quantization_operator} and the last term by \Cref{asu:noise_sto_grad}:
\begin{align*}
    &\Expec{\SqrdNrm{\ffrac{1}{N} \sum_\iN \bigpar{\gwkiHATRdmizd - \nabla F(\wkmhati)}}}{\wkm} \\
    &\qquad\leq \ffrac{\omgC\up}{N^2} \sum_\iN \Expec{\SqrdNrm{\gwkihatRdmizd }}{\wkm} + \ffrac{\sigma^2}{Nb} \\
    &\qquad\leq \ffrac{\omgC\up}{N^2} \sum_\iN \Expec{\SqrdNrm{\gwkihatRdmizd - \nabla F (\wkmhati)}}{\wkm} \\
    &\qqquad + \ffrac{\omgC\up}{N^2} \sum_\iN \Expec{\SqrdNrm{\nabla F (\wkmhati)}}{\wkm} +\ffrac{\sigma^2}{Nb} \,.
\end{align*}    
    
And again applying \Cref{asu:noise_sto_grad} on $\Expec{\SqrdNrm{\gwkihatRdmizd - \nabla F (\wkmhati)}}{\wkm}$ for $i$ in $\{1, \cdots N\}$:
\begin{align*}
    \Expec{\SqrdNrm{\ffrac{1}{N} \sum_\iN \bigpar{\gwkiHATRdmizd - \nabla F(\wkmhati)}}}{\wkm} &= \ffrac{\omgC\up}{N^2} \sum_\iN \Expec{\SqrdNrm{\nabla F (\wkmhati)}}{\wkm} \\
    &\qquad+\ffrac{\sigma^2(1 + \omgC\up)}{Nb} \,.
\end{align*}

Now, we have:
\begin{align*}
    \ffrac{\omgC\up}{N^2} \sum_\iN \Expec{\SqrdNrm{\nabla F (\wkmhati)}}{\wkm} &= \ffrac{\omgC\up}{N^2} \sum_\iN \Expec{\SqrdNrm{\nabla F (\wkmhati) - \nabla F (\wkm)}}{\wkm} \\
    &\qquad + \ffrac{\omgC\up}{N^2} \sum_\iN \Expec{\SqrdNrm{\nabla F (\wkm)}}{\wkm} \,,
\end{align*}

using smoothness (\Cref{asu:smooth}) gives:
\begin{align*}
    \ffrac{\omgC\up}{N^2} \sum_\iN \Expec{\SqrdNrm{\nabla F (\wkmhati)}}{\wkm} &= \ffrac{\omgC\up \omgC\dwn L^2}{N} \wkmHSqrd + \ffrac{\omgC\up}{N} \SqrdNrm{\nabla F (\wkm)} \,,
\end{align*}

and putting everythings together allows to conclude for \cref{app:eq:quad:sqrd_norm_diff_sto_and_full}.

\paragraph{Second part (\cref{app:eq:quad:sqrd_norm_sto}).} We start by introducing $\SqrdNrm{\nabla F(\wkm)}$:
\begin{align*}
    \Expec{\SqrdNrm{\ffrac{1}{N} \sum_\iN \gwkiHATRdmizd}}{\wkm} &= \Expec{\SqrdNrm{\ffrac{1}{N} \sum_\iN \gwkiHATRdmizd - \nabla F(\wkm)}}{\wkm} \\
    &\qquad+ \SqrdNrm{\nabla F(\wkm)} \\
    &= \Expec{\SqrdNrm{\ffrac{1}{N} \sum_\iN \gwkiHATRdmizd - \nabla F(\wkmhati)}}{\wkm} \\
    &\qquad+ \Expec{\SqrdNrm{\ffrac{1}{N} \sum_\iN \nabla F(\wkmhati) - \nabla F(\wkm)}}{\wkm} \\
    &\qqquad+ \SqrdNrm{\nabla F(\wkm)} \,.
\end{align*}

The second term of the previous line is controlled by \Cref{app:lem:rand_or_not} which distinguish the \MCM~and \RMCM-cases by defining a constant $\RandOrNot$ such that $\RandOrNot = 1$ for \MCM~and $\RandOrNot=N$ for \RMCM:
\begin{align*}
    \Expec{\SqrdNrm{\frac{1}{N} \sum_\iN \nabla F (\wkmhati) - \nabla F(\wkm)}}{\wkm}  \leq \frac{L^2 \omgC\dwn}{\RandOrNot} \wkmHSqrd \,.
\end{align*}

Thus, we have:
\begin{align*}
    \Expec{\SqrdNrm{\ffrac{1}{N} \sum_\iN \gwkiHATRdmizd}}{\wkm} &= \Expec{\SqrdNrm{\ffrac{1}{N} \sum_\iN \gwkiHATRdmizd - \nabla F(\wkmhati)}}{\wkm} \\
    &\qquad+ \ffrac{\omgC\dwn L^2}{\RandOrNot} \wkmHSqrd + \SqrdNrm{\nabla F(\wkm)} \,,
\end{align*}

and \cref{app:eq:quad:sqrd_norm_diff_sto_and_full} allows to conclude.
\end{proof}

\subsection{Control of the Variance of the local model for quadratic function (both \MCM~and \RMCM)}
\label{app:subsec:proof_prop_rmcm_assumption}

The next theorem replaces the \Cref{thm:contraction_mcm} in the case of quadratic functions. The results are almost identical except that in these settings we control the variance using non-degraded points $(w_{t})_{t \in \N}$. This is necessary because,  for quadratic functions, the analysis is slightly different. Previously, we upper-bounded the inner product in the decomposition (\cref{eq:decomp}) by a ``strong contraction'' that was allowing to subtract $\SqrdNrm{\nabla F(\wkmhat)}$ and an extra residual term. Here we instead directly get a smaller contraction proportional to $\SqrdNrm{\nabla F(\wkm)}$ (but without any residual!). Indeed  for all $k$ in $\N$, we have $\FullExpec{\nabla F(\wkmhat)} = \nabla F(\wkm)$. This difference will appear in  \Cref{app:subsec:proof_quadratic}. 

As a consequence, we need to also control the variance of the local iterates that will appear when expanding the expected squared gradient $\E{\SqrdNrm{\tilde g_k}}$ by an affine function of the squared norms of the gradients \textbf{at the non perturbed points}. This is what \Cref{app:thm:contraction_rmcm_quadratic} provides.

\fbox{
\begin{minipage}{\textwidth}
\begin{theorem}
\label{app:thm:contraction_rmcm_quadratic}
Consider the \MCM~update as in \cref{eq:model_compression_eq_statement} or the \RMCM~update as described in \Cref{sec:randomization_def}. Under \Cref{asu:cvx_or_strongcvx,asu:expec_quantization_operator,asu:smooth,asu:noise_sto_grad} with $\mu=0$,  if $\gamma \leq \ffrac{1}{8 L \omgC\dwn \sqrt{(1/\RandOrNot + \omgC\up/N)}}$ and $\alpha\dwn \leq 1/(8 \omgC\dwn)$, then for all $k$ in $\N$:
\begin{align*}
    &\ffrac{1}{N} \sum_\iN \Expec{\SqrdNrm{w_k - H_{k-1}^i}}{\wkm} \\
    &\qquad\leq 2 \gamma^2 \bigpar{\frac{1}{\alpha\dwn} + \ffrac{\omgC\up}{N}} \sum_{t=1}^{k} (1 - \ffrac{\alpha\dwn}{2})^{k-t} \Expec{\SqrdNrm{ \nabla F(w_{t-1})}}{w_{t-1}} \\
    &\qqquad+ \ffrac{4\gamma^2 \sigma^2(1+\omgC\up)}{\alpha\dwn N b} \,.
\end{align*}
\end{theorem}
\end{minipage}
}

\begin{proof}
Let $k$ in $\N^*$ and $i$ in $\{1, \dots N\}$, from \Cref{app:thm:contraction_mcm} we have:
\begin{align*}
    \Expec{\SqrdNrm{w_k - H_{k-1}^i}}{\wkm} = \text{Var} + \text{Bias}^2 = 2\gamma^2 \text{Var}_{1} + 2 \alpha\dwn^2 \text{Var}_2 + \text{Bias}^2 \,,
\end{align*}

with
\[
\left\{
    \begin{array}{ll}
        \text{Var}_1 &= \Expec{ \SqrdNrm{\ffrac{1}{N} \sum_\iN \gwkiHATRdmizd + \Expec{\nabla F(\wkmhati)}{\wkm}}}{\wkm} \\
    	\text{Var}_2 &= \omgC\dwn \wkmHSqrd \\
    	\text{Bias}^2 &= \SqrdNrm{\Expec{w_k - H_{k-1}}{w_{k-1}}} \,.
    \end{array}
\right.
\]

Recall that in the case of quadratic functions, we have for all $i$ in $\llbracket 1, N \rrbracket$: $\Expec{\nabla F(\wkmhati)}{\wkm} = \nabla F(\wkm)$. And so for the first term of variance we can decompose as following:
\begin{align*}
    \text{Var}_1 &= \Expec{ \SqrdNrm{\frac{1}{N} \sum_\iN  \gwkiHATRdmizd - \Expec{\nabla F(\wkmhati)}{\wkm}}}{\wkm} \\
    &= \Expec{ \SqrdNrm{\frac{1}{N} \sum_\iN  \gwkiHATRdmizd - \nabla F(\wkm)}}{\wkm} \\
    &= \Expec{\SqrdNrm{\frac{1}{N} \sum_\iN  \gwkiHATRdmizd - \nabla F (\wkmhati)}}{\wkm} \\
    &\qquad+ \Expec{\SqrdNrm{\frac{1}{N} \sum_\iN \nabla F (\wkmhati) - \nabla F(\wkm)}}{\wkm} \,.
\end{align*}

The first part is handled by \cref{app:eq:quad:sqrd_norm_diff_sto_and_full} of \Cref{app:lem:quad:grad_sto_grad}:
\begin{align*}
    \Expec{\SqrdNrm{\frac{1}{N} \sum_\iN  \gwkiHATRdmizd - \nabla F (\wkmhati)}}{\wkm}  &= \ffrac{\omgC\up \omgC\dwn L^2}{N} \wkmHSqrd \\
    &\qquad+ \ffrac{\omgC\up}{N} \SqrdNrm{\nabla F (\wkm)} \\
    &\qquad+\ffrac{\sigma^2(1 + \omgC\up)}{Nb} \,,
\end{align*}
        
and the second part is tackled by \Cref{app:lem:rand_or_not} where is defined a constant $\RandOrNot$ such that $\RandOrNot = 1$ in the \MCM-case, and $\RandOrNot=N$ in the \RMCM-case: $
    \Expec{\SqrdNrm{\frac{1}{N} \sum_\iN \nabla F (\wkmhati) - \nabla F(\wkm)}}{\wkm}  \leq \frac{L^2 \omgC\dwn}{\RandOrNot} \wkmHSqrd$.
        
Finally, given that $\text{Var} = 2\gamma^2 \text{Var}_{1} + 2 \alpha\dwn^2 \text{Var}_2 $ we have:
\begin{align*}
    \text{Var} &\leq 2 \gamma^2 L^2 \omgC\dwn \bigpar{\ffrac{1}{\RandOrNot} + \ffrac{\omgC\up}{N}} \wkmHSqrd + 2 \alpha\dwn^2 \omgC\dwn \SqrdNrm{\wkm - H_{k-2}^i} \\
    &\qquad+ \ffrac{2\gamma^2 \omgC\up}{N} \SqrdNrm{\nabla F (\wkm)} +\ffrac{2\gamma^2 \sigma^2(1 + \omgC\up)}{Nb}\,.
\end{align*}

Now we focus on the squared bias $\text{Bias}^2$ exactly like in  \Cref{app:thm:contraction_mcm} and we obtain:
\begin{align*}
    \text{Bias}^2 \le (1 - \alpha\dwn) \SqrdNrm{\wkm - H_{k-2}^i} + \gamma^2 (1 + \frac{1}{\alpha\dwn}) \SqrdNrm{\nabla F(\wkm)} \,.
\end{align*}

At the end:
\begin{align*}
     \Expec{\SqrdNrm{w_k - H_{k-1}^i}}{\wkm} &\leq 2 \gamma^2 L^2 \omgC\dwn \bigpar{\ffrac{1}{\RandOrNot} + \ffrac{\omgC\up}{N}} \wkmHSqrd \\
     &\qquad+ \gamma^2 ( 1 + \ffrac{1}{\alpha\dwn} + \ffrac{2 \omgC\up}{N}) \SqrdNrm{\nabla F (\wkm)} \\
    &\qquad+ \bigpar{ (1 - \alpha\dwn) + 2 \alpha\dwn^2 \omgC\dwn} \SqrdNrm{\wkm - H_{k-2}^i} \\
    &\qquad+\ffrac{2\gamma^2 \sigma^2(1 + \omgC\up)}{Nb} \,.
\end{align*}

Summing this last equation over the $N$ devices gives:
\begin{align*}
    &\ffrac{1}{N} \sum_\iN \Expec{\SqrdNrm{w_k - H_{k-1}^i}}{\wkm} \\
    &\qquad\leq \bigpar{1 - \alpha\dwn + 2\alpha\dwn^2 \omgC\dwn +  \gamma^2 L^2 \omgC\dwn \bigpar{\ffrac{1}{\RandOrNot} + \ffrac{\omgC\up}{N}}} \wkmHSqrd \\
    &\qqquad+ \gamma^2 ( 1 + \ffrac{1}{\alpha\dwn} + \ffrac{2 \omgC\up}{N}) \SqrdNrm{\nabla F (\wkm)} \\
    &\qqquad +\ffrac{2\gamma^2 \sigma^2(1 + \omgC\up)}{Nb} \,.
\end{align*}

Exactly like in \Cref{app:thm:contraction_mcm}, we need and by taking $\alpha\dwn = 1/(8 \omgC\dwn)$:
\[
\left\{
    \begin{array}{ll}
        2\alpha\dwn^2 \omgC\dwn \leq \frac{1}{4} \alpha\dwn \Longleftrightarrow \alpha\dwn \leq \ffrac{1}{8 \omgC\dwn} \,, \\
    	2 \gamma^2 L^2 \omgC\dwn \bigpar{\ffrac{1}{\RandOrNot} + \ffrac{\omgC\up}{N}} \leq \frac{1}{4} \alpha\dwn = \ffrac{1}{32 \omgC\dwn} \Longleftrightarrow  \gamma \leq \ffrac{1}{8 L \omgC\dwn \sqrt{(1/\RandOrNot + \omgC\up/N)}} \,, \\
    	1 + \frac{1}{\alpha\dwn} \leq \frac{2}{\alpha\dwn} \text{~which is not restrictive.}
    \end{array}
\right.
\]

Thus, we can write:
\begin{align*}
    \ffrac{1}{N} \sum_\iN \Expec{\SqrdNrm{w_k - H_{k-1}^i}}{\wkm} &\leq \bigpar{1 - \ffrac{\alpha\dwn}{2}} \wkmHSqrd  \\
    &\quad + 2\gamma^2 ( \ffrac{1}{\alpha\dwn} + \ffrac{\omgC\up}{N}) \SqrdNrm{\nabla F (\wkm)} \\
    &\quad+\ffrac{2\gamma^2 \sigma^2(1 + \omgC\up)}{Nb} \,.
\end{align*}

Finally, we take the full expectation without any conditioning, we iterate over $k$ and compute the geometric sums:
\begin{align*}
    \ffrac{1}{N} \sum_\iN \FullExpec{\SqrdNrm{w_k - H_{k-1}^i}} &\leq (1- \frac{\alpha\dwn}{2})^k \SqrdNrm{w_0 - H_{-1}} + \ffrac{4\gamma^2 \sigma^2(1+\omgC\up)}{\alpha\dwn N b} \\
    &\qquad + 2 \gamma^2(\frac{1}{\alpha\dwn} + \ffrac{\omgC\up}{N}) \sum_{t=1}^{k} (1 - \ffrac{\alpha\dwn}{2})^{k-t} \FullExpec{\SqrdNrm{\nabla F(w_{t-1})}} \,.
\end{align*}

and the result follows.

\end{proof}

\subsection{Proof for quadratic function (Theorem 5)}
\label{app:subsec:proof_quadratic}

\fbox{
\begin{minipage}{\textwidth}
\begin{theorem}
\label{app:thm:rmcm_quad_guarantee}
Under \Cref{asu:cvx_or_strongcvx,asu:expec_quantization_operator,asu:smooth,asu:noise_sto_grad} with $\mu=0$, if the function is quadratic, for $\gamma= 1/(L\sqrt{K})$ and a given learning rate $\alpha\dwn = 1/(8 \omgC\dwn)$, after running $K$ iterations: 
\begin{equation*}
    \FullExpec{F(\bar{w}_K) - F_*} \leq \ffrac{\SqrdNrm{w_{0} - w_*}L}{\sqrt K} + \ffrac{\sigma^2 \Phi(\gamma)  }{NbL\sqrt{K}}  \,.
\end{equation*}
with $\Phi=(1 + \omgC\up)\left(1 +  32 \frac{\omgC\dwn^2}{\sqrt{K}} \times \frac{1}{\sqrt{K}} \left( \frac{1}{\RandOrNot} + \frac{\omgC\up}{N}\right) \right)$ and $ \RandOrNot = N$ for \RMCM, and 1 for \MCM.
\end{theorem}
\end{minipage}
}

The structure of the proof is different from the one used in \Cref{app:sec:proofs_mcm,app:sec:proof_for_ghost}.

\begin{proof}
Let $k$ in $\N$, by definition:
\begin{align*}
    \SqrdNrm{w_k - w_*} \leq \SqrdNrm{\wkm - w_*} - 2\gamma \PdtScl{\gwkHAT}{\wkm - w_*} + \gamma^2 \SqrdNrm{\gwkHAT} \,.
\end{align*}

Because $F$ is quadratic, we have $\Expec{\nabla F(\wkmhat)}{\wkm} = \nabla F(\wkm)$, thus taking expectation gives:
\begin{align*}
    \Expec{\SqrdNrm{w_k - w_*}}{\wkm} \leq \SqrdNrm{\wkm - w_*} - 2\gamma \PdtScl{\nabla F(\wkm)}{\wkm - w_*} + \gamma^2 \Expec{\SqrdNrm{\gwkHAT}}{\wkm} \,.
\end{align*}

We can directly apply convexity with \cref{app:eq:convex} from \Cref{app:ineq_convex_demi_and_demi}:
\begin{align*}
    \Expec{\SqrdNrm{w_k - w_*}}{\wkm} &\leq \SqrdNrm{\wkm - w_*} - \gamma \bigpar{F(\wkm) - F(w_*) + \ffrac{1}{L} \SqrdNrm{\nabla F(\wkm)}} \\
    &\quad+ \gamma^2 \Expec{\SqrdNrm{\gwkHAT}}{\wkm} \,.
\end{align*}

Now, with \cref{app:eq:quad:sqrd_norm_sto} of \Cref{app:lem:quad:grad_sto_grad}:
\begin{align*}
    \Expec{\SqrdNrm{w_k - w_*}}{\wkm} &\leq \SqrdNrm{\wkm - w_*} - \gamma(F(\wkm) - F(w_*)) - \ffrac{\gamma}{L} \SqrdNrm{\nabla F(\wkm)} \\
    &\qquad +\gamma^2 \left( \bigpar{1 + \ffrac{\omgC\up}{N}} \SqrdNrm{\nabla F(\wkm)} \right. \\
    &\qqquad\quad+ L^2 \omgC\dwn \bigpar{\ffrac{1}{\RandOrNot} + \ffrac{\omgC\up}{N}} \wkmHSqrd \\
    &\qqquad\quad+ \left. \ffrac{\sigma^2 (1 + \omgC\up)}{Nb} \right) \,,
\end{align*}

which gives:
\begin{align*}
    \Expec{\SqrdNrm{w_k - w_*}}{\wkm}&\leq \SqrdNrm{\wkm - w_*} - \gamma(F(\wkm) - F(w_*)) - \ffrac{\gamma}{L} \SqrdNrm{\nabla F(\wkm)} \\
    &\qquad +\gamma^2 \bigpar{1 + \ffrac{\omgC\up}{N}} \SqrdNrm{\nabla F(\wkm)} \\
    &\qquad+ \gamma^2 L^2 \omgC\dwn \bigpar{\ffrac{1}{\RandOrNot} + \ffrac{\omgC\up}{N}} \wkmHSqrd \\
    &\qquad+ \ffrac{\sigma^2 \gamma^2 (1 + \omgC\up)}{Nb} \,.
\end{align*}

Taking full expectation, and because for all $i$ in $\{1, \cdots, N\}$, $\FullExpec{\SqrdNrm{\wkm - H_{k-2}^i}} = \FullExpec{\Expec{\SqrdNrm{\wkm - H_{k-2}^i}}{\widehat{w}_{k-2}}}$, we can use the inequality controlling $\wkmHSqrd$ (see \Cref{app:thm:contraction_rmcm_quadratic}):
\begin{align*}
    \FullExpec{\SqrdNrm{w_k - w_*}} &\leq \FullExpec{\SqrdNrm{\wkm - w_*}} - \gamma \FullExpec{F(\wkm) - F(w_*)} \\
    &\quad- \ffrac{\gamma}{L} \bigpar{1 - \gamma L\bigpar{1 + \ffrac{\omgC\up}{N}}} \FullExpec{\SqrdNrm{\nabla F(\wkm)}} \\
    &\quad+ \gamma^2 L^2 \omgC\dwn \bigpar{\ffrac{1}{\RandOrNot} + \ffrac{\omgC\up}{N}} \times 2 \gamma^2 \bigpar{\frac{1}{\alpha\dwn} + \ffrac{\omgC\up}{N}} \sum_{t=1}^{k} (1 - \ffrac{\alpha\dwn}{2})^{k-t} \FullExpec{\SqrdNrm{ \nabla F(w_{t-1})}} \\
    &\quad+ \ffrac{\sigma^2 \gamma^2 (1 + \omgC\up)}{Nb} + \gamma^2 L^2 \omgC\dwn \bigpar{\ffrac{1}{\RandOrNot} + \ffrac{\omgC\up}{N}} \times \ffrac{4 \sigma^2 \gamma^2 (1 + \omgC\up)}{\alpha\dwn Nb} \,.
\end{align*}

Next, we consider - as in previous proofs - that $\gamma L(1+\omgC\up/N) \le 1/2$, and thus $\ffrac{\gamma}{L} \bigpar{1 - \gamma L\bigpar{1 + \ffrac{\omgC\up}{N}}} \geq \ffrac{\gamma}{2}$. Next we carry out the ``top-down recurrence'':
\begin{align*}
    &\FullExpec{\SqrdNrm{w_k - w_*}} \leq \SqrdNrm{w_{0} - w_*} - \gamma \sum_{j=1}^k \FullExpec{F(w_{k-j}) - F(w_*)} \\
    &\qqquad - \frac{\gamma}{2L} \sum_{j=1}^{k} \FullExpec{\SqrdNrm{\nabla F( w_{k-j-1})}} \\
    &\qqquad + \sum_{j=1}^k  2 \gamma^4 L^2 \omgC\dwn \bigpar{\ffrac{1}{\RandOrNot} + \ffrac{\omgC\up}{N}} \bigpar{\ffrac{1}{\alpha\dwn} + \ffrac{\omgC\up}{N}} \sum_{t=1}^{k-j} \bigpar{1 - \ffrac{\alpha\dwn}{2}}^{k-j-t  } \FullExpec{\SqrdNrm{\nabla F(w_{t-1})}} \\
    &\qqquad+ \sum_{j=1}^k \ffrac{\gamma^2\sigma^2(1 + \omgC\up)}{Nb} \bigpar{1 + \ffrac{4\gamma^2 L^2 \omgC\dwn}{\alpha\dwn} \bigpar{\ffrac{1}{\RandOrNot} + \ffrac{\omgC\up }{N}}} \,.
\end{align*}

We invert the double-sum, it leads to:
\begin{align*}
    &\FullExpec{\SqrdNrm{w_k - w_*}} \leq \SqrdNrm{w_{0} - w_*} - \gamma \sum_{j=1}^k \FullExpec{F(w_{j-1}) - F(w_*)} \\
    &\qqquad - \frac{\gamma}{2L} \FullExpec{\SqrdNrm{\nabla F(\wkm)}} \\
    &\qquad+ \ffrac{2}{\alpha\dwn} \times 2 \gamma^4 L^2 \omgC\dwn \bigpar{\ffrac{1}{\RandOrNot} + \ffrac{\omgC\up}{N}} \bigpar{\ffrac{1}{\alpha\dwn} + \ffrac{\omgC\up}{N}} \FullExpec{\SqrdNrm{\nabla F(w_{-1})}} \\
    &\qqquad + \sum_{j=1}^{k-1} \bigpar{ \ffrac{2}{\alpha\dwn} \times 2 \gamma^4 L^2 \omgC\dwn \bigpar{\ffrac{1}{\RandOrNot} + \ffrac{\omgC\up}{N}} \bigpar{\ffrac{1}{\alpha\dwn} + \ffrac{\omgC\up}{N}} - \frac{\gamma}{2L} } \FullExpec{ \SqrdNrm{\nabla F(w_{j-1})}} \\
    &\qqquad+ \ffrac{\gamma^2\sigma^2(1 + \omgC\up)}{Nb} \bigpar{1 + \ffrac{4\gamma^2 L^2 \omgC\dwn}{\alpha\dwn} \bigpar{\ffrac{1}{\RandOrNot} + \ffrac{\omgC\up }{N}}} \times k \,.
\end{align*}

Now, we consider that $\ffrac{4\omgC\dwn\gamma^4 L^2}{\alpha\dwn} \bigpar{\ffrac{1}{\RandOrNot} + \ffrac{\omgC\up}{N}} \bigpar{\ffrac{1}{\alpha\dwn} + \ffrac{\omgC\up}{N}} < \ffrac{\gamma}{2L} $, thus we have:
\begin{align*}
    \frac{\gamma}{k} \sum_{t=1}^k \FullExpec{F(w_{t-1}) - F(w_*)} & \leq \ffrac{\SqrdNrm{w_{0} - w_*}}{k} + \ffrac{\gamma^2\sigma^2(1 + \omgC\up)}{Nb} \bigpar{1 + \ffrac{4\gamma^2 L^2 \omgC\dwn}{\alpha\dwn} \bigpar{\ffrac{1}{\RandOrNot} + \ffrac{\omgC\up }{N}}} \,.
\end{align*}

Finally, by Jensen, for any $K$ in $\N^*$, taking $\gamma$ such that: 
\[
    \gamma L \leq \min \left\{\ffrac{1}{8 \omgC\dwn \sqrt{\frac{1}{\RandOrNot} + \frac{\omgC\up}{N}}}, \ffrac{1}{2 \bigpar{ 1 + \ffrac{\omgC\up}{N}}}, \ffrac{1}{\sqrt[3]{  \ffrac{8 \omgC\dwn}{\alpha\dwn} \bigpar{ \ffrac{1}{\RandOrNot} + \ffrac{\omgC\up}{N}} \bigpar{ \ffrac{1}{\alpha\dwn} + \ffrac{\omgC\up}{N}}}} \right\}
\]

and with $\alpha\dwn \leq\ffrac{1}{8 \omgC\dwn}$, we recover \Cref{thm:rmcm_quad_guarantee}:
\begin{align*}
    \FullExpec{F(\bar{w}_K) - F(w_*)} & \leq \ffrac{\SqrdNrm{w_{0} - w_*}}{\gamma K} + \ffrac{\gamma\sigma^2 \Phi^{\mathrm{Rd}}(\gamma)}{Nb}  \,,
\end{align*}

denoting $\Phi^{\mathrm{Rd}}(\gamma) = (1 + \omgC\up) \bigpar{1 + \ffrac{4 \gamma^2 L^2 \omgC\dwn}{K} \bigpar{\ffrac{1}{\RandOrNot} + \ffrac{\omgC\up }{N}}}$.

\end{proof}

\section{Adataptation to the heterogeneous scenario}
\label{app:sec:adaptation_to_heterogeneous_case}

In this section, we give the complete proof of \cref{thm:cvgce_mcm_strongly_convex,thm:cvgce_mcm_convex} in the case of heterogeneous workers. 

\fbox{
\ \ \begin{minipage}{0.95\textwidth}
We choose to not merge the proofs in the homogeneous and heterogeneous cases. This is to avoid the technicalities associated with the heterogeneity and the uplink compression (that have been extensively studied in previous works \cite{mishchenko_distributed_2019,horvath_stochastic_2019,li_acceleration_2020,philippenko_artemis_2020}) in the proof of our main results which aim at alleviating the impact of downlink compression. 
We thus propose two proofs that can be read almost independently in order to make proof-checking easier. 
We stress that the result in the homogeneous setting is not exactly a consequence of the heterogeneous case (the constants are degraded in the heterogeneous framework) but merging  the proofs is ultimately possible.
\end{minipage} \ \ }

\Cref{app:subsec:lemmas_for_up_mem} first presents some lemmas from \cite{philippenko_artemis_2020} required to handle the additional uplink memory. 
\Cref{app:lem:bounding_compressed_term}  (resp. \Cref{app:lem:noise_over_local_grad}) corresponds to Lemma S5 (resp. Lemma S7) evaluated at point $\wkmhat$; and \Cref{app:lem:recursive_inequalities_over_memory} corresponds to Lemma S13.
Secondly, \Cref{app:subsec:proofs_mcm_heterog} gives the demonstration of \MCM. We denote $\Phi^{\heterog}(\gamma) := (1+8\omgC\up) \bigpar{1 + \ffrac{ 8 \gamma L \omgC\dwn}{\alpha\dwn} }$ and $\gamma_{\max}^{\heterog}$ such that:
$$
\gamma_{\max}^{\heterog} L \leq \min \bigg\{ \gamma_{\max}, \ffrac{1}{16 \frac{\omgC\up}{N}}, \ffrac{1}{8 \sqrt{2 \ffrac{\omgC\dwn}{\alpha\dwn} \cdot \ffrac{\omgC\up}{N} }} \bigg\} \,.
$$

We make the following assumption on the heterogeneity.
\begin{assumption}[Bounded gradient at $w_*$]
\label{asu:bounded_noises_across_devices}
 There is a constant $B$ in $\mathbb{R_+}$, s.t.: 
$\frac{1}{N} \sum_{i=0}^N \| \nabla F_i(w_*)\|^2 = B^2\,.$ And we denote for all $i$ in $\llbracket 1 , N \rrbracket$, $h_*^i = \nabla F_i(w_*)$.
\gs
\end{assumption}

\subsection{Control of the uplink memory}
\label{app:subsec:lemmas_for_up_mem}

In this section we give the theorems that are required by the uplink memory.

\begin{lemma}[Bounding the compressed term]
\label{app:lem:bounding_compressed_term}
The squared norm of the compressed term sent by each node to the central server can be bounded as following:
\begin{align*}
    \forall k \in \N\,, \,\forall i \in \llbracket 1 , N \rrbracket\,, \quad \SqrdNrm{ \Delta_k^i } \leq 2 \left( \SqrdNrm{ \gwkihat - h_*^i } + \SqrdNrm{ h_k^i - h_*^i } \right)\,. \\
\end{align*}
\end{lemma}

\begin{lemma}[Noise over local gradients]
\label{app:lem:noise_over_local_grad}
Let $k \in N^*$ and $i \in \llbracket 1, N \rrbracket$. The noise in the stochastic gradients as defined in \Cref{asu:noise_sto_grad,asu:bounded_noises_across_devices} can be controlled as following:
\begin{align*}
    \frac{1}{N^2} \sum_{i=1}^N \Expec{\SqrdNrm{\gwkihat - h_*^i}}{\wkm} &\leq \ffrac{2L}{N} \Expec{\PdtScl{\nabla F(\wkmhat)}{\wkmhat - w_*}}{\wkm} + \ffrac{2 \sigma^2}{Nb}  \,. 
\end{align*}
\end{lemma}

\begin{lemma}[Recursive inequalities over memory term]
\label{app:lem:recursive_inequalities_over_memory}
Let $k \in \N$ and let $i \in \llbracket 1 , N \rrbracket$. The memory term used in the uplink broadcasting can be bounded using a recursion:
\begin{align*}
\begin{split}
    \Expec{\upMemTerm_{k}}{\wkm} &\leq (1 - \TOne) \Expec{\upMemTerm_{k-1}}{\wkm} \\
    &\qquad+ \ffrac{2\TTwo L}{N}\Expec{ \PdtScl{\nabla F (\wkmhat)}{\wkmhat - w_*}}{\wkm} \\
    &\qquad+  \frac{2}{N} \ffrac{\sigma^2}{b} \TTwo \,.
\end{split}
\end{align*}
\end{lemma}

\begin{lemma}[Squared-norm of stochastic gradients]
For any $k$ in $\N^*$, the squared-norm of gradients can be bounded a.s.:
\label{app:lem:grad_sto_to_grad_heterog}
\begin{align*}
    &\Expec{\SqrdNrm{\gwkHAT}}{\wkmhat} \leq \bigpar{1 + \frac{4\omgC\up }{N}} L \Expec{\PdtScl{\nabla F(\wkmhat)}{\wkmhat - w_*}}{\wkmhat} \\
    &\qqquad + 2\omgC\up\Expec{\upMemTerm_{k-1}}{\wkmhat} + \frac{\sigma^2}{Nb} (1 + 4 \omgC\up) \,,\\
    &\Expec{\SqrdNrm{\gwkHAT - \nabla F(\wkmhat)}}{\wkmhat} \leq \frac{4\omgC\up L }{N} \Expec{\PdtScl{\nabla F(\wkmhat)}{\wkmhat - w_*}}{\wkmhat} \\
    &\qqquad + 2\omgC\up\Expec{\upMemTerm_{k-1}}{\wkmhat} + \frac{\sigma^2}{Nb} (1 + 4 \omgC\up) \,,
\end{align*}
\Cref{app:lem:grad_sto_to_grad_heterog} extends \Cref{app:lem:grad_sto_to_grad}.
\end{lemma}

\begin{proof}
Let $k$ in $\N^*$, then: 
\begin{align*}
    \Expec{\SqrdNrm{\gwkHAT}}{\wkmhat} &= \SqrdNrm{\nabla F(\wkmhat)} + \Expec{\SqrdNrm{\gwkHAT - \nabla F(\wkmhat)}}{\wkmhat} \\
    &\leq L\Expec{\PdtScl{\nabla F(\wkmhat)}{\wkmhat - w_*}}{\wkmhat} + \Expec{\SqrdNrm{\gwkHAT - \nabla F(\wkmhat)}}{\wkmhat}
\end{align*}

Secondly:
\begin{align*}
    &\Expec{\SqrdNrm{\gwkHAT - \nabla F(\wkmhat)}}{\wkmhat} \\
    &\quad= \Expec{\SqrdNrm{\ffrac{1}{N} \sum_\iN \bigpar{\widehat{\Delta}_{k-1}^i + h_{k-1}^i - \nabla F_i(\wkmhat)}}}{\wkmhat}\\
    &\quad= \Expec{\SqrdNrm{\ffrac{1}{N} \sum_\iN \bigpar{\widehat{\Delta}_{k-1}^i + h_{k-1}^i -  \gwkihat + \gwkihat  - \nabla F_i(\wkmhat)}}}{\wkmhat} \\
    &\quad= \Expec{\SqrdNrm{\ffrac{1}{N} \sum_\iN \widehat{\Delta}_{k-1}^i - \Delta_{k-1}^i}}{\wkmhat} + \Expec{\SqrdNrm{\ffrac{1}{N} \sum_\iN \gwkihat  - \nabla F_i(\wkmhat)}}{\wkmhat}\,,
\end{align*}
the inner product being null.

Next, expanding the squared norm again, and because the two sums of inner products are null as the stochastic oracle and uplink compressions are independent:
\begin{align*}
    \Expec{\SqrdNrm{\gwkHAT - \nabla F(\wkmhat)}}{\wkmhat} &= \frac{1}{N^2} \sum_\iN \Expec{\SqrdNrm{\widehat{\Delta}_{k-1}^i - \Delta_{k-1}^i}}{\wkmhat} \\
    &\qquad+ \frac{1}{N^2} \sum_\iN \Expec{\SqrdNrm{\gwkihat  - \nabla F_i(\wkmhat)}}{\wkmhat} \,.
\end{align*}

Then, for any $i$ in $\llbracket 1, N \rrbracket$ as $\Expec{\SqrdNrm{\widehat{\Delta}_{k-1}^i - \Delta_{k-1}^i}}{\wkmhat} = \Expec{\Expec{\SqrdNrm{\widehat{\Delta}_{k-1}^i - \Delta_{k-1}^i}}{\g_k^i}}{\wkmhat}$, and using \Cref{asu:expec_quantization_operator} we have:
\begin{align*}
    \Expec{\SqrdNrm{\gwkHAT - \nabla F(\wkmhat)}}{\wkmhat} &\leq \frac{\omgC\up}{N^2} \sum_\iN \Expec{\SqrdNrm{\Delta_{k-1}^i}}{\wkmhat} \\
    &\qquad+ \frac{1}{N^2} \sum_\iN \Expec{\SqrdNrm{\gwkihat  - \nabla F_i(\wkmhat)}}{\wkmhat} \,.
\end{align*}

Furthermore with \Cref{app:lem:bounding_compressed_term} and \Cref{asu:noise_sto_grad}:
\begin{align*}
    \Expec{\SqrdNrm{\gwkHAT - \nabla F(\wkmhat)}}{\wkmhat} &\leq \frac{2\omgC\up}{N^2} \sum_\iN \Expec{\SqrdNrm{\gwkhat - h_*^i}}{\wkmhat} \\
    &\qquad + 2\omgC\up\Expec{\upMemTerm_{k-1}}{\wkmhat} + \frac{\sigma^2}{Nb} \,.
\end{align*}

And finally with \cref{app:lem:noise_over_local_grad}:
\begin{align*}
    \Expec{\SqrdNrm{\gwkHAT - \nabla F(\wkmhat)}}{\wkmhat} &\leq \frac{4\omgC\up L }{N} \Expec{\PdtScl{\nabla F(\wkmhat)}{\wkmhat - w_*}}{\wkmhat} + \frac{4\omgC\up\sigma^2}{Nb} \\
    &\qquad + 2\omgC\up\Expec{\upMemTerm_{k-1}}{\wkmhat} + \frac{\sigma^2}{Nb} \,,
\end{align*}
from which we derive the two inequalities of the lemma.

\end{proof}

\subsection{Proofs for \MCM}
\label{app:subsec:proofs_mcm_heterog}

In this section, we provide the demonstration of \Cref{thm:cvgce_mcm_strongly_convex,thm:cvgce_mcm_convex} in the convex and strongly-convex cases with heterogeneous workers.

\subsubsection{Control of the Variance of the local model for \MCM}
\label{app:subsec:proof_prop_mcm_assumption_heterog}

In this section, the aim is to control the variance of the local model for \MCM~but in the setting of heterogeneous worker, as done previously in \Cref{app:thm:contraction_mcm}.

\fbox{
\begin{minipage}{\textwidth}
\begin{theorem}
\label{app:thm:contraction_mcm_heterog}
Consider the \MCM~update as in \cref{eq:model_compression_eq_statement}. Under \Cref{asu:expec_quantization_operator,asu:smooth,asu:noise_sto_grad},  if $\gamma \leq 1/({8\omgC\dwn L})$ and $\alpha\dwn \leq 1/(8 \omgC\dwn)$, then for all $k$ in $\N$:
\begin{align*}
    \Expec{\dwnMemTerm_{k}}{\wkm} &\leq \bigpar{1 - \ffrac{\alpha\dwn}{2}} \Upsilon_{k-1} \\
    &\qquad+ 2 \gamma^2 L \bigpar{\frac{1}{\alpha\dwn} + \ffrac{4\omgC\up }{N}}  \Expec{\PdtScl{\nabla F(\wkmhat)}{\wkmhat - w_*}}{\wkmhat}\\
    &\qquad + 4\gamma^2\omgC\up \Expec{\upMemTerm_{k-1}}{\wkmhat} + \ffrac{2\gamma^2 \sigma^2(1+4\omgC\up)}{Nb} \,.
\end{align*}
\end{theorem}
\end{minipage}
}

\begin{proof}
Let $k$ in $\N$, we recall that by definition: 
\[
\left\{
    \begin{array}{ll}
        \Omega_k = w_k - H_{k-1} \\
    	\widehat{\Omega}_k = \C\dwn(\Omega_k) \\
    	\widehat{w}_k = \widehat{\Omega}_k + H_{k-1} \,.
    \end{array}
\right.
\]

We start the proof by performing a bias-variance decomposition, and exactly like in the proof of \Cref{app:thm:contraction_mcm}, we obtain:
\begin{align*}
\SqrdNrm{\Omega_k} = \text{Bias}^2 + 2\gamma^2\text{Var}_{12} + 2\gamma^2\text{Var}_{12} + 2\alpha\dwn^2\text{Var}_{2}
\end{align*}

We first have:
\begin{align*}
    \text{Var}_{11} &= \Expec{ \SqrdNrm{\gwkHAT - \nabla F(\wkmhat)}}{\wkm} = \Expec{\Expec{ \SqrdNrm{\gwkHAT - \nabla F(\wkmhat)}}{\wkmhat}}{\wkm}\,,
\end{align*}
so, we can use \Cref{app:lem:grad_sto_to_grad_heterog}: 
\begin{align*}
    \text{Var}_{11} &= \frac{4\omgC\up L }{N} \Expec{\PdtScl{\nabla F(\wkmhat)}{\wkmhat - w_*}}{\wkmhat} + 2\omgC\up \Expec{\upMemTerm_{k-1}}{\wkmhat} + \frac{\sigma^2}{Nb} (1 + 4 \omgC\up) \,.
\end{align*}

The other terms are exactly as before in \Cref{app:thm:contraction_mcm}:

\[
\left\{
    \begin{array}{ll}
         &\text{Var}_{12} \leq L^2 \omgC\dwn \Upsilon_{k-1} \\
    &\text{Var}_{2} \leq \omgC\dwn \Upsilon_{k-1} \\
    &\text{Bias}^2 \leq (1-\alpha\dwn) \Upsilon_{k-1} + \gamma^2  L(1 + \ffrac{1}{\alpha\dwn}) \Expec{\PdtScl{\nabla F(\wkmhat)}{\wkmhat - w_*}}{\wkmhat} \,. 
    \end{array}
\right.
\]

At the end:
\begin{align*}
    \Expec{\dwnMemTerm_{k}}{\wkm} &\leq (1 - \alpha\dwn) \Upsilon_{k-1} + \gamma^2  L(1 + \ffrac{1}{\alpha\dwn}) \Expec{\PdtScl{\nabla F(\wkmhat)}{\wkmhat - w_*}}{\wkmhat} \\
    &\qquad + \frac{8\omgC\up \gamma^2  L }{N} \Expec{\PdtScl{\nabla F(\wkmhat)}{\wkmhat - w_*}}{\wkmhat} \\
    &\qquad + 4\gamma^2\omgC\up \Expec{\upMemTerm_{k-1}}{\wkmhat} + \frac{2\gamma^2 \sigma^2}{Nb} (1 + 4 \omgC\up)\\
    &\qquad + 2 \gamma^2 L^2 \omgC\dwn \Upsilon_{k-1}  + 2 \alpha\dwn^2 \omgC\dwn \Upsilon_{k-1} \,,
\end{align*}
    
which is equivalent to:
\begin{align*}
    \Expec{\dwnMemTerm_{k}}{\wkm} &\leq \bigpar{1 - \alpha\dwn + 2\gamma^2 L^2 \omgC\dwn + 2 \alpha\dwn^2 \omgC\dwn} \SqrdNrm{w_{k-1} - \upMemTerm_{k-2}} \\
    &\qquad + \gamma^2 L \bigpar{1 + \frac{1}{\alpha\dwn} + \ffrac{8 \omgC\up }{N}} \Expec{\PdtScl{\nabla F(\wkmhat)}{\wkmhat - w_*}}{\wkmhat}  \\
    &\qquad + 4\gamma^2\omgC\up \Expec{\upMemTerm_{k-1}}{\wkmhat}  + \ffrac{2\gamma^2 \sigma^2(1+4\omgC\up)}{Nb} \,.
\end{align*}

Next, we require as in \Cref{app:thm:contraction_mcm}:
\[
\left\{
    \begin{array}{ll}
        2\alpha\dwn^2 \omgC\dwn \leq \frac{1}{4} \alpha\dwn \Longleftrightarrow \alpha\dwn \leq \ffrac{1}{8 \omgC\dwn} \,, \\
    	2\gamma^2 L^2 \omgC\dwn \leq \frac{1}{4} \alpha\dwn = \ffrac{1}{32 \omgC\dwn} \,, \text{~by taking $\alpha\dwn=\ffrac{1}{8 \omgC\dwn}$ \quad} \Longleftrightarrow  \gamma \leq \ffrac{1}{8 \omgC\dwn L} \,, \\
    	1 + \frac{1}{\alpha\dwn} \leq \frac{2}{\alpha\dwn} \text{~which is not restrictive if $\omgC\dwn \geq 1$,}
    \end{array}
\right.
\]

and it leads to the final result taking unconditional expectation.
\end{proof}

\subsubsection{Convex case}

\fbox{
\begin{minipage}{\textwidth}
\begin{theorem}[Convergence of \MCM~in the heterogeneous and convex case]
\label{app:thm:cvgce_mcm_convex_heterog}
Under \Cref{asu:cvx_or_strongcvx,asu:expec_quantization_operator,asu:smooth,asu:noise_sto_grad} with $\mu=0$ (convex case), for learning rates $\alpha\dwn~\leq~\ffrac{1}{8 \omgC\dwn}$ and $\alpha\up(1 + \omgC\up) \leq 1$, 
taking a step size s.t. $\gamma \leq \gamma_{\max}^{\mathrm{Heterog}}$,
for any $k$ in $\N$, defining:
$$V_k := \FullExpec{\SqrdNrm{w_k - w_*}} + \gamma^2 \cst_1 \FullExpec{\upMemTerm_{k}} +\gamma L \cst_2 \FullExpec{\dwnMemTerm_{k}} \,,$$
with $\cst_1 = 2 \omgC\up ( 1 + 8 \gamma L \omgC\dwn / \alpha\dwn) / \TOne$, $\cst_2 = 4 \gamma L \omgC\dwn / \alpha\dwn$, we have: 
\begin{align*}
\begin{split}
    V_k \leq V_{k-1} - \gamma \FullExpec{F(\wkmhat) - F(w_*)} + \frac{\gamma^2 \sigma^2 \Phi^{\heterog}(\gamma)}{Nb}
    \,.
\end{split}
\end{align*}
\end{theorem}
\end{minipage}
}

\begin{proof}
We denote for $k$ in $\N^*$ $\gwktilde = \ffrac{1}{N} \sum_\iN \widehat{\Delta}_{k-1}^i + h_{k-1}^i$ with $\Delta_{k-1}^i = \gwkihat - h_{k-1}^i$, and $\upMemTerm_k = \frac{1}{N^2} \sum_\iN \Expec{\SqrdNrm{h_{k}^i - \nabla F_i(w_*)}}{\wkmhat}$.

Let $k$ in $\N^*$, by definition:
\begin{align*}
    \SqrdNrm{w_k - w_*} \leq \SqrdNrm{\wkm - w_*} - 2\gamma \PdtScl{\gwkHAT}{\wkm - w_*} + \gamma^2 \SqrdNrm{\gwkHAT} \,.
\end{align*}
Next, we expend the inner product as following:
\begin{align*}
    \SqrdNrm{w_k - w_*} \leq \SqrdNrm{\wkm - w_*} - 2\gamma \PdtScl{\gwkHAT}{\wkmhat - w_*} - 2\gamma \PdtScl{\gwkHAT}{\wkm - \wkmhat} + \gamma^2 \SqrdNrm{\gwkHAT} \,.
\end{align*}

Taking expectation conditionally to $\wkm$, and using $\Expec{ \gwkHAT}{\wkm}= \Expec{\Expec{\gwkHAT}{\wkmhat}}{\wkm} =\Expec{\nabla F(\wkmhat)}{\wkm}$, we obtain:
\begin{align*}
    \Expec{\SqrdNrm{w_k - w_*}}{\wkm} &\leq \SqrdNrm{\wkm - w_*} - \Expec{2\gamma \PdtScl{\nabla F (\wkmhat)}{\wkmhat - w_*}}{\wkm} \\
    &\qquad- 2\gamma \Expec{\PdtScl{\nabla F (\wkmhat)}{\wkm - \wkmhat}}{\wkm} \\
    &\qquad+ \gamma^2 \Expec{\SqrdNrm{\gwkHAT}}{\wkm} \,.
\end{align*}

Then invoking \Cref{app:lem:grad_sto_to_grad} to upper bound the squared norm of the stochastic gradients, and noticing that $\Expec{\PdtScl{\nabla F(\wkm)}{\wkmhat-\wkm}}{\wkm}=0$ leads to:
\begin{align}
    \Expec{\SqrdNrm{w_k - w_*}}{\wkm} &\leq \SqrdNrm{\wkm - w_*} - 2\gamma\Expec{ \PdtScl{\nabla F (\wkmhat)}{\wkmhat - w_*}}{\wkm} \nonumber \\
    &\quad - 2\gamma\Expec{ \PdtScl{\nabla F (\wkmhat) - \nabla F(\wkm)}{\wkm - \wkmhat}}{\wkm} \label{eq:decomp_heterog} \\
    &\quad + \gamma^2 \left( \bigpar{1 + \frac{4\omgC\up }{N}} L \Expec{\PdtScl{\nabla F(\wkmhat)}{\wkmhat - w_*}}{\wkmhat} \right.  \nonumber\\
    &\qquad \left.+ 2\omgC\up \Expec{\upMemTerm_{k-1}}{\wkmhat} + \frac{\sigma^2}{Nb} (1 + 4 \omgC\up) \right) \,. \nonumber
\end{align}

Now using Cauchy-Schwarz inequality (\cref{app:basic_ineq:cauchy_schwarz}) and smoothness: 
\begin{align*}
    &- \Expec{2\gamma \PdtScl{\nabla F (\wkmhat) - \nabla F(\wkm) }{\wkm - \wkmhat}}{\wkm} \\
    &\qqquad= 2\gamma \Expec{\PdtScl{\nabla F (\wkmhat) - \nabla F(\wkm)}{\wkmhat - \wkm}}{\wkm} \\
    &\qqquad\leq 2 \gamma L \Expec{\SqrdNrm{\wkmhat - \wkm}}{\wkm} \,,
\end{align*}

and thus:
\begin{align*}
\begin{split}
    \Expec{\SqrdNrm{w_k - w_*}}{\wkm} &\leq \SqrdNrm{\wkm - w_*} - 2\gamma\Expec{ \PdtScl{\nabla F (\wkmhat)}{\wkmhat - w_*}}{\wkm} \\
    &\qquad + 2\gamma L \Expec{\SqrdNrm{\wkmhat - \wkm}}{\wkm} \\
    &\qquad + \bigpar{1 + \frac{4\omgC\up }{N}} \gamma^2 L \Expec{\PdtScl{\nabla F(\wkmhat)}{\wkmhat - w_*}}{\wkmhat}  \\
    &\qquad + 2\omgC\up \gamma^2 \Expec{\upMemTerm_{k-1}}{\wkmhat} + \frac{\sigma^2 \gamma^2 }{Nb} (1 + 4 \omgC\up) \,.
\end{split}
\end{align*}

As $\gamma \leq \ffrac{1}{2L \bigpar{1 + \ffrac{\omgC\up}{N}}}$, and thus $\bigpar{1 - \frac{\gamma L}{2} \bigpar{1 + \ffrac{4 \omgC\up}{N}} } \geq 1/2$; this allows to simplify the coefficient of the scalar product:
\begin{align}
\label{app:eq:mcm_before_lyapunov_heterog}
\begin{split}
    \Expec{\SqrdNrm{w_k - w_*}}{\wkm} &\leq \SqrdNrm{\wkm - w_*} - \gamma\Expec{ \PdtScl{\nabla F (\wkmhat)}{\wkmhat - w_*}}{\wkm} \\
    &\qquad + 2\gamma L \Expec{\SqrdNrm{\wkmhat - \wkm}}{\wkm} \\
    &\qquad + 2\omgC\up \gamma^2 \Expec{\upMemTerm_{k-1}}{\wkmhat} + \frac{\sigma^2 \gamma^2 }{Nb} (1 + 4 \omgC\up) \,.
\end{split}
\end{align}

With \Cref{app:lem:recursive_inequalities_over_memory}, we have :
\begin{align}
\label{app:eq:recursive_inequalities_over_memory_mcm_heterog}
\begin{split}
    \Expec{\upMemTerm_{k}}{\wkm} &\leq (1 - \TOne) \Expec{\upMemTerm_{k-1}}{\wkm} \\
    &\qquad+ \ffrac{2\TTwo L}{N}\Expec{ \PdtScl{\nabla F (\wkmhat)}{\wkmhat - w_*}}{\wkm} \\
    &\qquad+  \frac{2}{N} \ffrac{\sigma^2}{b} \TTwo \,.
\end{split}
\end{align}

and \Cref{app:thm:contraction_mcm_heterog} gives:
\begin{align}
\label{app:eq:constraction_mcm_heterog}
\begin{split}
    \Expec{\dwnMemTerm_{k}}{\wkm} &\leq \bigpar{1 - \ffrac{\alpha\dwn}{2}} \Upsilon_{k-1} \\
    &\qquad+ 2 \gamma^2 L \bigpar{\frac{1}{\alpha\dwn} + \ffrac{4\omgC\up }{N}}  \Expec{\PdtScl{\nabla F(\wkmhat)}{\wkmhat - w_*}}{\wkmhat}\\
    &\qquad + 4\gamma^2\omgC\up \Expec{\upMemTerm_{k-1}}{\wkmhat} + \ffrac{2\gamma^2 \sigma^2(1+4\omgC\up)}{Nb} \,.
\end{split}
\end{align}

We take the full expectation (without conditioning) and we set:
$$V_k := \FullExpec{\SqrdNrm{w_k - w_*}} + \gamma^2 \cst_1 \FullExpec{\upMemTerm_{k}} +  \gamma L \cst_2 \FullExpec{\dwnMemTerm_{k}} \,,$$
with $\cst_1 = 2 \omgC\up ( 1 + 8 \gamma L \omgC\dwn / \alpha\dwn) / \TOne$ and $\cst_2 = 4 \omgC\dwn / \alpha\dwn$. 

We combine previous equations as follows
$(\ref{app:eq:mcm_before_lyapunov_heterog}) + \gamma^2 \cst_1 (\ref{app:eq:recursive_inequalities_over_memory_mcm_heterog}) +\cst_2 (\ref{app:eq:constraction_mcm_heterog})$:
\begin{align}
\label{app:eq:combination_mcm_all_eq}
\begin{split}
    &\FullExpec{\SqrdNrm{w_k - w_*}} +\gamma^2 \cst_1 \FullExpec{\upMemTerm_{k}} +  \gamma L \cst_2 \FullExpec{\dwnMemTerm_{k}} \leq \SqrdNrm{\wkm - w_*} \\
    &\qquad- \gamma \bigpar{1  - \gamma L \bigpar{ \bigpar{ \ffrac{1}{\alpha\dwn} + \frac{4 \omgC\up}{N}}  \gamma L \cst_2 + \ffrac{\TTwo \cst_1}{N}} }\FullExpec{ \PdtScl{\nabla F (\wkmhat)}{\wkmhat - w_*}} \\
    &\qquad+ \bigpar{2\omgC\up ( 1+ 2 \gamma L \cst_2 ) + (1 - \TOne) \cst_1} \gamma^2 \FullExpec{\upMemTerm_{k-1}} \\
    &\qquad+ \bigpar{2 \gamma L \omgC\dwn + \bigpar{1 - \frac{\alpha\dwn}{2}}  \gamma L \cst_2}\FullExpec{\dwnMemTerm_{k-1}}\\
    &\qquad+ \frac{\gamma^2 \sigma^2 }{Nb} \bigpar{ (1+4\omgC\up) (1 + 2  \gamma L \cst_2) + 2 \TTwo \cst_1}
    \,,
\end{split}
\end{align}

We first observe that:
\begin{align*}
    2 \gamma L \omgC\dwn + \bigpar{1 - \frac{\alpha\dwn}{2}}  \gamma L \cst_2 \leq  \gamma L \cst_2 \Longleftrightarrow \cst_2 \geq \ffrac{4 \omgC\dwn}{\alpha\dwn}\,,\quad \text{which is true by definition of $\cst_2$.}
\end{align*}

Secondly, ensuring that the factor multiplying $\FullExpec{\upMemTerm_{k-1}}$ on the right hand side is smaller than $\gamma^2 \cst_1$ requires:
\begin{align*}
    &2\omgC\up ( 1+ 2  \gamma L \cst_2 ) + (1 - \TOne) \cst_1 \leq \cst_1 \\
    \Longrightarrow \quad& \cst_1 \geq \ffrac{2 \omgC\up ( 1 + 8 \gamma L \omgC\dwn/\alpha\dwn)}{\TOne} \quad \text{because $\cst_2 = 4 \omgC\dwn/\alpha\dwn$.}
\end{align*}

Finally, we have that $1  - \gamma L \bigpar{ \bigpar{ \ffrac{1}{\alpha\dwn} + \ffrac{4 \omgC\up}{N}}  \gamma L \cst_2 + \ffrac{\TTwo \cst_1}{N}} \geq \frac{1}{2} $, if we take $\gamma $ such that:
\[
\left\{
    \begin{array}{ll}
        \bigpar{\ffrac{1}{\alpha\dwn} + \ffrac{4 \omgC\up}{N}} (\gamma L)^2 \cst_2 \leq 1/4  \Longrightarrow \gamma \leq \ffrac{1}{4L\sqrt{\ffrac{\omgC\dwn}{\alpha\dwn} \bigpar{\ffrac{1}{\alpha\dwn} + \frac{4 \omgC\up}{N}} }} \\
        \ffrac{\gamma L \TTwo \cst_1}{N} \leq 1/4 \Longleftrightarrow \ffrac{2 \gamma L \omgC\up}{N} \bigpar{1 + 8 \gamma L \omgC\dwn/\alpha\dwn } \leq 1 / 4
    \end{array}
\right.
\]

We rewrite the second condition as follows:
\[
\left\{
    \begin{array}{ll}
        16 (\gamma L)^2 \ffrac{\omgC\up \omgC\dwn}{\alpha\dwn N }\leq 1 / 8 \Longleftrightarrow \gamma \leq \ffrac{1}{8 L \sqrt{2 \ffrac{\omgC\dwn}{\alpha\dwn} \cdot \ffrac{\omgC\up}{N} }} \\
        \ffrac{2 \gamma L  \omgC\up}{N \TOne} \leq 1 /8\Longleftrightarrow \gamma \leq \ffrac{N}{16 L \omgC\up} \,.
    \end{array}
\right.
\]

Applying convexity, we derive:
\begin{align*}
\begin{split}
    V_k \leq V_{k-1} - \gamma \FullExpec{F(\wkmhat) - F(w_*)} + \frac{\gamma^2 \sigma^2 \Phi^{\heterog}(\gamma)}{Nb}
    \,,
\end{split}
\end{align*}

with $\Phi^{\heterog}(\gamma) = (1+8\omgC\up) \bigpar{1 + \ffrac{ 8 \gamma L \omgC\dwn}{\alpha\dwn} }$.
Invoking Jensen inequality (\ref{app:basic_ineq:jensen}) leads to $\FullExpec{F(\wkmhat)}\geq \FullExpec{F(\wkm)}$, and we finally obtain:
\begin{align*}
V_k \leq V_{k-1} - \gamma \FullExpec{F(w_{k-1}) - F(w_*)} + \frac{\gamma^2 \sigma^2 \Phi^{\heterog}(\gamma)}{Nb}\,.
\end{align*}

\end{proof}

\subsubsection{Strongly-convex case}
\label{app:subsec:stronglyMCM_heterog}

\fbox{
\begin{minipage}{\textwidth}
\begin{theorem}[Convergence of \MCM~in the heterogeneous and strongly-convex case]
\label{app:thm:cvgce_mcm_strongly_convex_heterog}
Under \Cref{asu:cvx_or_strongcvx,asu:expec_quantization_operator,asu:smooth,asu:noise_sto_grad} with $\mu=0$ (convex case), for learning rates $\alpha\dwn~\leq~\ffrac{1}{8 \omgC\dwn}$ and $\alpha\up (1 + \omgC\up) \leq 1$, for any sequence $(\gamma_k)_{k \in \N} \leq \gamma_{\max}^{\mathrm{Heterog}}$,
for any $k$ in $\N$, defining:
$$V_k := \FullExpec{\SqrdNrm{w_k - w_*}} + \gamma_k^2 \cst_1 \FullExpec{\upMemTerm_{k}} + \gamma_k L \cst_2 \FullExpec{\dwnMemTerm_{k}} \,,$$
with $\cst_1 = 2 \omgC\up ( 1 + 8 \gamma L \omgC\dwn / \alpha\dwn) / \TOne$, $\cst_2 = 4 \gamma L \omgC\dwn / \alpha\dwn$, 
we have:
\begin{align*}
\begin{split}
    V_k \leq (1 - \gamma_k \mu) V_{k-1} - \gamma \FullExpec{F(\wkmhat) - F(w_*)} + \frac{\gamma^2 \sigma^2 \Phi^{\heterog}(\gamma)}{Nb}
    \,.
\end{split}
\end{align*}

\end{theorem}
\end{minipage}
}

\begin{proof}
Let $k$ in $\N^*$, the proof starts like the one for \MCM~in the convex case with heterogeneous worker, and we start from \cref{app:eq:mcm_before_lyapunov_heterog} but we consider a variable step size $\gamma_{k} = 2/(\mu (k+1) + \widetilde{L})$ that depends of the iteration $k$ in $\N$.

We consider this following Lyapunov function:
$$V_k = \FullExpec{\SqrdNrm{w_k - w_*}} + \gamma_k^2 \cst_1 \FullExpec{\upMemTerm_{k}}+ \gamma_k L  \cst_2 \FullExpec{\dwnMemTerm_{k}} \,,$$
with $\cst_1 = 2 \omgC\up ( 1 + 8 \gamma_k L \omgC\dwn / \alpha\dwn) / \TOne$ and $\cst_2 = 4  \omgC\dwn / \alpha\dwn $. 
\begin{align*}
\begin{split}
    &\FullExpec{\SqrdNrm{w_k - w_*}} +\gamma_k^2 \cst_1 \FullExpec{\upMemTerm_{k}} + \gamma_k L \cst_2 \FullExpec{\dwnMemTerm_{k}} \leq \SqrdNrm{\wkm - w_*} \\
    &\qquad- \gamma_k \bigpar{1  - \gamma_k L \bigpar{ \bigpar{ \ffrac{1}{\alpha\dwn} + \frac{4 \omgC\up}{N}} \gamma_k L  \cst_2 + \ffrac{\TTwo \cst_1}{N}} }\FullExpec{ \PdtScl{\nabla F (\wkmhat)}{\wkmhat - w_*}} \\
    &\qquad+ \bigpar{2\omgC\up ( 1+ 2 \gamma_k L \cst_2 ) + (1 - \TOne) \cst_1} \gamma_k^2 \FullExpec{\upMemTerm_{k-1}} \\
    &\qquad+ \bigpar{2 \gamma_k L \omgC\dwn + \bigpar{1 - \frac{\alpha\dwn}{2}} \gamma_k L  \cst_2}\FullExpec{\dwnMemTerm_{k-1}}\\
    &\qquad+ \frac{\gamma_k^2 \sigma^2 }{Nb} \bigpar{ (1+4\omgC\up) (1 + 2 \gamma_k L  \cst_2) + 2 \TTwo \cst_1}
    \,,
\end{split}
\end{align*}

To ensure a $(1-\gamma \mu)$-convergence we first choose
$\bigpar{1 - \ffrac{\alpha\dwn}{2} + \ffrac{2\omgC\dwn}{\cst_2}} \gamma_k L \cst_2 \leq (1-\gamma_{k} \mu) \gamma_{k-1} L \cst_2  $ i.e $1 - \ffrac{\alpha\dwn}{2} + \ffrac{2\omgC\dwn}{\cst_2} \leq \ffrac{(1 - \gamma_k \mu) \gamma_{k-1}}{\gamma_k}$.

We need that for all $k \in \N$, $\ffrac{1 - \gamma_k \mu}{\gamma_k} \leq \ffrac{1}{\gamma_{k-1}}$ i.e., $1 - \gamma_k \mu \leq \ffrac{\gamma_k}{\gamma_{k-1}}$, but:
\begin{align*}
    \ffrac{\gamma_k}{\gamma_{k-1}} = \ffrac{\mu k - \mu + \widetilde{L}}{\mu k +\widetilde{L}} = 1 - \ffrac{\mu}{\mu k + \widetilde{L}} \text{\quad and\quad} 1 - \gamma_k \mu = 1- \ffrac{2\mu}{\mu k + \widetilde{L}} \,,
\end{align*}
and so, the inequality is always true.

Thus we must have $2 \omgC\dwn / \cst_2 \leq \alpha\dwn / 2$ which is true by definition of $\cst_2$.

Secondly, we need:
\begin{align*}
    &\bigpar{2\omgC\up ( 1+ 2 \gamma_k L \cst_2 ) + (1 - \TOne)  \cst_1} \gamma_k^2 \leq (1 - \gamma_k \mu) \gamma_{k-1}^2 \cst_1 \\
    \Longleftrightarrow \quad& 2\omgC\up ( 1+ 2 \gamma_k L \cst_2 ) + (1 - \TOne)  \cst_1 \leq \ffrac{\gamma_{k-1}}{\gamma_k} \cst_1 \quad\text{ because $\ffrac{1 - \gamma_k \mu}{\gamma_k} \leq \ffrac{1}{\gamma_{k-1}}$,}\\
\end{align*}

because $\gamma_k/\gamma_k \leq \gamma_{k-1}/\gamma_k$, it is true if we verify the following stronger condition:
\begin{align*}
    &2\omgC\up ( 1+ 2 \gamma_k L \cst_2 ) + (1 - \TOne)  \cst_1 \leq \ffrac{\gamma_{k}}{\gamma_k} \cst_1 \\
   & \cst_1 \geq \ffrac{2 \omgC\up \bigpar{ 1 + 8 \gamma_k L \omgC\dwn) / \alpha\dwn}}{\TOne} \quad \text{because $\cst_2 = 4 \omgC\dwn/\alpha\dwn$} \,. \\
\end{align*}

Finally, in order to apply convexity we must verify: $1  - \gamma L \bigpar{ \bigpar{ \ffrac{1}{\alpha\dwn} + \frac{4 \omgC\up}{N}} \cst_2 + \ffrac{\TTwo \cst_1}{N}} \geq \frac{1}{2} $.

We take $\gamma_k $ such that:
\[
\left\{
    \begin{array}{ll}
        \bigpar{\ffrac{1}{\alpha\dwn} + \ffrac{4 \omgC\up}{N}} \gamma_k L \cst_2 \leq 1/4  \Longrightarrow \gamma_k \leq \ffrac{1}{4L\sqrt{\ffrac{\omgC\dwn}{\alpha\dwn} \bigpar{\ffrac{1}{\alpha\dwn} + \frac{4 \omgC\up}{N}} }} \\
        \ffrac{\gamma L \TTwo \cst_1}{N} \leq 1/4 \Longleftrightarrow \ffrac{2 \gamma_k L \omgC\up}{N} \bigpar{1 + 8 \gamma_k L \omgC\dwn/\alpha\dwn } \leq 1 / 4
    \end{array}
\right.
\]

We rewrite the second condition as following:
\[
\left\{
    \begin{array}{ll}
        16 (\gamma_k L)^2 \ffrac{\omgC\dwn}{\alpha\dwn N}\leq 1 / 8 \Longleftrightarrow \gamma_k \leq \ffrac{1}{8 L \sqrt{2 \ffrac{\omgC\dwn}{\alpha\dwn} \cdot \ffrac{\omgC\up}{N} }} \\
        \ffrac{2 \gamma_k L \omgC\up}{N} \leq 1 /8\Longleftrightarrow \gamma_k \leq \ffrac{1}{16 L \ffrac{\omgC\up}{N}} \,.
    \end{array}
\right.
\]

Now, we can apply strong-convexity:
\begin{align*}
\begin{split}
    V_k \leq (1 - \gamma_k \mu) V_{k-1} - \gamma_k \FullExpec{F(\wkmhat) - F(w_*)} + \frac{\gamma_k^2 \sigma^2 \Phi^{\heterog}(\gamma_k)}{Nb}
    \,,
\end{split}
\end{align*}

with $\Phi^{\heterog}(\gamma) = (1+8\omgC\up) \bigpar{1 + \ffrac{ 8 \gamma_k L \omgC\dwn}{\alpha\dwn} }$. 

Invoking Jensen inequality (\ref{app:basic_ineq:jensen}) leads to $\FullExpec{F(\wkmhat)}\geq \FullExpec{F(\wkm)}$, we finally obtain:
\begin{align*}
V_k \leq (1 - \gamma_k \mu) V_{k-1} - \gamma_k \FullExpec{F(w_{k-1}) - F(w_*)} + \frac{\gamma_k^2 \sigma^2 \Phi^{\heterog}(\gamma_k)}{Nb}\,.
\end{align*}

\end{proof}

\section{Neurips Checklist}
\label{app:sec:checklist}

\begin{enumerate}

\item For all authors...
\begin{enumerate}
  \item Do the main claims made in the abstract and introduction accurately reflect the paper's contributions and scope?
    \answerYes{See \Cref{sec:theory_randomization,sec:theory}.}
  \item Did you describe the limitations of your work?
    \answerYes{For \RMCM, see \Cref{sec:communication_trade_offs}.}
  \item Did you discuss any potential negative societal impacts of your work?
    \answerYes{See \Cref{sec:intro}.}
  \item Have you read the ethics review guidelines and ensured that your paper conforms to them?
    \answerYes{}
\end{enumerate}

\item If you are including theoretical results...
\begin{enumerate}
  \item Did you state the full set of assumptions of all theoretical results?
    \answerYes{See all assumptions in \Cref{sec:theory}.}
	\item Did you include complete proofs of all theoretical results?
    \answerYes{See all demonstrations in \Cref{app:sec:adaptation_to_heterogeneous_case,app:sec:proofs_quad,app:sec:proof_for_ghost,app:sec:proofs_mcm}}
\end{enumerate}

\item If you ran experiments...
\begin{enumerate}
  \item Did you include the code, data, and instructions needed to reproduce the main experimental results (either in the supplemental material or as a URL)?
    \answerYes{All the code is provided on our \href{https://github.com/philipco/mcm-bidirectional-compression/}{github repository}}
  \item Did you specify all the training details (e.g., data splits, hyperparameters, how they were chosen)?
    \answerYes{See \Cref{app:sec:experiments}.}
	\item Did you report error bars (e.g., with respect to the random seed after running experiments multiple times)?
    \answerYes{}
	\item Did you include the total amount of compute and the type of resources used (e.g., type of GPUs, internal cluster, or cloud provider)?
    \answerYes{See \Cref{app:subsec:carbon_footprint}}
\end{enumerate}

\item If you are using existing assets (e.g., code, data, models) or curating/releasing new assets...
\begin{enumerate}
  \item If your work uses existing assets, did you cite the creators?
    \answerYes{We used four dataset : cifar10, mnist, quantum and superconduct.}
  \item Did you mention the license of the assets?
    \answerNo{The dataset are under the MIT licence which is a short and simple permissive license with conditions only requiring preservation of copyright and license notices. As our work is under the same licence, there is no need to remind the licence of the four used dataset.}
  \item Did you include any new assets either in the supplemental material or as a URL?
    \answerNo{}
  \item Did you discuss whether and how consent was obtained from people whose data you're using/curating?
    \answerNA{}
  \item Did you discuss whether the data you are using/curating contains personally identifiable information or offensive content?
    \answerNA{}
\end{enumerate}

\item If you used crowdsourcing or conducted research with human subjects...
\begin{enumerate}
  \item Did you include the full text of instructions given to participants and screenshots, if applicable?
    \answerNA{}{}
  \item Did you describe any potential participant risks, with links to Institutional Review Board (IRB) approvals, if applicable?
    \answerNA{}
  \item Did you include the estimated hourly wage paid to participants and the total amount spent on participant compensation?
    \answerNA{}{}
\end{enumerate}

\end{enumerate}

\end{document}